\documentclass[twoside]{article}

\usepackage[preprint]{aistats2024}

\usepackage[round]{natbib}

\usepackage[utf8]{inputenc} %
\usepackage[T1]{fontenc}    %
\usepackage{hyperref}       %
\usepackage{url}            %
\usepackage{booktabs}       %
\usepackage{multirow}
\usepackage{amsfonts}       %
\usepackage{nicefrac}       %
\usepackage{microtype}      %
\usepackage{xcolor}         %
\usepackage{subcaption}
\usepackage{amsmath}
\usepackage{tikz}
\usepackage{pgfplots}
\usepackage{xcolor}
\usepackage{comment}
\usetikzlibrary{bayesnet}
\usetikzlibrary{arrows.meta, calc, positioning, quotes, fit, backgrounds}
\tikzstyle{solidarrow} = [arrows={-Stealth[length=8pt]}]
\tikzstyle{dashedarrow} = [solidarrow, dashed]
\tikzstyle{paramnode} = [circle, fill=black, scale=0.5]

\tikzset{mega thick/.style={line width=2pt}}

\tikzstyle{latent} = [circle,fill=white,draw=black,inner sep=1pt,
minimum size=25pt, font=\fontsize{10}{10}\selectfont, node distance=1]
\tikzstyle{latentsmall} = [latent]

\tikzstyle{platecaptionnoposition} = [caption, node distance=0, inner sep=0pt]
\tikzstyle{topright} = [platecaptionnoposition, above left=5pt and 0pt of #1.north east] 
\tikzstyle{bottomright} = [platecaptionnoposition, below left=5pt and 0pt of #1.south east] 

\newcommand{\plateex}[5][]{ %
	\node[wrap=#3] (#2-wrap) {}; %
	\node[#4=#2-wrap] (#2-caption) {#5}; %
	\node[plate=(#2-wrap)(#2-caption), #1] (#2) {}; %
}

\usepgfplotslibrary{fillbetween}
\pgfplotsset{compat=1.18}
\usepackage{wrapfig}
\usepackage[capitalise]{cleveref}
\usepackage{amsthm}
\usepackage{multirow}
\usepackage{aligned-overset}

\usepackage{titletoc}

\newcommand{\SD}[1]{{\leavevmode\color{green!90!black} #1}}

\DeclareMathOperator{\erf}{erf}

\DeclareMathOperator{\tr}{tr}
\DeclareMathOperator{\diag}{diag}
\DeclareMathOperator{\E}{\mathbb{E}}
\DeclareMathOperator{\vectorize}{vec}
\newcommand{\vz}{\mathbf{z}}
\newcommand{\vx}{\mathbf{x}}
\newcommand{\vxn}{\mathbf{x}^{(n)}}
\newcommand{\vlambda}{\boldsymbol{\lambda}}

\newcommand{\vnu}{\boldsymbol{\nu}}

\renewcommand{\d}{\mathrm{d}}
\newcommand{\Wt}{\tilde{W}}

\newcommand{\RRR}{\mathbb{R}}
\newcommand{\RRRnorm}{\mathbb{R}_{\mathrm{norm}}} 
\newcommand{\TT}{\mathcal{T}}

\newcommand{\DKL}[2]{D_{\mathrm{KL}}\big(#1 \, \Vert\, #2\big) }

\newcommand{\ELBO}{\mathcal{L}^\mathrm{EL}}     %
\newcommand{\HELBO}{\mathcal{L}^\mathcal{H}}    %
\newcommand{\sigmaopt}{\sigma^2_\mathrm{opt}}
\newcommand{\vlambdaopt}{\vlambda_\mathrm{opt}}
\newcommand{\taunsqr}[1]{\big(\tau^{(n)}_{#1}\big)^2}
\newcommand{\taun}[1]{\tau^{(n)}_{#1}}
\newcommand{\mycomment}[1]{\textcolor{red}{#1}}

\makeatletter
\usepackage{tikz}
\usetikzlibrary{calc}
\newcommand*\circled[1]{\tikz[baseline=(char.base)]{
    \node[shape=circle, draw, inner sep=1pt,
        minimum height={\f@size*1.2},] (char) {\vphantom{WAH1}#1};}}
\makeatother

\newcommand{\exclude}[1]{{\leavevmode\color{gray} #1}}
\renewcommand{\exclude}[1]{}  %

\newtheorem{theorem}{Theorem}
\newtheorem{lemma}{Lemma}

\begin{document}

\runningauthor{Velychko, Damm, Fischer, Lücke}

\twocolumn[

\aistatstitle{Learning Sparse Codes with Entropy-Based ELBOs}

\vspace{-1ex}
\aistatsauthor{Dmytro Velychko \And Simon Damm }
\aistatsaddress{
	Machine Learning Lab\\ 
    University of Oldenburg, Germany 
    \And Faculty of Computer Science\\
    Ruhr University Bochum, Germany}
\vspace{-1ex}

\aistatsauthor{ Asja Fischer \And Jörg Lücke }
\aistatsaddress{ 
    Faculty of Computer Science\\ 
	Ruhr University Bochum, Germany
    \And Machine Learning Lab \\
	University of Oldenburg, Germany
    } ]

\begin{abstract}
  \ \\[-4ex]
  Standard probabilistic sparse coding assumes a Laplace prior, %
  a linear mapping from latents to observables, and Gaussian observable distributions. 
  We here derive a solely entropy-based learning objective for the parameters of standard sparse coding.
  The novel variational objective has the following features: 
  (A)~unlike MAP approximations, it uses non-trivial posterior approximations for probabilistic inference;
  (B)~the novel objective is fully analytic; and 
  (C)~the objective allows for a novel principled form of annealing. %
  The objective is derived by first showing that the standard ELBO objective converges to
  a sum of entropies, which matches similar recent results for generative models with Gaussian priors.
  The conditions under which the ELBO becomes equal to entropies are then shown to have analytic solutions,
  which leads to the fully analytic objective. %
  Numerical experiments are used to demonstrate the feasibility of learning with such entropy-based ELBOs. We investigate
  different posterior approximations including Gaussians with correlated latents and deep amortized approximations. 
  Furthermore, we numerically investigate entropy-based annealing which results in improved learning. %
Our main contributions are theoretical, however, and they are twofold: 
(1)~we provide the first demonstration on how a recently shown convergence of the ELBO to entropy sums can be used for learning; and (2)~using the entropy objective, we derive a fully analytic ELBO objective for the standard sparse coding generative model.
\end{abstract}

\section{INTRODUCTION AND RELATED WORK}
\label{sec:Intro}
Sparse coding seeks to represent data vectors $\vx$ by latent vectors $\vz$. Sparse coding requires the vectors~$\vz$ to be
{\em sparse}, i.e., on average only few of the values $z_h$ significantly contribute in representing any given vector $\vx$. 
Our main focus will be the (by far) most standard data model for probabilistic sparse coding \citep[][]{williams_bayesian_1995,olshausen_emergence_1996,seeger_bayesian_2007}.
The model assumes a Laplacian (a.k.a.\ double-exponential) prior distribution for latents \mbox{$\vz\in\RRR^H$}, and a Gaussian  noise distribution for observables \mbox{$\vx\in\RRR^D$}, %
\begin{align}
\begin{split}
\label{EqnPSC}
    p(\mathbf{\vz}) \,&=\, \prod_{h=1}^H \frac{1}{2} \exp\big(-|z_h|\big) \quad \text{and} \\
	p_{\Theta}(\vx \,|\, \vz) \,&=\, \mathcal{N} \big(\vx\,|\,W\vz, \sigma^2\mathbb{I}\big) \enspace, 
\end{split}
\end{align}
where %
weight matrix $W\in\RRR^{D\times{}H}$ and observation noise
$\sigma^2>0$ are the model parameters \mbox{$\Theta=(W,\sigma^2)$}.
The sparse coding model, and in particular the Laplace prior distribution, are closely related to deterministic sparse coding approaches that use the $l_1$-objective \citep[e.g.,][]{hastie_statistical_2015}. A standard form of deterministic sparse coding addresses the optimization problem
\begin{align}
	\label{EqnSCOp}
	\min_{\vz^{(1)},\ldots,\vz^{(N)}}\Big\{\,   \underbrace{ \sum_{n=1}^N \big\| \vx^{(n)} - \Wt\vz^{(n)} \big\|^2}_{\mathrm{reconstruction}}   
	\,+\, \tilde{\gamma} \underbrace{ \sum_{n=1}^N \sum_{h=1}^H \big|z_h^{(n)}\big| }_{\mathrm{sparsity}}   \,\Big\} \enspace, 
    \raisetag{4mm}
\end{align} 
where $\vz^{(1)},\ldots,\vz^{(N)}$ are deterministic latent vectors corresponding to data vectors 
 $\vx^{(1)},\ldots,\vx^{(N)}$, and where $\Wt\in\RRR^{D\times{}H}$ with columns of unit length. 
The constant $\tilde{\gamma}$ (often also denoted $\lambda$) weights the sparsity term vs.\ the reconstruction term of the objective.

For more than two decades, sparse coding approaches 
have been very thoroughly investigated with large numbers of papers dedicated to theoretical investigations of the respective optimization problems, and with many
papers using different (including deep) forms of sparse coding for numerous tasks. Such tasks included (to name a few) denoising, inpainting, compression, disentanglement, or super-resolution \citep[e.g.][]{mairal_sparse_2014, yao_patch-based_2022, cheng_rethinking_2022, drefs_direct_2023}.

For the standard probabilistic data model, given in \cref{EqnPSC}, the presumably most common way to derive algorithms for parameter optimization is maximum likelihood (ML) estimation. That is, we seek those parameters 
of the model %
that maximize the (marginal) log-likelihood ${\cal L}^{\mathrm{LL}}(\Theta)$ in dependence of the likelihood parameters ${\Theta= (W, \sigma^2)}$ with\vspace{-0.4ex}
\begin{align}\label{EqnLL}
	{\cal L}^{\mathrm{LL}}(\Theta) &= \frac{1}{N} \sum_{n=1}^N \log\Big( \int p_\Theta(\vx^{(n)} | \vz)\, p(\vz)\, \d\vz \Big) \enspace .\vspace{-0.4ex}
\end{align}
In order to facilitate the challenging problem of maximizing ${\cal L}^{\mathrm{LL}}$, approximations to ML optimization are very commonly applied. One of the most common approximation methods applied for probabilistic sparse coding (and probabilistic generative models in general) is variational approximation \citep[e.g.][]{jaakkola_variational_1997}. Concretely, instead of maximizing the likelihood directly, a lower bound of the log-likelihood is maximized, which is referred to as free-energy or ELBO \citep[e.g.,][]{neal_hinton_view_1998,jordan_introduction_1998}: 
\exclude{
	\begin{align}%
		\mathcal{L}(\Phi,W,\sigma^2) = \frac{1}{N} \sum_{n=1}^N \int q^{(n)}(\vz;\Phi) \log\big( \frac{p(\vx^{(n)}) | \vz, W, \sigma^2)\big)\, p(\vz)}{q(\vz)} \,\d\vz  
	\end{align}
}
\begin{align}
\begin{split}
\label{EqnELBO}
	\ELBO (\Phi,\Theta) 
	= \frac{1}{N} \sum_{n=1}^N 
	\Big[
	    &\int q^{(n)}_{\Phi}(\vz) \log p_\Theta(\vx^{(n)}\,|\,\vz) \d\vz \\ 
	    &-\, \DKL{q^{(n)}_{\Phi}(\vz)}{p(\vz)} 
	\Big] \enspace .
 \end{split}
\end{align}
Given the data model in \cref{EqnPSC}, the ELBO is defined by the family of variational distributions $q_\Phi^{(n)}(\vz)$ used to approximate the true posteriors of a given model. 
The standard choice for probabilistic sparse coding are Gaussian variational distributions to approximate the analytically intractable posteriors of the model. For sparse coding as in \cref{EqnPSC}, the true posteriors are known to be mono-modal \citep[][]{olshausen_emergence_1996,seeger_bayesian_2007} due to log-concavity. Therefore, Gaussian approximations (by matching mode and correlations) can be considered as capturing the most essential structure of the model's true posteriors. %

The optimization of lower bounds such as the ELBO usually represents an easier optimization problem than optimizing the likelihood itself. 
However, the crucial challenge for both of these optimizations is posed by the integrals over potentially high-dimensional latent spaces. For the standard sparse coding model, no analytic solutions have been reported, so far. In particular, no analytic solutions have been reported for the common case of using Gaussians as the family of variational distributions.

It could be argued that deterministic algorithms are, nonetheless, available if much more simplifying approximations than Gaussians are used for optimization. 
The arguably most common approach is given by maximum a-posteriori (MAP) training \citep[][]{olshausen_emergence_1996}.  
From a probabilistic perspective, MAP approximations may be interpreted as a limit case of variational approximations in which the family of variational distributions are delta-distributions. The high-dimensional integrations over latent space are then trivially solved.
As a source for its high popularity, MAP approximations allow
for linking the standard probabilistic model in \cref{EqnPSC} to the deterministic $l_1$-sparse coding objective in \cref{EqnSCOp}. That is, the sparse coding objective in \cref{EqnSCOp} can be recovered if the MAP approximation is applied.  Another source of the ongoing popularity of MAP (also in general) is the resulting closed-form objective (cf.\,Eq.\;\ref{EqnSCOp}), i.e., no high-dimensional integrals have to be numerically estimated.

However, from a probabilistic machine learning perspective, delta-distributions do not represent theoretically well-grounded approximations. One consequence of using MAP is, for instance, that the ELBO objective is rendered non-finite and thus cannot be considered as a learning objective anymore \citep[see, e.g.,][for a discussion]{barello_sparse_coding_2018}; also the meaning of the ELBO as a lower bound of the log-likelihood ceases to provide meaning in the MAP case. 
Moreover, severe degeneracies are introduced: optimization of $W$ tends to yield infinite entries \citep{olshausen_emergence_1996}, which has to be manually corrected, and data noise $\sigma^2$ and sparsity are not learnable independently of one another. More generally, no probabilistic encoding is provided, i.e., neither can a probabilistic objective be used for tasks such as model selection nor is there uncertainty information available for data encoding (with all the negative consequences for downstream tasks one may seek to address). Such major drawbacks have, consequently, resulted in substantial research efforts to allow for appropriate uncertainty estimation in sparse coding.
Strategies that were followed include (i)~the application expectation propagation \citep{seeger_bayesian_2007}, (ii)~sampling-based fully Bayesian approaches \citep{mohamed_bayesian_2012}, (iii)~amendments of the original data model %
\citep[][]{berkes_sparsity_2007,sheikh_truncated_2014} such that non-trivial variational optimizations could be applied, (iv)~the use of amortized variational distributions for the original data model \citep[][]{barello_sparse_coding_2018}, or (v)~amendments of both data model and variational distributions \citep[][]{tonolini_variational_2020,drefs_direct_2023}.

\section{ELBO CONVERGENCE TO ENTROPY SUMS}
\label{sec:Convergence}
There are many ways to rewrite the ELBO and relate it to Kullback-Leibler divergence, entropies, cross-entropies, mutual information, and expected reconstruction error \citep[e.g.][]{alemi2018fixing, hoffman2016elbo, zhao2017infovae}. 
In contrast, we in this work seek to rewrite the ELBO as
a sum of entropies. The reformulation is obtained through
assuming convergence of a subset of the model parameters, i.e.,
our reformulation is valid on a submanifold in the space
of all model parameters.
Our reformulation is consequently different from previously known reformulations that are valid for the entire space of parameters.

We will first show that the ELBO for the sparse coding model given in \cref{EqnPSC} converges to a sum of three entropies given optimization
of specific model parameters.
The derived results will apply to general variational distributions, the specific variational family of Gaussian distributions
will only be used later.
Our derivations are based on recent results for variational autoencoders (VAEs) that show convergence of
standard (Gaussian) VAEs to entropy sums \citep[][]{DammEtAl2023}. These results have deeper roots in the exponential family property of
Gaussians \citep[][]{Lucke2022Convergence} and we here, for the first time, show that the ELBO of standard sparse coding %
converges to entropy sums.

Our main focus will be on learning, which contrasts with previous work \citep[][]{DammEtAl2023,Lucke2022Convergence} that
investigated the properties of the ELBO at stationary points. 
The focus on learning means that we will exploit properties the ELBO in \cref{EqnELBO} attains if a subset of model parameters have converged.
To specify those parameters we will first reparameterize the sparse coding model introduced in \cref{EqnPSC} before we investigate convergence to entropy sums.

\subsection{Reparameterization of Sparse Coding}

Consider an elementary Bayesian network (\cref{fig:graphical-model}, left) for probabilistic latent variable models, that covers models such as sparse coding, as in \cref{EqnPSC}, probabilistic PCA \citep{tipping_probabilistic_1999}, and VAEs \citep{KingmaWelling2014}. For our derivations of the entropy-based ELBO, we will use a slightly altered form of the model with learnable prior parameters (\cref{fig:graphical-model}, right). Concretely, we constrain the columns of the weight matrix (now termed $\Wt$) to be of unit length but we use parameterized Laplace distributions for the prior (instead of the parameterless standard choice). 

The sparse coding model is thus given by
\begin{align} %
	\begin{split}
		p_\Theta(\mathbf{\vz}) \,&=\, \prod_{h=1}^H \frac{1}{2\lambda_h} \exp\left(-\frac{|z_h|}{\lambda_h}\right) \ \ \mbox{and}\ \\
		p_\Theta(\vx | \vz) \,&=\, \mathcal{N} \big(\vx\,|\,\Wt\vz, \sigma^2\mathbb{I}\big) \enspace,  	
	\end{split}
	\label{EqnPSC2}
\end{align}
where $\forall{}h:\ \sum_{i} \big(\Wt_{ih}\big)^2\,=\,1$, such that $\Theta$ now reads %
${\Theta=(\vlambda,\Wt,\sigma^2) \in \RRR_+^H \times \RRRnorm^{D\times{}H} \times \RRR_+ }$. 
The prior parameters $\lambda_h$ are commonly referred to as scales. 

It is straightforward to show that the reparameterized model in \cref{EqnPSC2} parameterizes the same family of distributions $p_\Theta(\vx)$
as the original model in \cref{EqnPSC}, with a one-to-one mapping between their respective parameters.
This parameterization is important in order to show that the ELBO of sparse coding becomes equal to entropy sums under certain conditions.
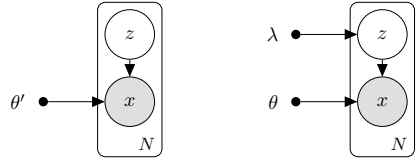
\begin{figure}[!t]
	\centering
	\scalebox{.75}{
		\begin{tikzpicture}
			\node[latent] (z) {$z$};
			\node[obs, below=0.3 cm of z] (x) {$x$};
			\node[paramnode, left=1 cm of x] (theta) {};
			\node[left=0.1 cm of theta] (thetalabel) {$\theta'$};
			\edge {z, theta} {x};
			\plateex {zx} {(z)(x)} {bottomright} {$N$}
		\end{tikzpicture}
		\hspace{1.5cm}
		\begin{tikzpicture}
			\node[latent] (z) {$z$};
			\node[obs, below=0.3 cm of z] (x) {$x$};
			\node[paramnode, left=1 cm of z] (lambda) {};
			\node[left=0.1 cm of lambda] (lambdalabel) {$\lambda$};
			\node[paramnode, left=1 cm of x] (theta) {};
			\node[left=0.1 cm of theta] (thetalabel) {$\theta$};
			\edge {lambda} {z};
			\edge {z, theta} {x};
			\plateex {zx} {(z)(x)} {bottomright} {$N$}
		\end{tikzpicture}
	}
	\caption{\textbf{Latent variable model.} \textit{Left:} graphical model
 representation corresponding to many popular latent variable models, including VAEs. \textit{Right:}~graphical model with learnable prior parameters and constrained likelihood parameters as used in this work.}%
	\label{fig:graphical-model}
\end{figure}

\subsection{Equality of ELBO and Entropy Sums}
\exclude{
	THEOREM ABOUT ENTROPY CONVERGENCE
}
Motivated by previous work \citep{DammEtAl2023}, we now investigate if
the ELBO of the model in \cref{EqnPSC2} becomes equal to entropy sums during learning. \citet{DammEtAl2023} did show equality of ELBO and entropy sums for Gaussian models (Gaussian prior and Gaussian noise model) at all stationary points. 
For this work, we will show equality to entropy sums for the ELBO of sparse coding. 
But, furthermore, it will be important for this work to explicitly note that only the parameters $\vlambda$ and $\sigma^2$ have to be at stationary points in order to realize equality to entropy sums.

\begin{theorem}[ELBO converges to a sum of entropies]
	\label{theo:laplace-prior}
	Consider the ELBO in \cref{EqnELBO} for the sparse coding model in \cref{EqnPSC2} with parameters $\Theta=(\vlambda,\Wt,\sigma^2)$.
	If the parameters $\vlambda$ and $\sigma^2$ are at a stationary point, i.e., %
	\begin{align}
		\label{eq:condition_stationary_points}
		\textstyle\frac{\partial}{\partial \vlambda}\ELBO(\Phi, \Theta) = 0\ \ \mbox{and}\ \ \frac{\partial}{\partial \sigma^2}\ELBO(\Phi, \Theta) = 0\ \enspace,
	\end{align}
	then it applies for any variational distributions $q_{\Phi}(\vz)$ and for any matrix $\Wt$ (with unit column lengths) that:
	\begin{align}
        &\hspace{-0ex}\ELBO(\Phi, \Theta) \nonumber \\
		&=\textstyle\frac{1}{N} \sum_n \mathcal{H}[q_\Phi^{(n)}(\vz)] 
		- \mathcal{H}[p_\Theta(\vz)] 
		- \mathcal{H}[p_\Theta(\vx | \vz)]  \enspace .
		\label{eq:three-entropies}
	\end{align}
\end{theorem}
\begin{proof}
	The ELBO objective in \cref{EqnELBO} can be rewritten to consist of three summands, i.e., %
	\begin{align*} %
            \textstyle  \ELBO(\Phi, \Theta) =  \\[-1ex]
                \frac{1}{N} \sum_n \mathcal{H}&[q^{(n)}_{\Phi}(\vz)] 
				+ \overbrace{ 
						\textstyle\frac{1}{N} \sum_n \int q^{(n)}_{\Phi}(\vz) \log p_{\Theta}(\vz) \d\vz
					}^{\ELBO_1(\Phi, \Theta)} \\
				+ &\underbrace{ 
						\textstyle\frac{1}{N} \sum_n \int q^{(n)}_{\Phi}(\vz) \log p_{\Theta}(\vx^{(n)}|\vz) \d\vz
					}_{\ELBO_2(\Phi, \Theta)} \enspace .
	\end{align*}
	The first summand is already in the form of an (average) entropy.
        The last summand, $\ELBO_2(\Phi, \Theta)$, has the form
	\begin{align*}
        \textstyle
		\ELBO_2(\Phi, \Theta) = \frac{1}{N}&\sum_n \Big( - \frac{D}{2} \log\!\big( 2\pi\sigma^2 \big) \\ & - \frac{1}{2\sigma^2} \int q^{(n)}_{\Phi}\!(\vz)\, \| \vxn - \tilde W\vz \|^2 \mathrm{d}\vz 
		\Big) \enspace.
	\end{align*}
	
	If $\frac{\partial}{\partial \sigma^2}\mathcal{L}(\Phi, \Theta) = 0$,
        we get $\ELBO_2(\Phi, \Theta) = -\mathcal{H}[p_\Theta(\vx | \vz)]$,
	which can be shown analogously to the Gaussian noise distribution used in Gaussian variational autoencoders \cite[][]{DammEtAl2023}.\footnote{For completeness, we reiterate the derivation for our case in \cref{app:gaussian-likelihood-convergence}.}
	
    To show that \cref{eq:three-entropies} holds, it is therefore left to show that the summand $\ELBO_1(\Phi, \Theta)$ has the form of
    an entropy under the conditions of the theorem.
    We invoke the factorization $q_\Phi^{(n)}(\vz) = q_\Phi^{(n)}(\vz_{/h} | z_h) q_\Phi^{(n)}(z_h)$ to simplify the integral $\ELBO_1(\Phi, \Theta)$ such that $\ELBO_1(\Phi, \Theta)=$
    \begin{align}
    \begin{split}
		   &\hspace{-1ex}\sum_h \Big( -\log(2\lambda_h) - \frac{1}{N} \sum_n \frac{1}{\lambda_h} \int \, |z_h| \, q_\Phi^{(n)} (z_h) \, \d z_h \Big) \label{eq:laplace:f1}.
    \end{split}
    \end{align}
    \exclude{
	We can consider every dimension $h$ separately. 
        At stationary points we have $\frac{\partial \ELBO_1(\Phi, \Theta)}{\partial \lambda_h}=0$, which implies:
	\begin{align}
		\label{eq:lambda_opt_condition}
		-\frac{\partial \log(2\lambda_h)}{\partial \lambda_h} - \frac{\partial \lambda^{-1}_h}{\partial \lambda_h} \int \, |z_h| \, q_\Phi^{(n)} (z_h) \,\d z_h = 0
		\enspace .
	\end{align}
	Thus, at stationary points the integral in \cref{eq:lambda_opt_condition} reads: 
	\begin{align}
		\int \, |z_h| \, q_\Phi^{(n)} (z_h) \, \d z_h 
		= -\frac{\partial \lambda_h}{\partial \lambda^{-1}_h} \frac{\partial \log(2\lambda_h)}{\partial \lambda_h}  
		= \lambda^\star_h \label{eq:laplace:intabsz} \enspace .
	\end{align}
	Inserting the solution of the integral of \cref{eq:laplace:intabsz} back into \cref{eq:laplace:f1} and denoting  $\Theta^\star = (\tilde W, \sigma^2, \vlambda^\star)$  we get the final negative entropy value, which completes the proof:
    \begin{align*}
		\ELBO_1(\Phi, \Theta^*)
		= \sum_h \left( -\log(2\lambda^\star_h) -\frac{1}{N} \sum_n \frac{1}{\lambda^*_h}\lambda^\star_h \right) \\
		= - \sum_h \log(2\lambda^*_h e) 
		= -\mathcal{H}[p_{\Theta}(\vz)]\enspace . 
    \end{align*}
    }
    As only the term $\ELBO_1(\Phi, \Theta^*)$ of the ELBO depends on $\vlambda$, we obtain at stationary points of \cref{eq:laplace:f1} w.r.t $\lambda_h$:
    \begin{align*}
      0 
      &= \frac{\partial}{\partial \lambda_h} \ELBO(\Phi, \Theta) 
       = \frac{\partial}{\partial \lambda_h} \ELBO_1(\Phi, \Theta) \\
      &= -\frac{1}{\lambda_h} + \frac{1}{N} \sum_n \frac{1}{\lambda_h^2} \int \, |z_h| \, q_\Phi^{(n)} (z_h) \, \d z_h \\
      &= \frac{1}{\lambda_h} \left( -1 + \frac{1}{N} \sum_n \frac{1}{\lambda_h} \int \, |z_h| \, q_\Phi^{(n)} (z_h) \, \d z_h \right) ,
    \end{align*}
    for all $h$. As $\lambda_h \neq 0$, it follows that
    \begin{align}
      \frac{1}{N} \sum_n \frac{1}{\lambda_h} \int \, |z_h| \, q_\Phi^{(n)} (z_h) \, \d z_h = 1 .  \label{eq:conv_proof:eqal_one}
    \end{align}
    Now we insert \cref{eq:conv_proof:eqal_one} into \cref{eq:laplace:f1} and
    obtain:
    \begin{align*}
      \ELBO_1(\Phi, \Theta) 
      &= - \sum_h \log(2e\lambda_h)
      = -\mathcal{H}[p_{\Theta}(\vz)] . \qedhere
    \end{align*}
\end{proof}
In \cref{app:FactorizationTheorem} we present a simple but more general theorem that \textit{constructively} proves convergence to entropies for a small class of exponential family distributions. More general convergence criteria were presented by \citet{Lucke2022Convergence}, see \cref{app:proof-with-general-conditions}.

\section{ENTROPY-BASED ELBOS AS LEARNING OBJECTIVES}

\exclude{THEOREM ABOUT $\lambda_h$ and $\sigma^2$ SOLUTIONS (THIS WILL BE PLACED AS ONE OF OUR TWO MAIN CONTRIBUTIONS)}
The entropy sum expression in \cref{theo:laplace-prior} %
does by itself {\em not} represent a learning objective because it requires the conditions in \cref{eq:condition_stationary_points}; and these are usually not satisfied during optimization.
However, do note that the conditions only concern a subset of the parameters of the ELBO, i.e., $\vlambda$ and $\sigma^2$. No conditions have to be fulfilled for the parameters $\Wt$ and the variational parameters $\Phi$. 
Importantly, this means that the expression in \cref{eq:three-entropies} can potentially be used as a learning objective if we can derive solutions for $\vlambda$ and $\sigma^2$ that satisfy the conditions stated in \cref{eq:condition_stationary_points}. 
For our specific choice of variational distributions $q_\Phi^{(n)}(\vz)$ we can, notably, find analytic such solutions. %

\begin{theorem}[Optimal scales and variance]
	\label{theo:optimal-parameters}
	For the sparse coding model in \cref{EqnPSC2} consider the ELBO in \cref{EqnELBO} defined with Gaussian distributions $q_\Phi^{(n)}(\vz) = \mathcal{N}(\vz \,|\, \vnu^{(n)}, \TT^{(n)})$, for $n=1,\dots,N$, as
	family of variational distributions. The variational parameters are consequently given by $\Phi=(\vnu^{(1)},\ldots,\vnu^{(N)},\TT^{(1)},\ldots,\TT^{(N)})$ with $\vnu^{(n)}\in\RRR^H$ and positive semi-definite matrices $\TT^{(n)}\in\RRR^{H\times{}H}$. 
    For arbitrary such variational distributions
    and for an arbitrary matrix $\Wt$ (with unit length columns), we can then find the values for $\vlambda$ and $\sigma^2$ that
    satisfy \cref{eq:condition_stationary_points}.
    The solutions for $\vlambda$ and $\sigma^2$ are unique and are given by
\begin{align}
    \sigmaopt \big(\Phi,\Wt\big) &= \frac{1}{N} \sum_n \frac{1}{D}\Bigg[\tr(\tilde W^\mathrm{T} \tilde W\TT^{(n)}) \\
        &\hspace{4mm}+(\tilde W\vnu^{(n)}-\vx^{(n)})^\mathrm{T}(\tilde W\vnu^{(n)}-\vx^{(n)}) \Bigg] \enspace , \notag  \\   
    \label{eq:lambda_M}
	\forall h: \lambda^{\mathrm{opt}}_h&\big(\Phi\big) = \frac{1}{N} \sum_n \sqrt{\TT_{hh}^{(n)}} {\cal M}\Bigg( \frac{\nu_h^{(n)}}{\sqrt{\TT_{hh}^{(n)}}} \Bigg) \\ %
    \mbox{with }
	{\cal M}(a) &= \sqrt{ \frac{2}{\pi} } \exp\left(-\frac{1}{2}\,a^2 \right) + a \erf\left(\frac{a}{\sqrt{2}}\right)
	.
	\label{eq:DefFuncM}
\end{align}
\end{theorem}
\begin{proof}[Proof sketch]
    We solve the arising integrals analytically to find the corresponding parameters at stationary points. \cref{app:derive-analytic-elbo} contains the full derivations.
\end{proof}

\subsection{Entropy-based Learning Objective for Standard Sparse Coding}
\label{sec:Entropy-basedELBO}
\exclude{
	(THESE OBJECTIVES WILL BE PLACED AS THE SECOND OF OUR TWO MAIN CONTRIBUTIONS)
	\ \\
	\ \\
	THEOREM STATING THE OBJECTIVE 
}

We can now consider the subspace of all parameters $\Theta$
with optimal $\vlambda$ and $\sigma^2$. These are
all parameters that can be obtained from parameters $\Phi$ and $\Wt$ through the following function:
\begin{align}
    \Theta_{\mathrm{opt}}(\Phi,\Wt) = \big( \vlambdaopt(\Phi),\Wt,\sigmaopt(\Phi,\Wt)\big) \enspace,
 \label{eq:theta-opt} %
\end{align} 
where $\vlambdaopt(\Phi)$ and $\sigmaopt(\Phi,\Wt)$ are provided by \cref{theo:optimal-parameters}.
As for all $\Theta_{\mathrm{opt}}(\Phi,\Wt)$ the conditions for \cref{theo:laplace-prior} are fulfilled, it applies for all $\Phi$ and $\Wt$ that
	\begin{align}
        \begin{split}
        &\hspace{0ex}\ELBO\big(\Phi, \Theta_{\mathrm{opt}}(\Phi,\Wt)\big) = \frac{1}{N} \sum_n \mathcal{H}[q_\Phi^{(n)}(\vz)] 
		\\
		& \hspace{3ex}- \mathcal{H}[p_{\Theta_{\mathrm{opt}}(\Phi,\Wt)}(\vz)]
		- \mathcal{H}[p_{\Theta_{\mathrm{opt}}(\Phi,\Wt)}(\vx | \vz)] \enspace.
        \end{split}
		\label{eq:theorem-comb}
	\end{align}
The entropy-based right-hand-side of \cref{eq:theorem-comb} only depends on $\Phi$ and $\Wt$, and it suggests itself as a novel objective for these remaining parameters. Importantly, as the entropies in \cref{eq:theorem-comb} are all given in closed-form, and as $\vlambdaopt(\Phi)$ and $\sigmaopt(\Phi,\Wt)$ are analytic functions, the novel objective is an analytic function as well. Using the expressions for the entropies in \cref{eq:theorem-comb} and the solutions $\vlambdaopt(\Phi)$ and $\sigmaopt(\Phi,\Wt)$, the objective is given by (see  \cref{app:derive-analytic-elbo} for intermediate steps):
\begin{align}
    \begin{split}
        &\HELBO(\Phi, \Wt)=  \frac{1}{N} \sum_{n=1}^{N} \frac{1}{2} \log\big(\,|\,2\pi\,e\,\TT^{(n)} \,|\,\big)      \\
        &- \sum_{h=1}^H \log \Bigg( 2\,e\,  \frac{1}{N} \sum_{n=1}^{N} \sqrt{\TT_{hh}^{(n)}} {\cal M}\Bigg( \frac{\nu_h^{(n)}}{\sqrt{\TT_{hh}^{(n)}}} \Bigg)    \Bigg)  \\
        &-\frac{D}{2} \log \Bigg( 2\pi{}\,e\,\frac{1}{N} \sum_{n=1}^{N} \frac{1}{D}\Big[\tr(\tilde W^\mathrm{T} \tilde W\TT^{(n)}) \\
        &\hspace{5mm}+\big(\tilde W\vnu^{(n)}-\vx^{(n)}\big)^\mathrm{T}\big(\tilde W\vnu^{(n)}-\vx^{(n)}) \Big] \Bigg) \enspace,
    \end{split} \label{eq:analytic-elbo} %
\end{align}
where ${\cal M}\big( \cdot \big)$ is the analytic function from \cref{eq:DefFuncM}.

Considering the new objective $\HELBO$, there is, however, a subtle but important difference compared to $\ELBO$: the solutions for $\vlambda$ and $\sigma^2$ introduce dependencies between model parameters $\Theta$ and variational parameters $\Phi$. As a consequence, 
the standard lower-bound relation between log-likelihood (only depending on $\Theta$) and $\HELBO$ becomes more intricate away from stationary points. We can, however, show that the objective $\HELBO(\Phi, \Wt)$ has the 
same stationary points (together with $\vlambdaopt(\Phi)$ and $\sigma^2_{\mathrm{opt}}(\Phi,\Wt)$) as the original ELBO $\ELBO(\Phi, \Theta)$.

\begin{theorem}
    \label{theo:analytic-elbo}
    Consider the sparse coding model formulated in \cref{EqnPSC2} with model parameters ${\Theta=(\vlambda,\Wt,\sigma^2) \in \RRR_+^H}\times \RRRnorm^{D\times{}H} \times \RRR_+$, and variational parameters $\Phi = (\Phi_\nu, \Phi_\TT)$ that parameterize mean $ \vnu^{(n)} \in \RRR^H$ and covariance $\TT^{(n)} \in \RRR^{H\times H}$ (in amortized or non-amortized fashion) where $\Phi_\nu \cap \Phi_\TT = \emptyset$.
    Then, the set of stationary points of the original objective $\ELBO(\Phi,\Theta)$, given in \cref{EqnELBO}, and of the entropy-based objective $\HELBO(\Phi,\Wt)$, given in \cref{eq:analytic-elbo}, coincide.
    Furthermore, at any stationary point it holds %
	\begin{align}
		\label{eq:fixpointsLL}
		\ELBO(\Phi^\star,\Theta^\star) = \HELBO(\Phi^\star,\Wt^\star) \enspace .
	\end{align}

	\exclude{
		
		Consider the same conditions as for Theorem ..., i.e., the sparse coding model of \cref{EqnPSC2} and Gaussian variational distributions.
		Then the ELBO of \cref{EqnELBO}
		
		sparse coding model in \cref{EqnPSC2} and variational optimization using Gaussian distributions with variational
		parameters $\Phi=(\vnu^{(1)},\TT^{(1)},\ldots,\vnu^{(N)},\TT^{(N)})$. 
		Then, its ELBO in \cref{EqnELBO} be written as a sum of entropies
		and thus has an analytic solution given by:
		\begin{align}
			\begin{split}
				\HELBO&(\Phi,\Wt) = \hspace{-1ex} \phantom{+} \frac{1}{N} \sum_{n=1}^{N} \frac{1}{2} \log\big(\,|\,2\pi\,e\,\TT^{(n)} \,|\,\big)      \\
				& - \sum_{h=1}^H \log \Big( 2\,e\, \frac{1}{N} \sum_n \left[ \frac{2\sqrt{\TT_{hh}^{(n)}}}{\sqrt{2\pi}} \exp\left(-\frac{1}{2} \frac{(\nu_h^{(n)})^2}{\TT_{hh}^{(n)}}\right) + \nu_h^{(n)} \erf\left(\frac{\nu_h^{(n)}}{\sqrt{2 \TT_{hh}^{(n)}}}\right) \right]    \Big)  \\
				& -\frac{D}{2} \log \Big( 2\pi{}e \frac{1}{DN} \sum_n \left[\tr\big(\Wt^\mathrm{T}\Wt\TT^{(n)}\big) + \big(\Wt\nu^{(n)}-\vx^{(n)}\big)^\mathrm{T}(\Wt\nu^{(n)}-\vx^{(n)}) \right]      \Big) \enspace .
			\end{split}
		\end{align}
	}
\end{theorem}
\begin{proof}[Proof sketch]
As all stationary points must satisfy \cref{eq:condition_stationary_points}, \cref{eq:fixpointsLL} holds directly by \cref{theo:laplace-prior}. 
To show that any stationary point of one objective is also a stationary point of the other we show that the gradients of both objectives 
coincide whenever \cref{eq:condition_stationary_points} holds.
The full proof is deferred to \cref{Theorem3:Proof}.
 \exclude{
	The solutions for $\vlambda$ and $\sigma^2$ in ... and ... are defined for any parameters $\Phi$ and $\Wt$. Therefore, we can define the function $\vlambda(\Phi)$ using
	the right-hand-side of ..., and the function $\sigma^2(\Phi,\Wt)$ using the right-hand-side of ... . In virtue of Theorem~\ref{theo:optimal-parameters}, the functions $\vlambda(\Phi)$
	and $\sigma^2(\Phi,\Wt)$ map to those values of $\vlambda$ and $\sigma^2$ at which applies conditions (\ref{eq:condition_stationary_points}) are fulfilled. In virtue of Theorem~\ref{theo:laplace-prior}, the ELBO ${\cal L}(\Phi,\Theta)$ is equal to the three entropies expression (\ref{eq:three-entropies}). We use the formulas for the individual entropies
	to obtain:
	\begin{align}
		\mathcal{L}(\Phi, \Wt, \vlambda, \sigma^2) = ...
		\label{eq:three-entropies-explicit}
	\end{align}
	We then define $\HELBO(\Phi,\Wt) = \mathcal{L}(\Phi, \Wt, \vlambda(\Phi), \sigma^2(\Wt))$, i.e., the solution ... and ... remove the dependence of $\mathcal{L}$ on $\vlambda$
	and $\sigma^2$. As the solutions for $\vlambda$ and $\sigma^2$ represent (global) maxima of $\mathcal{L}(\Phi, \Wt, \vlambda, \sigma^2)$ w.r.t.\ $\vlambda$ and $\sigma^2$, the new function
	$\HELBO(\Phi,\Wt)$ is greater equal ${\cal L}(\Phi,\Theta)$. As ${\cal L}(\Phi,\Theta)$ is a lower bound of the log-likelihood ${\cal L}^{\mathrm{LL}}(\Theta)$ for arbitrary $\Phi$ and as there is by definition of $\HELBO$ always a $\Theta$ such that $\HELBO(\Phi,\Wt)={\cal L}(\Phi,\Theta)$, we conclude that $\HELBO(\Phi,\Wt)$  ?????? ...
	\begin{align}
		\mathcal{L}(\Phi, \Theta) = 
		\frac{1}{N} \sum_n \mathcal{H}[q_\Phi^{(n)}(\vz)] 
		- \mathcal{H}[p_\Theta(\vz)] 
		- \mathcal{H}[p_\Theta(\vx | \vz)] \enspace .
	\end{align}
    Use optimal parameters of the entropies from \cref{theo:optimal-parameters} to complete the entropy-based expression of ELBO in \cref{eq:three-entropies}.}
\end{proof}
In virtue of \cref{theo:analytic-elbo}, $\HELBO(\Phi, \Wt)$, given in \cref{eq:analytic-elbo}, can be used as novel objective for standard sparse coding with Laplace prior. 
Note that also the error function is known to be an analytic function, which can be seen, e.g., by considering its representation by the B\"urmann series \citep{schopf_burmanns_2014}. The series representations also highlight that for all practical reasons, very accurate (and readily available) closed-form approximations of the error function can be used for optimization (see \cref{app:approx-error-function}).

\exclude{
	Writing the error function in the form of a convergent, infinite power series(?) serves to highlight that
	objective ... is an analytic (but not a closed-form) solution.
	\ \\
	ANALYTIC FORMULA FOR THE ELBO!!
	\ \\
	CLOSED-FORM FORMULA FOR THE ELBO (Bürman approximation, mention others)
}
\exclude{ WE PLACE CONSTANTS IN FRONT OF THE ENTROPIES
	
	CITE: BETA-VAE
	and
	Improving Explorability in Variational Inference with Annealed Variational Objectives, Huang et al., NeurIPS 2018
}

\exclude{
	\subsection{Entropy ELBOs for Diagonal and Isotropic Gaussians}
	
	HERE WE STATE:
	
	ELBO FOR DIAGONAL GAUSSIANS
	
	ELBO FOR PROPORTIONAL 1 GAUSSIANS
}

\subsection{Properties of the New Objective}
\Cref{eq:analytic-elbo} represents the most general form of the novel objective.
In case of diagonal covariance matrices, $q^{(n)}(\vz) = \mathcal{N}\big(\vz \,|\, \vnu^{(n)}, \diag\big( \taunsqr{1} \ldots \taunsqr{H} \big) \big)$, \cref{eq:analytic-elbo} simplifies significantly as $\tr(\tilde W^\mathrm{T} \tilde W\TT^{(n)})$ becomes $\sum_h \taunsqr{h}$, and the log-determinant of $\TT^{(n)}$ is easy to compute (see \cref{app:derive-other-elbo} for the details).%

Also note that the entropy objective is fully compatible with amortized inference, i.e., if the variational parameters are functions (usually deep neural networks) of data points: $\vnu^{(n)} = \text{DNN}_{\nu}(\vx^{(n)}; \Phi)$ and $ \TT^{(n)} = \text{DNN}_{\TT}(\vx^{(n)}; \Phi)$.
In this context, the functions can map to diagonal covariance matrices (as is standard), to full rank covariance matrices, or to intermediate low-rank versions. The entropy-ELBO remains an analytic function in all these cases.
As a consequence, we can use standard gradient-based approaches for analytic functions to optimize all parameters. Without an analytic ELBO, sampling-based estimation of integrals and the reparameterization trick, or similar approaches to estimate ELBO gradients are required (Sec.\,\ref{sec:num} and \cref{app:amortized-learning} for details and experiments).

\subsection{Entropy Annealing}
\label{subsec:EntropyAnnealingDef}
Direct optimizations of ELBO objectives often result in locally optimal solutions. This observation is a main motivation to use {\em annealed} versions of the ELBO objective. 
A very prominent example is $\beta$-annealing \citep{higgins2017beta, huang_improving_2018}. 
In $\beta$-annealing, the  KL-divergence term is weighted: ${\cal L}(\Phi, \Theta) = \int q_\Phi(\vz)\log p_\Theta(\vx|\vz) \d \vz - \beta \DKL{q_\Phi(\vz)}{p_\Theta(\vz)}$.\footnote{In terms of entropies the KL-divergence corresponds -- at optimality -- to the gap between prior entropy and average variational entropy.}
The here derived entropy-based ELBOs %
invite to new types of annealing. As all terms of the ELBO are of the same principled type (i.e., entropy) it is straightforward
to reweight the entropy contributions for annealing. The annealed objective thus becomes:
\begin{align}
  \begin{split}
	\HELBO_{\gamma, \delta}(\Phi, \Theta) = 
	&\frac{1}{N} \sum_n \mathcal{H}[q_\Phi^{(n)}(\vz)] \\
	&- \gamma \, \mathcal{H}[p_\Theta(\vz)] 
	- \delta \, \mathcal{H}[p_\Theta(\vx | \vz)]  \enspace . \label{eq:three-entropies-annealing}
   \end{split}
\end{align}
Notice that we can effectively anneal equivalently to ${\beta}$-annealing by setting $\delta = \frac{1}{\beta}$ and $\gamma=1$. By using $\gamma=\delta\geq{}1$, we recover {\em energy tempering} (a.k.a., $\alpha$-annealing) \citep{katahira_deterministic_2008, huang_improving_2018}. \Cref{eq:three-entropies-annealing} suggests a third alternative using $\gamma\geq{}1$ and $\delta=1$, which we will call {\em prior annealing}.
\subsection{Relation to \texorpdfstring{$l_1$}{ℓ₁}-Sparse Coding}
\label{subsec:classicalSC}

\cref{eq:analytic-elbo} as well as the ELBOs for less general Gaussian distributions (see \cref{app:derive-other-elbo}), represent analytic learning objectives for sparse coding.
Therefore, it may be of interest to study the relation of entropy-ELBOs to objectives for standard $l_1$ sparse coding, which are likewise analytic functions. And $l_1$-objectives have extensively been researched \citep[][]{daubechies_iterative_2004, lee_efficient_2006, beck_fast_2009, gregor_learning_2010, hastie_statistical_2015}.
At first sight, the similarity between entropy-ELBOs and $l_1$-objectives does not seem to go very far because the intricate ELBO in \cref{eq:analytic-elbo} seems very different from objectives like \cref{EqnSCOp}.
At closer inspection, the similarity to $l_1$-objectives is higher than it first seems, however.
In this context, consider the entropy-based ELBO for Gaussian distributions with diagonal covariance matrix (derived in \cref{app:derive-other-elbo}). If we use an annealed version in analogy to \cref{eq:three-entropies-annealing}, the resulting 
objective function is given by \cref{eqn:helbodiag-anneal} in the appendix. We set $\delta=1$ to use prior annealing.
If we now focus on the optimization of variational parameters $\vnu^{(n)}$, then just the latter two of the three entropies are relevant (the first is independent of $\vnu^{(n)}$). Removing constant terms of these remaining entropies then results in the following objective for $\vnu^{(n)}$ that has to be minimized:\footnote{Minimization is the convention for $l_1$-objectives.}
\begin{align}
\underbrace{\frac{D}{2} \log \big( \sigma_{\mathrm{opt}}^2(\Phi,\Wt) \big)}_{\mathrm{reconstruction}} +\,\gamma\,\underbrace{\sum_{h=1}^H \log\big(\lambda^{\mathrm{opt}}_h(\Phi) \big)}_{\mathrm{sparsity}} \,.
\label{eqn:lonelocal}
\end{align}
Considering the result of \cref{theo:optimal-parameters} for $\sigma_{\mathrm{opt}}^2(\Phi,\Wt)$ for %
diagonal covariances, we obtain
\begin{align}
\sigma_{\mathrm{opt}}^2(\Phi,\Wt) \,
  &= \,\frac{1}{D}\frac{1}{N} \sum_{n=1}^N \|\Wt\vnu^{(n)}-\vxn)\|^2 + \frac{H}{D}\, \bar{\tau}^2, \nonumber
\end{align}
where $\bar{\tau}^2$ is just the average of $\taunsqr{h}$ across data points and latent dimensions (see \cref{eq:tau_avg}). %
Hence, the first term of \cref{eqn:lonelocal} is (a monotonic function of) a mean squared reconstruction error. Using again \cref{theo:optimal-parameters}, this time for $\lambda^{\mathrm{opt}}_h(\Phi)$, we can now more closely inspect the second term. For this, we define a smoothed magnitude function $|\cdot|^*$ using
the function ${\cal M}(a)$ of \cref{eq:DefFuncM} such that the solutions for the prior parameters become:
\begin{align}
	\lambda^{\mathrm{opt}}_h(\Phi) = \frac{1}{N} \sum_n \big|\nu_h^{(n)}\big|^\star\ \mbox{with}\ \big|\nu_h^{(n)}\big|^\star = \tau_h^{(n)} {\cal M}\Big( \frac{\nu_h^{(n)}}{\tau_h^{(n)}} \Big). \nonumber
\end{align}
Observe that for small $\tau_h^{(n)}$ compared to $\nu_h^{(n)}$, we indeed obtain that $\big|\nu_h^{(n)}\big|^\star \approx \big|\nu_h^{(n)}\big|$, see \cref{app:m-function}. 
Hence, the second term of \cref{eqn:lonelocal} penalizes large values of $\nu_h^{(n)}$ according to a (logarithmic) $l_1$ sparsity penalty.

Taking gradients w.r.t.\ $\vnu^{(n)}$ of the objective in \cref{eqn:lonelocal} makes the similarity to classical $l_1$-objectives like \cref{EqnSCOp} still more salient because the logarithms disappear as well as the $\bar{\tau}^2$-offset in $\sigma_{\mathrm{opt}}^2(\Phi,\Wt)$. However, gradients of the reconstruction and sparsity term will be weighted by $1/\sigma_{\mathrm{opt}}^2(\Phi,\Wt)$ and $1/\lambda^{\mathrm{opt}}_h(\Phi)$, respectively. %

In the next section, we will use different values of $\gamma\geq{}1$, i.e., prior annealing to investigate resulting encodings, e.g., for image patches.
This will allow us to numerically investigate the similarity of (annealed) entropy-ELBOs and $l_1$-objectives.
From a theoretical perspective, a notable difference to standard $l_1$-objectives is, however, that \cref{eqn:lonelocal} is derived from the standard ELBO objective. That ELBO is itself an approximation for maximum likelihood parameter estimation. As a consequence, the optimal $\gamma$ is known in our case ($\gamma=1$), while for classical $l_1$-sparse coding the weighting factor is an important free parameter 
that has to be tuned.

\section{EXPERIMENTS}
\label{sec:num}
We use numerical experiments to verify the feasibility of the novel entropy-based objectives. Our main interests will be convergence speed and insights into the effects of entropy annealing on sparsity. The source code for the experiments is available at \url{https://github.com/Learning-with-Entropies/sparse-coding.git}
\subsection{Verification on Artificial Data}
We first investigated learning based on entropy-based ELBOs using artificial data with ground-truth. Data consisted of $N=1000$ data points with horizontal and vertical bars \citep{foldiak_forming_1990,Hoyer2002} on a $5\times5$ grid ($D=25$). We used (unamortized) Gaussian variational distributions with full covariance matrix (i.e., variational parameters for all data samples) to optimize the model given in \cref{EqnPSC2} with $H=10$. Using \cref{eq:analytic-elbo}, the variational parameters were then optimized jointly with $\tilde W \in \RRR^{D \times H}$ by applying L-BFGS \citep{liu_limited_1989}, which is readily available in PyTorch \citep{paszke2017automatic}. After convergence, each recovered generative field (GF; i.e., each column of $\tilde W$) contained one bar (\cref{fig:bars-dataset}). 
Convergence was fast: after approximately 100 L-BFGS calls 
to optimize the entropy-ELBO, ELBO values were already close to values of the ELBO for (ground-truth) generating parameters (for details, see \cref{app:bars-dataset}). 

\begin{figure}[hbt]
	\centering
	\begin{minipage}[b]{.23\textwidth}
		\centering
		\includegraphics[width=\textwidth, trim={0.5cm 0.5cm 0.5cm 0.5cm}, clip]{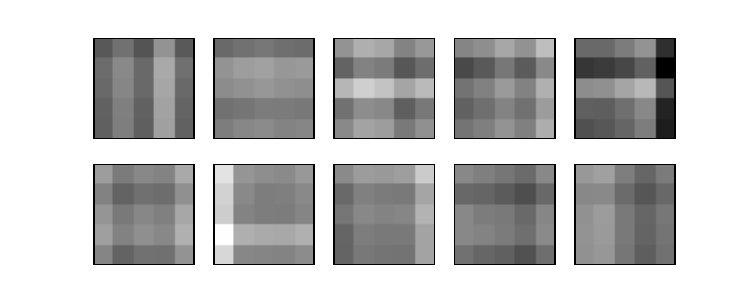}
		\subcaption{Training data samples}\label{fig:bars:1}
	\end{minipage}
	\hfill
	\begin{minipage}[b]{.23\textwidth}
		\centering
		\includegraphics[width=\textwidth, trim={0.5cm 0.5cm 0.5cm 0.5cm}, clip]{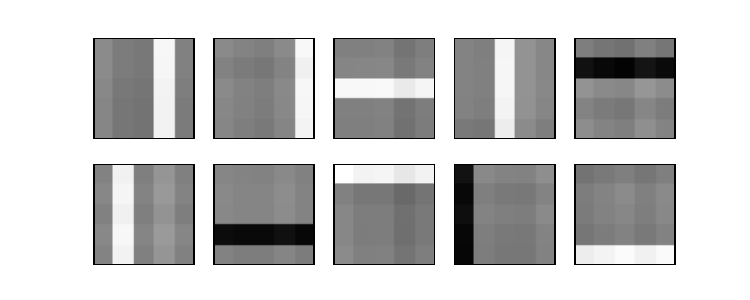}
		\subcaption{Learned generative fields}\label{fig:bars:2}
	\end{minipage}
	\caption{\textbf{Artificial sparse bars dataset.} (a)~Example data which is constructed by Laplace-distributed activation of horizontal and vertical bars. %
    (b)~Optimizing $\HELBO$ recovers the bars and their activations (up to signs of generative fields).\vspace{-2.2ex} %
    }
	\label{fig:bars-dataset}
\end{figure}

\subsection{Natural Image Patches, Sparsity and Entropy Annealing}
After verification on artificial data, we investigated entropy-ELBOs using
the presumably most standard application of sparse coding models: natural image patches.
Learning sparse codes for image patches based on \cref{EqnPSC} is also the most common approach to explain neuronal receptive fields in cortical area V1 \citep{olshausen_emergence_1996}. Originally estimated with MAP approximation, further extensions were developed to employ, e.g., VAE-style training with stochastic estimation of the ELBO gradient using the “reparameterization trick” \citep[e.g.][]{barello_sparse_coding_2018, tonolini_variational_2020}. Here we used analytic entropy-ELBOs to optimize the model in \cref{EqnPSC2} on whitened images \citep{olshausen_emergence_1996} with 
$N=204\,800$, $D=16\times16$ and $H=100$ and $400$ (\cref{app:overcomplete-basis} for experiments with $H=400$).
We explored different versions of the entropy-ELBO, \cref{eq:analytic-elbo}, in order to investigate the effect of entropy-based annealing and to compare it to amortized optimization. To allow for sufficient computational efficiency, we used variational distributions with diagonal (see \cref{eqn:helbodiag}) or low-rank covariance matrices.
Concretely, we used (A)~an entropy-ELBO without annealing; (B)~an entropy-ELBO (\cref{eqn:helbodiag-anneal}) with prior annealing ($\gamma\geq{}1$, $\delta=1$); (C)~an entropy-ELBO using amortized variational distributions with diagonal covariance; and (D)~the same amortized entropy-ELBO as in (C) but with low-rank approximation of the covariance matrices.
For (D) we also use prior annealing ($\gamma\geq{}1$, $\delta=1$).
For (A) and (B) we used EM-like updates: for every minibatch, we optimized variational parameters with L-BFGS and then took a gradient step to update $\Wt$. For (C) and (D) neural networks were used to map data to means and to (diagonal or low-rank) covariances, and for parameter optimization we used Adam-based gradient ascent provided by the standard PyTorch implementation \citep{paszke2017automatic}.

For all four different versions 
of entropy-ELBOs, optimization ultimately resulted in the familiar Gabor-like generative fields: \Cref{fig:compare-optimization-4} shows ELBO-optimization for the different versions, \cref{app:compare-annealing} shows final GFs for $H=100$ and $H=400$.
However, salient quantitative differences could be observed (see \cref{fig:compare-optimization-4} and \cref{fig:image-patches-dataset}). Prior annealing of non-amortized entropy-ELBOs resulted in the fastest convergence and the highest ELBO values. Without annealing, non-amortized entropy-ELBOs finally resulted in very similar ELBO values but required longer to converge.
Using also diagonal encoder covariances but amortized optimization, entropy-ELBOs converged more slowly and showed lower final ELBO values (see \cref{fig:compare-optimization-4}). Final ELBOs improved using low-rank covariance approximations and prior annealing (again \cref{fig:compare-optimization-4}). Low-rank covariances had a stronger effect on improvements \citep[][for a related analysis]{Wipf2023} than prior annealing. 
In general, we observed that annealing
has a comparably smaller effect for the amortized ELBO versions, which may be due to an interaction between annealing and standard Adam optimizers (\cref{app:amortized-learning} for details). 

\begin{figure}[t!]
  \centering
\includegraphics[width=0.99\linewidth, trim={-0.2cm 1.0cm -0.2cm 0.5cm}]{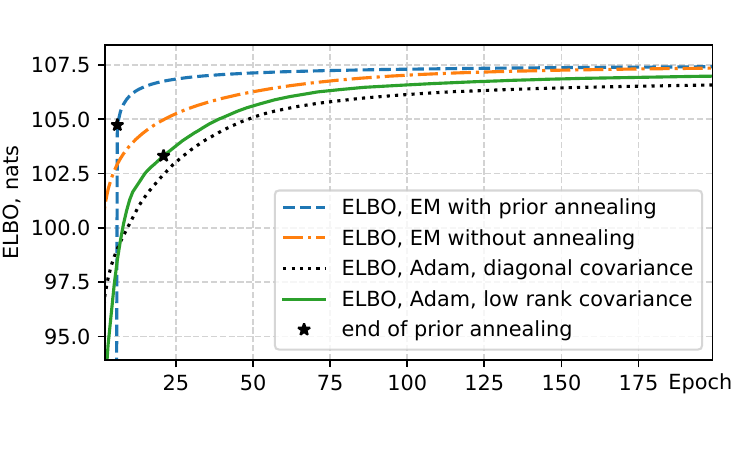}
\caption{\textbf{Optimization of entropy-ELBOs}. Two non-amortized optimizations and two amortized optimizations are shown. Two optimizations use annealing.\vspace{-1ex}}
  \label{fig:compare-optimization-4}
\end{figure} 
\exclude{
\begin{figure}[Ht]
		\centering
		\includegraphics[width=\linewidth, trim={0 0.6cm 0 0cm}, clip]{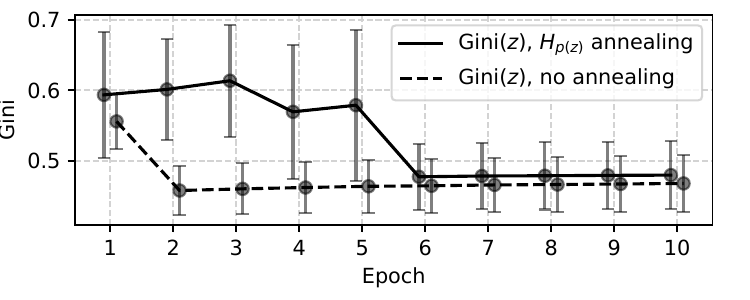}
	\caption{\textbf{Prior anealing} on natural image patches dataset.
            Gini coefficients (mean $\pm$SD bars) of the latent codes (\cref{fig:image-patches-dataset} for example generative fields.}
	\label{fig:image-patches-ELBO}
\end{figure}
}
\begin{table}[b]
	\centering
	\caption{{\bf Different entropy annealings}. No annealing (top), prior annealing (middle) and $\beta$-annealing (bottom) are compared (\cref{{app:compare-annealing}} for details).} 
 \resizebox{\columnwidth}{!}{%
	\begin{tabular}{ll rrr r r}
		\toprule
         \multicolumn{2}{l}{\textbf{ANNEALING}}
        & $\mathcal{H}[p_\Theta(\vz)]$ & $\mathcal{H}[p_\Theta(\vx|\vz)]$ & $\mathcal{H}[q_\Phi(\vz)]$ & ELBO & Gini$(\vz)\pm$SD \\
		\midrule 
		\multicolumn{2}{l}{No annealing} &   32.40 & -234.70 &  -98.97 &  103.33 & $0.47\pm0.04$ \\
		\midrule 
		\multirow{3}{*}{$\mathcal{H}[p_\Theta(\vz)]$} 
		& $\gamma=10.0$ & -218.72 &   10.71 & -365.22 & -157.21 & $0.59\pm0.09$ \\
		& $\gamma=2.0$  &  -17.29 & -144.81 & -112.68 &   49.41 & $0.58\pm0.10$ \\
		& $\gamma=1.0$  &   29.65 & -235.50 &  -99.67 &  106.18 & $0.48\pm0.05$ \\
		\midrule 
		\multirow{3}{*}{$\mathcal{H}[p_\Theta(\vx | \vz)]$} 
		& $\delta=0.14$   & -228.33 &   30.59 & -234.02 &  -36.27 & $0.55\pm0.03$  \\
		& $\delta=0.50$   & -17.77 &  -115.24 & -72.77 &   60.24 & $0.62\pm0.04$  \\
		& $\delta=1.0$   &   29.74 & -234.98 &  -99.43 &  105.80 & $0.47\pm0.05$  \\
		\bottomrule
	\end{tabular}
 }
	\label{tab:annealing}
\end{table}

Next, we were interested in the different types of annealing suggested by entropy-ELBOs, see \cref{eq:three-entropies-annealing}. Prior annealing ($\gamma\geq{}1$, $\delta=1$) very quickly resulted in a sparse encoding and localized GFs (see \cref{fig:image-patches-dataset}). This is consistent with the role of $\gamma$ in weighting a sparsity penalty term, see \cref{eqn:lonelocal}. Hence, it is {\em prior annealing} with $\gamma\geq{}1$ which is analogous to high weights for the 
sparsity penalty in $l_1$ sparse coding. That prior annealing results in sparser encodings is also confirmed when, e.g., using the Gini index \citep{niall_measures-of-sparsity_2009} as a measure of sparsity (both Gini values and ELBO values are high, see \cref{tab:annealing}).
In contrast, $\beta$-annealing (used in $\beta$-VAEs), represents a type of regularization different from prior annealing resulting in  
much less localized GFs (\cref{{app:compare-annealing}} and \cref{fig:beta-annealing-all}). %

\exclude{
\ \\

\ \\
-------------------------------------------------
\ \\
\ \\
Hence, weighting the prior entropy more strongly corresponds to 

scemes. Prior annealing did result 

Annealing we did not observe to improve  

is very fast for inference, however. 

of optimization based on the different entropy-ELBOs. All generative fields do ultimately converge 

The latter uses an X-X-X-X ResNet DNN (APP XX for details). Fig. XX shows a comparison of non-annealed optimization ()

\ \\
OLD TEXT\\
\ \\
(C)~an entropy-ELBO with $\beta$-VAE-like annealing ($\gamma=\delta\geq{}1$ TRUE?); and
sWe used 

FORMULA 160

Concretely we investigated: (A)~an entropy ELBO with uncorrelated Gaussian distributions and no annealing; (B)~an entropy ELBO

mainly three different versions of the analytic ELBO: We used (A) an entropy

full covariance. However, we explored one entropy-ELBO
with amortized Gaussians and low-rank approximations.

For efficient optimization, we used the versio

\cref{eq:analytic-elbo} 
We used the model \cref{EqnPSC2} with up to $H=400$ latents. For optimization, we used the analytic ELBO \cref{eq:analytic-elbo} but with diagonal covariances for the variational distributions to be numerically more efficient (App.\,XXX).

Here we apply the entropy-based ELBO to

We here use $N=204\,800$ whitened natural images \citep{olshausen_sparse_nodate}
with $D=16\times16$ patches sampled at random positions.

We used $N=204\,800$ of 16$\times$16 randomly sampled image patches and

with $H=100$ dimensions 

. We used objective \cref{eq:analytic-elbo} but with diagonal covariances for the variational distributions (App.\,XXX). \mycomment{Pointer to appendix!}
The objective was iteratively updated:
first, for a batch of samples, we optimized the variational parameters (mean and diagonal covariance for each data point),
then we updated the $\tilde W$ matrix with a standard gradient ascent step. 
\mycomment{How were the variational parameters optimized? Also gradient? Please state.}
After every epoch, we computed the entropy-ELBOs by fixing $\tilde W$ and optimized the variational distributions for all data points. %
\cref{fig:image-patches-ELBO,fig:image-patches-dataset} shows the generative fields for a run with $H=100$ (see App. XXX for runs with $H=400$).  
ELBO values quickly increased with values after two eqochs being alread very close to those reached finally after ... epochs. We observed this behavior in all runs.

close to their final  
and $\tilde W$ usually did not take more than two epochs.
}

\exclude{
We found a positive effect of the prior entropy weighting and annealing on the sparseness and the localization of the estimated generative fields, see \cref{fig:image-patches-ELBO,fig:image-patches-dataset}.
By reweighing the prior entropy we could control the presence of dense generative fields to capture high-frequency textures, and the latent code sparseness, see \cref{tab:annealing}.

\begin{figure}[ht]
		\centering
		\includegraphics[width=\linewidth, trim={0 0.6cm 0 0cm}, clip]{images/vanhateren/olshausen-100/mini-vanhateren-gini-notitle.pdf}
		\includegraphics[width=\linewidth]{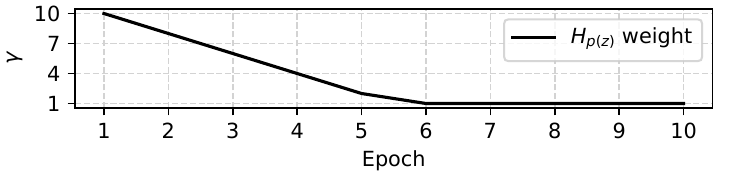}
	
	\caption{\textbf{Prior anealing} on natural image patches dataset.
            Gini coefficients (mean $\pm$SD bars) of the latent codes (top), and the prior annealing schedule (bottom). 
            See \cref{fig:image-patches-dataset} for the learned generative fields.}
	\label{fig:image-patches-ELBO}
\end{figure}
We use the Gini coefficient as a measure of sparsity \citep{niall_measures-of-sparsity_2009}.
Such small numerical difference of ELBO for visually different generative fields may relate to the analysis of linear filter performance for modeling of natural images \citep{eichhorn_natural_2009}, where orientation-selective filters learned with ICA provided less than 5\% improvement in average log-loss compared to PCA.

}

\begin{figure}[!ht]
	\begin{minipage}[b]{.23\textwidth}
		\centering
		\includegraphics[width=\linewidth, trim={2cm 2.5cm 2cm 2.5cm}, clip]{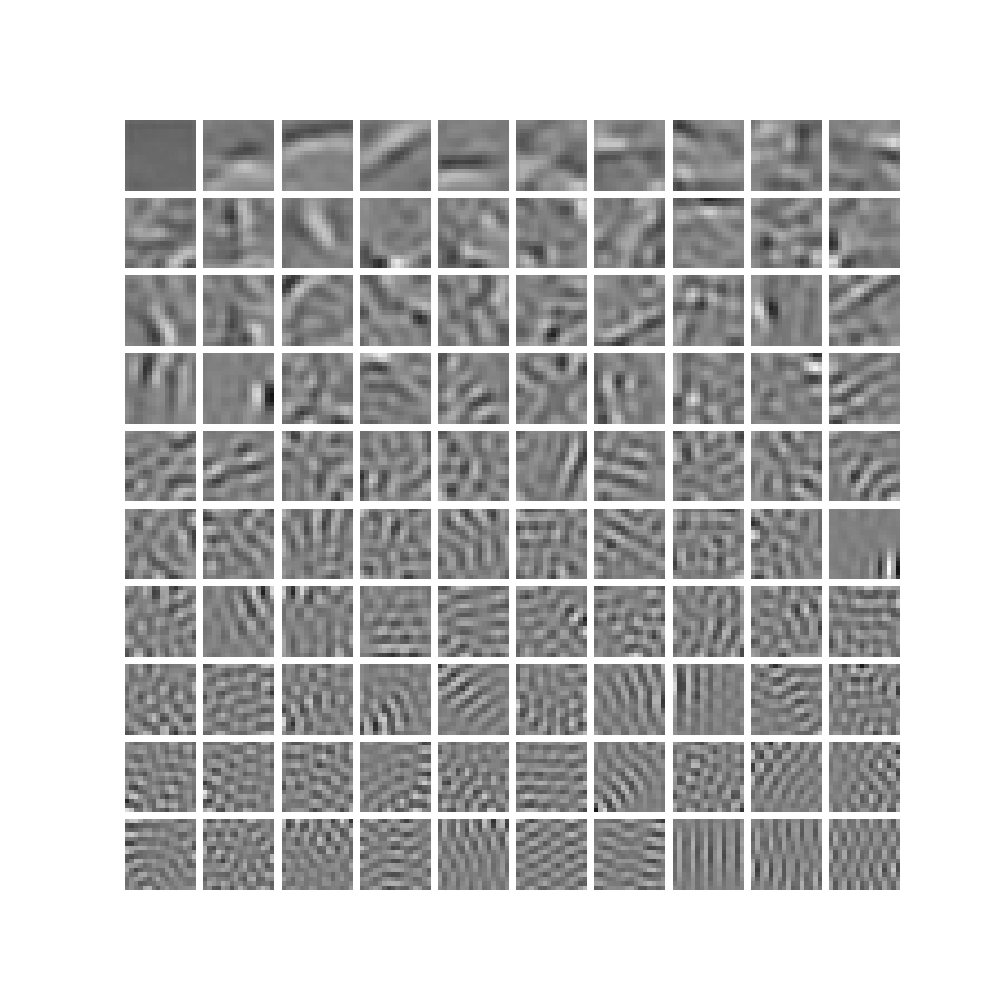}
		\subcaption{No annealing, epoch 10}\label{fig:natural:2}
	\end{minipage}
	\begin{minipage}[b]{.23\textwidth}
		\centering
		\includegraphics[width=\linewidth, trim={2cm 2.5cm 2cm 2.5cm}, clip]{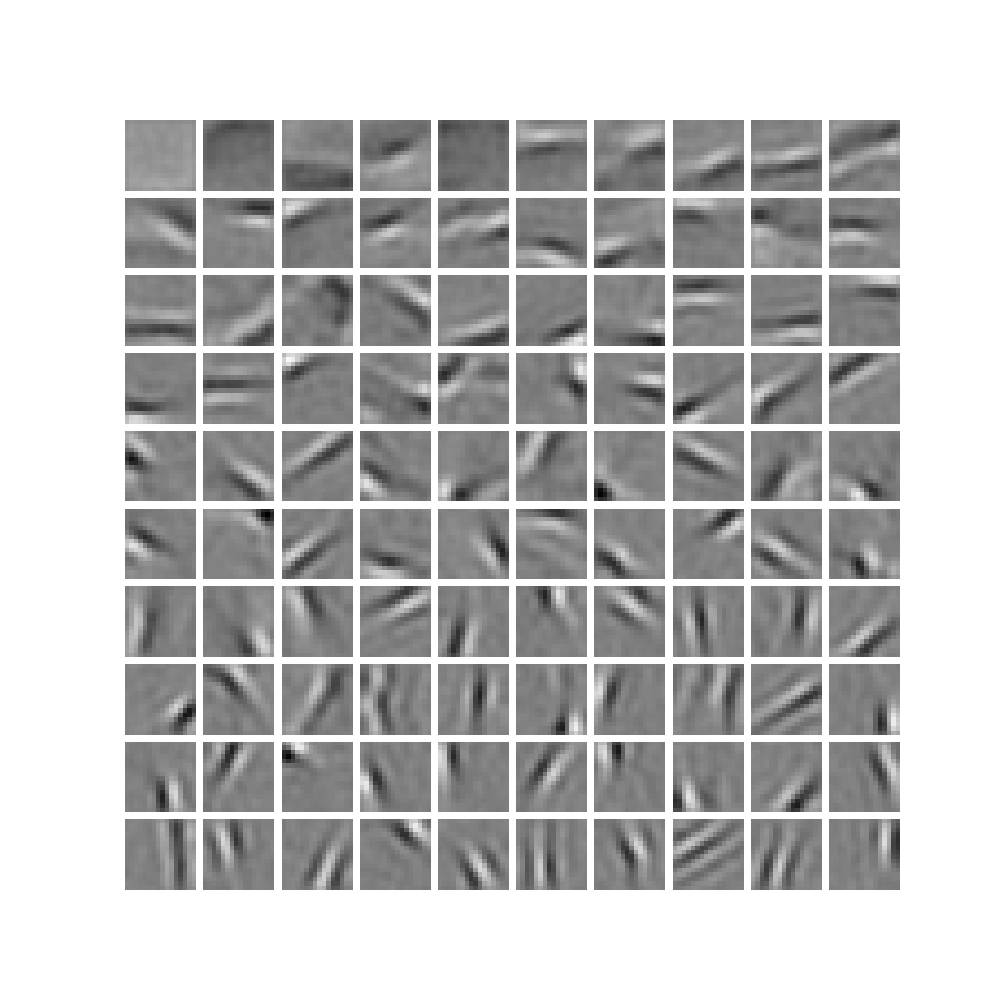} 
		\subcaption{Prior annealing, epoch 1}\label{fig:natural:3}
	\end{minipage} \\
	\begin{minipage}[b]{.23\textwidth}
		\centering
		\includegraphics[width=\linewidth, trim={2cm 2.5cm 2cm 2.5cm}, clip]{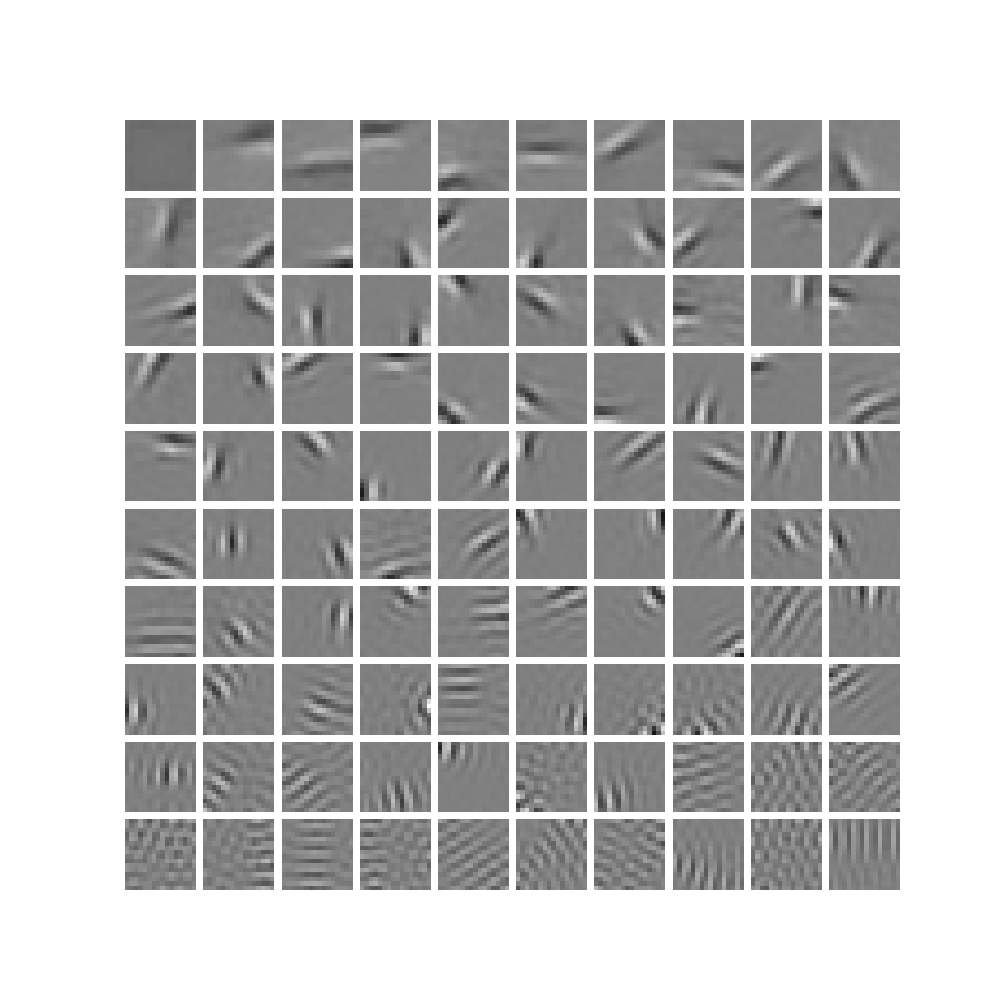}
		\subcaption{Prior annealing, epoch 10}\label{fig:natural:4}
	\end{minipage}
	\begin{minipage}[b]{.23\textwidth}
		\centering
		\includegraphics[width=\linewidth, trim={2cm 2.5cm 2cm 2.5cm}, clip]{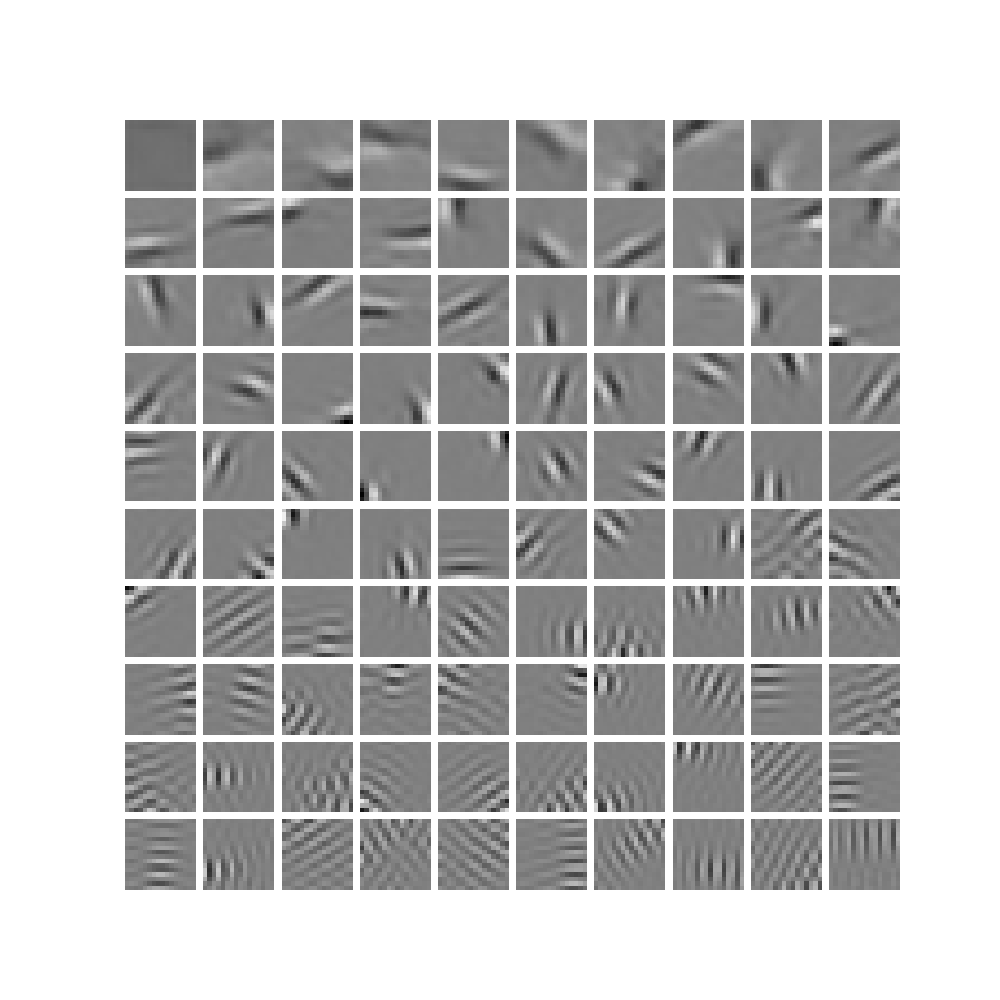}
		\subcaption{Amortized, epoch 200}\label{fig:natural:5}
	\end{minipage}
	\caption{\textbf{Learned generative fields} on natural image patches.
        Without annealing, the convergence is slow (a).
        With the prior entropy annealing, even after one epoch, we observe the familiar localized Gabor filters (b). The final GFs (c) comprise a set of Gabors and higher frequency texture-like images. Learning with amortized posterior results in similar GFs (d). 
        }
	\label{fig:image-patches-dataset}
\end{figure}

Finally, our analytic objective also allows for easy estimation of ELBO values for sparse coding models optimized using standard MAP-based approaches.
To show this, we used the original “sparsenet” code of Olshausen and Field 
\citep{olshausen_sparse_nodate}. After optimization, we used the resulting $W$ matrix, %
normalized its columns, and optimized only the variational parameters of Gaussians (means and diagonal covariances). %
The obtained ELBO values of $94.719$ with Gini$(z)=0.464\pm0.041$ indicates underfitting due to a manually selected weighting of the sparsity penalty. %

\exclude{
\subsection{Entropy Annealing}
\label{subsec:EntropyAnnealingLinear}
We ran experiments on the natural image patches dataset used by \citep{olshausen_sparse_nodate,olshausen_sparse_1997} and found substantial qualitative differences in the learned generative fields: prior annealing results in more localized Gabor filters, while likelihood entropy annealing simply turns off latent dimensions by driving $\lambda_h$ to small values with unstructured noisy generative fields for them (see \cref{app:compare-annealing}). Prior entropy annealing also produced the best Gini index and the highest ELBO values (see \cref{tab:annealing}).
\mycomment{Use the weighting factors $\gamma$ and $\delta$ and reference to the equation 16 here!}
}

\exclude{ WE PLACE CONSTANTS IN FRONT OF THE ENTROPIES
	
	CITE: BETA-VAE
	and
	Improving Explorability in Variational Inference with Annealed Variational Objectives, Huang et al., NeurIPS 2018
}

\exclude{
	\subsection{Entropy ELBOs for Diagonal and Isotropic Gaussians}
	
	HERE WE STATE:
	
	ELBO FOR DIAGONAL GAUSSIANS
	
	ELBO FOR PROPORTIONAL 1 GAUSSIANS
}

\exclude{
\subsection{Posterior Collapse Shrinks Prior Scales}
\label{subsec:PostCollapseShrinking}
\mycomment{If we need the space: Move this section to appendix and just mention collapse with one sentence or so.}
Contrary to the Gaussian prior and Gaussian variational distributions, the KL-divergence term of the ELBO \cref{eq:ELBO_FF} can not vanish if the prior is Laplace and the variational distribution is Gaussian. For the case with learnable parameters for the Laplace prior, the KL-divergence may get arbitrarily close to zero by assigning
very small values to the $\lambda_h$. We observed this effect in some experiments when we learned overcomplete sparse dictionaries for natural image patches (see \cref{app:overcomplete-basis}).
}

\section{DISCUSSION}
Our main contributions are \cref{theo:laplace-prior,theo:analytic-elbo,theo:optimal-parameters}. Taken together, these three theorems ensured that the here derived analytic objective in \cref{eq:analytic-elbo} can be used to optimize model parameters of standard probabilistic sparse coding. 
Apart from MAP-based approximations with known shortcomings for uncertainty encoding (cf. \cref{sec:Intro}), there are many other approaches for sparse coding that maintain non-trivial posterior approximations. It could, of course, be argued that for those approaches at least the optimization algorithms (if not the objectives) are described by analytic or closed-form equations. Examples are work by \citet[][]{seeger_bayesian_2008}, who used expectation propagation to derive a learning algorithm, or by \citet[][]{berkes_sparsity_2007}, who used a student-t prior and Gaussian scale mixture ideas to facilitate variational optimization.
The approach by \citet[][]{sheikh_truncated_2014} also provides an analytic objective for probabilistic sparse coding but at the cost of a combinatorial discrete optimization. 
We also state work by \citet{ChallisBarber2013}. The contribution focuses theoretically and empirically on fully Bayesian models in which weights are sparse, and the optimization bound is therefore different. But models and bound are closely related to probabilistic sparse coding (we elaborate in \cref{app:derive-analytic-elbo}). 

In contrast to these and other previous approaches, we here remain with the most standard choices for probabilistic sparse
coding. And it is for this setting that we show the ELBO to have an analytic solution. Concretely, we
remain (A)~with the (by far) most standard model, \cref{EqnPSC}; we use (B)~the presumably most standard
optimization framework (ELBOs for approximate maximum likelihood); and we use (C)~the most standard
posterior approximations (Gaussians). The here derived objective, presented in \cref{eq:analytic-elbo}, then shows that
all (high dimensional) integrals that emerge can be solved analytically. To the knowledge of the authors,
this has previously not been shown and/or empirically used.
However, we remark the similarity to problems emerging for probabilistic inference using sparse weights \citep[][]{ChallisBarber2013} (\cref{app:derive-analytic-elbo}), and we remark that analytic solutions for standard ELBOs can also be derived without knowledge of entropy convergence (see \cref{app:derive-classical-elbo}).

The results here derived do, notably, apply very generally. We have, for instance, numerically verified that potentially intricate deep neural networks (DNNs) can be used as encoders. The analytic objective then represents a deterministic DNN objective, and such objectives can conveniently be optimized with standard DNN tools. 
\Cref{eq:three-entropies} of \Cref{theo:laplace-prior} 
is still more general by applying for any decoder 
(linear or non-linear) with Gaussian observables.
\Cref{theo:laplace-prior} thus extends to sparse VAEs which are of recent interest \citep[][]{fallah_variational_2022,drefs_direct_2023,chen_encouraging_2023}. Future work can consequently investigate the here presented approaches like entropy annealing for such deep sparse coding models.

Conceptually maybe most relevantly, we here for the first time investigated how an ELBO objective can be reformulated as a solely entropy-based objective.
From a theoretical perspective, entropies are more deeply rooted in the foundations of probabilistic machine learning, mathematical statistics, and information theory. Furthermore, for the class of distributions usually used to define generative models (exponential family, constant base measure), entropies are closed-form and are equipped with potentially convenient properties (via their log-partition function). 
Also, the derivatives of entropies (that are used for learning) have similarly convenient properties, and future work can link
those to information geometry.
Entropy convergence has previously only been considered for analysis \citep[][]{lucke_henniges_closed-form_2012,DammEtAl2023}, and, so far, it has been unclear if or how entropy convergence
can be used for learning. In this work, we provided the first demonstration that solely entropy-based objectives {\em can} be used for learning, and there is no principled
obstacle to extending this general approach to further generative models in the future.

\subsubsection*{Acknowledgments}
This work was funded by the German Research Foundation (DFG) within the priority program SPP~2298 “Theoretical Foundations of Deep Learning” - project 464104047 (FI~2583/1-1 and LU~1196/9-1).
Asja \mbox{Fischer} also acknowledges support by the DFG under Germany’s Excellence Strategy – EXC-2092 CASA – 390781972.

\exclude{
OLD TEXT\\

Future work will extend this general approach to further generative models.

Entropies have theoretical advantages: the entropies of the most commonly used distributions (exponential family) are available in closed-form, and their derivatives have convenient properties. Entropies are also

are  

such that their summands solely consist of entropies. ADVANTAGES OF ENTROPIES ...

Convergence to entropies
applies not

also applies more generally:
        for any deterministic mapping from latents to observables, i.e., sparse coding using DNNs also for the decoding model. Sparse VAEs are of recent interest \citep[][]{drefs_direct_2023,RozellXXXX,otherECML}. Future work can consequently investigate here considered approaches such as entropy annealing also for deep sparse coding models.

        Other sparse priors can also be considered in this context. 
    - If such priors are of the exponential family of distributions, convergence to entropies carries over under commonly fulfilled conditions \citep[][]{Lucke2022Convergence}. (maybe make)

TO BE CONTINUED

FOR SPARSE CODING: 
- weighting of sparsity term is not a free parameter but can ultimately be derived from maximum likelihood parameter estimation and Gaussian posterior approximations
- establishes link
- convex optimization approaches developed for $l_1$-sparse coding can be used

\ \\

FOR OTHER SPARSE MODELS
- why sparse coding still interesting
- general principle still remains valid
- our results or part of our results are valid for different extensions
     - theorems 1-3 already cover any encoding model including potentially intricate DNNs
    - furthermore, Theorem 1 (convergence to entropies) also applies more generally:
        for any deterministic mapping from latents to observables, i.e., sparse coding using DNNs also for the decoding model. Sparse VAEs are of recent interest \citep[][]{drefs_direct_2023,RozellXXXX,otherECML}. Future work can consequently investigate here considered approaches such as entropy-based annealing also for deep sparse coding models.
        Other sparse priors can also be considered in this context. 
    - If such priors are of the exponential family of distributions, convergence to entropies carries over under commonly fulfilled conditions \citep[][]{Lucke2022Convergence}. (maybe make)

ENTROPY ELBOS
- maybe most relevantly, we here for the first time investigated how ELBOs can in principle be reformulated such that their summands solely consist of entropies. ADVANTAGES OF ENTROPIES ... From a theoretical perspective, entropies are more deeply rooted in the foundations of probabilistic machine learning, mathematical statistics and information theory. Furthermore, for the class of distributions usually used to define generative models (exponential family, constant base measure), entropies are closed-form and are equipped with potentially convenient properties defined by their log-partition function. Similar convenient properties can be shown for derivatives of entropies, which can be convenient for learning. 
Entropy convergence has previously only been considered for analysis \citep[][]{LuckeHenniges2011,DammEtAl2023}, and, so far, it has been unclear if or how entropy convergence
can be used for learning. In this work we provided the first demonstration that solely entropy-based objectives can be defined, and that they can be used for learning. Future work will extend this general approach to further generative models.

Those and other approaches are different from the  

it could be argued that 

it could, of course, be argued that
other forms of sparse coding 

closed-form learning approaches for sparse coding have been suggested. For instance, SEEGER used an expectation propagation for approximate learning. Also approaches using Gaussian scale mixture ideas (Berkes, Turner, Sahani, NeurIPS 2008) use 

or ... derive closed-form algorithms. But all such approaches deliberately divert from standard settings and/or involve further approximations.

A distinguishing feature of the result this study provides that

In this paper we demonstrated that the ELBO for standard probabilistic sparse coding, given in \cref{EqnELBO}, converges to a sum of entropies and hence admits an analytic solution (\cref{theo:laplace-prior}).
In particular, the ELBO values at optimality solely depend on the prior scales, observation noise and covariance of the variational distributions.
ADD MORE THEORETICAL RELEVANCE + RELATION TO $l_1$ SPARSE CODING.
We utilize this insight to derive an entropy-based objective by invoking analytic solutions for scales of the Laplace prior and the observation noise (\cref{theo:analytic-elbo}).

The entropy-based ELBO comes with additional possibilities to steer optimization in favorable directions. We demonstrate the utility of entropy-based learning by annealing the prior entropy which results in better likelihood and increased quality of the induced generative fields.
}

\bibliographystyle{abbrvnat} 
\bibliography{references.bib}

\section*{Checklist}

 \begin{enumerate}

 \item For all models and algorithms presented, check if you include: %
 \begin{enumerate}
   \item A clear description of the mathematical setting, assumptions, algorithm, and/or model. [Yes]
   \item An analysis of the properties and complexity (time, space, sample size) of any algorithm. [Yes]
   \item (Optional) Anonymized source code, with specification of all dependencies, including external libraries. [Yes]
 \end{enumerate}

 \item For any theoretical claim, check if you include:
 \begin{enumerate}
   \item Statements of the full set of assumptions of all theoretical results. [Yes]
   \item Complete proofs of all theoretical results. [Yes]
   \item Clear explanations of any assumptions. [Yes]     
 \end{enumerate}

 \item For all figures and tables that present empirical results, check if you include:
 \begin{enumerate}
   \item The code, data, and instructions needed to reproduce the main experimental results (either in the supplemental material or as a URL). [Yes]
    \item All the training details (e.g., data splits, hyperparameters, how they were chosen). [Yes]
    \item A clear definition of the specific measure or statistics and error bars (e.g., with respect to the random seed after running experiments multiple times). [Yes]
    \item A description of the computing infrastructure used. (e.g., type of GPUs, internal cluster, or cloud provider). [Not Applicable]
 \end{enumerate}

 \item If you are using existing assets (e.g., code, data, models) or curating/releasing new assets, check if you include:
 \begin{enumerate}
   \item Citations of the creator if your work uses existing assets. [Yes]
   \item The license information of the assets, if applicable. [Not Applicable]
   \item New assets either in the supplemental material or as a URL, if applicable. [Not Applicable]
   \item Information about consent from data providers/curators. [Not Applicable]
   \item Discussion of sensible content if applicable, e.g., personally identifiable information or offensive content. [Not Applicable]
 \end{enumerate}

 \item If you used crowdsourcing or conducted research with human subjects, check if you include:
 \begin{enumerate}
   \item The full text of instructions given to participants and screenshots. [Not Applicable]
   \item Descriptions of potential participant risks, with links to Institutional Review Board (IRB) approvals if applicable. [Not Applicable]
   \item The estimated hourly wage paid to participants and the total amount spent on participant compensation. [Not Applicable]
 \end{enumerate}
\end{enumerate}

\appendix

\onecolumn
\thispagestyle{empty} %

{\hsize\textwidth
\linewidth \hsize \toptitlebar 
{\centering {\Large\bfseries Learning Sparse Codes with Entropy-Based ELBOs: \\ Supplementary Materials \par}}
\bottomtitlebar} %

\exclude{
The appendix is organized as follows. First, we present additional convergence criteria for entropy-convergence of ELBOs in \cref{app:add_convergence_criteria}. The derivations of \cref{theo:optimal-parameters} and the analytic entropy-based objective \cref{eq:analytic-elbo} is given in
\cref{app:ELBOs_for_Laplace_Prio_SC}, the proof of \cref{theo:analytic-elbo} in \cref{Theorem3:Proof}. Additional numerical results are presented in \cref{app:numerical_results}.
}

\startcontents[appendices]
\printcontents[appendices]{l}{1}{\section*{Organization of the Appendix}\setcounter{tocdepth}{2}}

\ \\
\ \\
\ \\
\section{CONVERGENCE CRITERIA AND PROOFS}
\label{app:add_convergence_criteria}
This section contains an additional theorem, which allows a quick check of whether a model ELBO possesses the convergence to entropies property. Additionally, we provide a more detailed proof of convergence to entropies sums for the sparse coding models using the general conditions derived in \citep{Lucke2022Convergence}.

\subsection{Factorization Criteria for Natural Parameters}
\label{app:FactorizationTheorem}

 Here, we provide a simple theorem that gives a set of \textit{sufficient} conditions, under which the ELBO converges to a sum of entropies.

\begin{theorem}
  \label{theor:factorization}
  Consider a model $p_\Theta(\vx, \vz)=p_\Theta(\vx | \vz, \theta)p_\Theta(\vz|\vlambda)$.
  We assume that prior and observable (likelihood) distribution belong to an exponential family with constant base measure.
  If the model can be stated with the following factorization
  of the natural parameters,
    \begin{align}
    p_\Theta(\vz | \vlambda) &= \frac{1}{Z(\vlambda)} \exp(-T^\mathrm{T}(\vz) \eta_{\vlambda}(\vlambda))) \enspace , \\
    p_\Theta(\vx | \vz, \theta) &= \frac{1}{Z(\theta)} \exp(-T^\mathrm{T}(\vx) (\eta_z(\vz) \odot \eta_{\theta}(\theta))) \enspace ,
  \end{align}
  and if the Jacobians 
  $\frac{\partial \eta_{\theta}(\theta)}{\partial \theta}$ 
  and 
  $\frac{\partial \eta_{\vlambda}(\vlambda)}{\partial \vlambda}$ 
  of the parameter mappings are invertible,
  then the model ELBO converges to the following sum of entropies:
  \begin{align}
    \HELBO
    = \frac{1}{N} \sum_{n=1}^N \mathcal{H}[q_\Phi^{(n)}(\vz)]
      - \mathcal{H}[ p_\Theta(\vz | \vlambda)] 
      - \mathcal{H}[ p_\Theta(\vx | \vz, \theta)]  \enspace .
  \end{align}
\end{theorem}

We emphasize that, unlike the general conditions \citep{Lucke2022Convergence}, this theorem, although being more restrictive, allows us to not only \textit{check}, but also to easily \textit{construct} probability distributions for models that converge to entropies sums. If conditions from the \cref{theor:factorization} do not hold, one still has to check the more general conditions \citep{Lucke2022Convergence}.

\begin{proof}
Here we spell out the proof by introducing the above requirements to the model distributions and checking the convergence at stationary points. First, we write the ELBO of such models with approximate posterior $q_\Phi^{(n)}(\vz)$ and prove the convergence for the ${\ELBO_2(\Phi, \Theta)}$ term:

\begin{align*} %
        \textstyle  \mathcal{L}(\Phi, \Theta) = 
            \frac{1}{N} \sum_{n=1}^N \mathcal{H}&[q^{(n)}_{\Phi}(\vz)] 
            + \underbrace{ 
                    \textstyle\frac{1}{N} \sum_{n=1}^N \int q^{(n)}_{\Phi}(\vz) \log p_{\Theta}(\vz) d\vz
                }_{\ELBO_1(\Phi, \Theta)}
            + \underbrace{ 
                    \textstyle\frac{1}{N} \sum_{n=1}^N \int q^{(n)}_{\Phi}(\vz) \log p_{\Theta}(\vx^{(n)}|\vz) d\vz
                }_{\ELBO_2(\Phi, \Theta)} \enspace .
\end{align*}

We consider a class of distributions that belong to the exponential family with constant base measure $Z(\vz, \theta)$ and a factorizable negative energy term $E(\vx^{(n)}; \vz, \theta)$. It can be written as follows:
\begin{align}
 p_\Theta(\vx^{(n)} | \vz, \theta) &= \frac{1}{Z(\vz, \theta)} \exp(E(\vx^{(n)}; \vz, \theta)) \enspace , \\
E(\vx^{(n)}; \vz, \theta) &= \left\langle T(\vx^{(n)}), \eta(\vz, \theta) \right\rangle \enspace .
\end{align}

Here we introduce the following assumptions from the theorem:
\begin{align}
Z(\vz, \theta) &= Z(\theta) \label{asump:1} \enspace , \\
\eta(\vz, \theta) &= \eta_z(\vz) \odot \eta_{\theta}(\theta) \enspace . \label{asump:2}
\end{align}

Then the $\ELBO_2(\Phi, \Theta)$ term reads:
\begin{align}
  \ELBO_2(\Phi, \Theta)
&= 
  \frac{1}{N} 
  \sum_{n=1}^N 
  \int 
    q^{(n)}_{\Phi}(\vz)
    \left\langle T(\vx^{(n)}), \eta_z(\vz) \odot \eta_{\theta}(\theta) \right\rangle 
    d\vz 
  -\log Z(\theta) \\
&= 
  \frac{1}{N} 
  \sum_{n=1}^N 
  \left\langle  
    \eta_{\theta}(\theta), \int q^{(n)}_{\Phi}(\vz) T(\vx^{(n)}) \odot \eta_z(\vz) d\vz 
  \right\rangle  
  -\log Z(\theta) \\
&= 
  \left\langle  
    \eta_{\theta}(\theta), 
    \frac{1}{N} 
    \sum_{n=1}^N 
    \int q^{(n)}_{\Phi}(\vz) T(\vx^{(n)}) \odot \eta_z(\vz) d\vz 
  \right\rangle  
  -\log Z(\theta) \enspace . \label{eq:3}
\end{align}

We are interested in stationary points of $\ELBO_2(\Phi, \Theta)$ w.r.t. $\theta$, which means that:
\begin{align}
0 = \frac{\partial \ELBO_2(\Phi, \Theta^*)}{\partial \theta} 
&= 
  \frac{1}{N} 
  \sum_{n=1}^N 
    \int q^{(n)}_{\Phi}(\vz)
      \frac{\partial E(\vx^{(n)}; \vz, \theta^*) -\log Z(\vz, \theta^*)}{\partial \theta}  
      d\vz \\
&= 
  \frac{1}{N} 
  \sum_{n=1}^N 
    \int q^{(n)}_{\Phi}(\vz)
      \left(
        \frac{\partial \left\langle T(\vx^{(n)}), \eta(\vz, \theta^*) \right\rangle }{\partial \theta} 
        -\frac{\partial\log Z(\vz, \theta^*)}{\partial \theta} 
      \right) 
      d\vz \enspace . \label{eq:2}
\end{align}

Applying the assumptions (\ref{asump:1}) and (\ref{asump:2}), it can be rewritten as follows:
\begin{align}
0 = \frac{\partial \ELBO_2(\Phi, \Theta^*)}{\partial \theta} 
  &= \frac{1}{N} \sum_{n=1}^N 
  \int q^{(n)}_{\Phi}(\vz)
    \left(
      \frac{\partial \left\langle T(\vx^{(n)}), \eta_z(\vz) \odot \eta_{\theta}(\theta^*) \right\rangle }{\partial \theta} 
      -\frac{\partial\log Z(\theta^*)}{\partial \theta}
    \right) 
  d\vz \\
  &= \frac{\partial \eta_{\theta}(\theta^*) }{\partial \theta} 
    \frac{1}{N} 
    \sum_{n=1}^N 
    \int 
      q^{(n)}_{\Phi}(\vz) T(\vx^{(n)}) \odot \eta_z(\vz) 
      d\vz 
    - \frac{\partial\log Z(\theta^*)}{\partial \theta} \enspace .
\end{align}

If the $\frac{\partial \eta_{\theta}(\theta^*) }{\partial \theta}$ Jacobian is invertible, it follows that:
\begin{align}
  \frac{1}{N} 
  \sum_{n=1}^N 
    \int 
      q^{(n)}_{\Phi}(\vz) T(\vx^{(n)}) \odot \eta_z(\vz)  
      dz \label{eq:4}
&=
  \left[ 
     \frac
       {\partial \eta_{\theta}(\theta^*)}
       {\partial \theta} 
  \right]^{-1}
  \frac
    {\partial\log Z(\theta^*)}
    {\partial \theta} \\ 
&= 
  \frac
    {\partial\log Z(\theta^*)}
    {\partial \eta_{\theta}(\theta)}  \enspace . \label{eq:5}
\end{align}

Plugging (\ref{eq:4}-\ref{eq:5}) into (\ref{eq:3}) we get the final entropy representation:
\begin{align}
\ELBO_2(\Phi, \Theta^*) &= \left\langle  
    \eta_{\theta}(\theta^*), 
    \frac
    {\partial\log Z(\theta^*)}
    {\partial \eta_{\theta}(\theta)} 
    \right\rangle
      - \log Z(\theta^*) = -H[p_\Theta(\vx | \vz, \theta^*)] \enspace .
\end{align}

The last equation is a common form of entropies for exponential family distributions, see e.g. \citep{nielsen_entropies_2010}.

Now we take the parameterized prior distribution $p_{\theta}(\vz)$ and consider the $\ELBO_1(\Phi, \Theta)$ term:
\begin{align}
    \ELBO_1(\Phi, \Theta) 
    = \frac{1}{N} \sum_{n=1}^N  \int q^{(n)}_{\Phi}(\vz) \log p_{\theta}(\vz) d\vz \enspace . \label{eq:6}
\end{align}

Similarly to the likelihood distribution, we require it to belong to the following class of exponential family distributions:
\begin{align}
  p_{\vlambda}(\vz) 
  = \exp(\eta_{\vlambda}^\mathrm{T}(\vlambda)\eta(\vz) - \log Z(\vlambda)) \enspace .
\end{align}

Then we can rewrite the integral over the prior (\ref{eq:6}) as:
\begin{align}
  \ELBO_1(\Phi, \Theta) 
  = \frac{1}{N} \sum_{n=1}^N \int q^{(n)}_{\Phi}(\vz) \log p_{\vlambda}(\vz) d\vz 
  = \left\langle \eta_{\vlambda}(\vlambda) , \frac{1}{N} \sum_{n=1}^N  \int q^{(n)}_{\Phi}(\vz) \eta(\vz) d\vz \right\rangle - \log Z(\vlambda) \enspace . \label{eq:7}
\end{align}

We are interested in stationary points w.r.t. the $\vlambda$ parameters. Thus, setting the derivative to zero:
\begin{align}
  0 
  &= \frac{\partial\ELBO_1(\Phi, \Theta^*)}{\partial \vlambda}
  = \frac{\partial \eta_{\vlambda}(\vlambda^*)}{\partial \vlambda} \frac{1}{N} \sum_{n=1}^N \int q^{(n)}_{\Phi}(\vz) \eta(\vz) d\vz - \frac{1}{Z(\vlambda)}\frac{\partial Z(\vlambda^*)}{\partial \vlambda} \;\Rightarrow\; \\
   \frac{1}{N} \sum_{n=1}^N \int q^{(n)}_{\Phi}(\vz) \eta(\vz) d\vz 
  &= \left[\frac{\partial \eta_{\vlambda}(\vlambda^*)}{\partial \vlambda} \right]^{-1} \frac{\partial \log Z(\vlambda^*)}{\partial \vlambda} = \frac{\partial \log Z(\vlambda^*)}{\partial \eta_{\vlambda}(\vlambda)} \enspace .
\end{align}

Inserting it into (\ref{eq:7}) we get:
\begin{align}
  \ELBO_1(\Phi, \Theta) 
  =\left\langle \eta_{\vlambda}(\vlambda^*), \frac{\partial \log Z(\vlambda^*)}{\partial \vlambda} \right\rangle - \log Z(\vlambda^*) = -H[p_{\Theta}(\vz|\vlambda^*)] \enspace ,
\end{align}
which completes the proof.
\end{proof}

Now we show that the Laplace prior sparse coding model fulfills these sufficient conditions. The mapping functions read:
\begin{align}
    \eta_{\vlambda}(\vlambda) &= \vectorize[-\frac{1}{\lambda_1}, \cdots, -\frac{1}{\lambda_H}] \enspace , \\
    \eta_{\vz}(\vz) &= \vectorize [\Wt\vz, -\frac{1}{2}] \enspace , \\
    \eta_{\theta}(\theta) &= \vectorize [-\frac{1}{\sigma^2}] \enspace , \\
    T(\vx) &= \vectorize [\vx, \vx\vx^T] \enspace .
\end{align}
Jacobians $\frac{\partial \eta_{\theta}(\theta)}{\partial \theta}$ and $\frac{\partial \eta_{\vlambda}(\vlambda)}{\partial \vlambda}$ are diagonal (with non-zero elements) and thus clearly invertible, which satisfies the conditions. %

\subsection{Proof of Convergence to Three Entropies Using the General Conditions}
\label{app:proof-with-general-conditions}
Here we prove the convergence to three entropies for the sparse coding model with Laplace prior, using the general conditions \citep{Lucke2022Convergence}. For this, we have to show that the parametrizations of the prior and the likelihood distributions fulfill the corresponding criteria. 

\textit{Prior distribution parametrization criterion}. Let $\zeta(\Psi)$ be a function that maps model parameters to the natural parameters of $ p_\Theta(z)$, $\mathcal{I}_{(\Psi)} = [ \frac{\partial \zeta_i(\Psi)}{\Psi_j} ] $ is the Jacobian matrix of the mapping function. Then the following criterion should be met for the ELBO integral to be equal to entropy at convergence, for any function $f(\Phi, \Psi)$:
\begin{align}
  \mathcal{I}^\mathrm{T}_{(\Psi)} f(\Phi, \Psi) = 0 \;\implies\; \zeta(\Psi)^\mathrm{T} f(\Phi, \Psi) = 0 \enspace .
\end{align} 

The Laplace prior has a very simple mapping into the natural parameter space: $\zeta(\psi) = -\frac{1}{\vlambda}$. The Jacobian is a diagonal matrix and reads: 
\begin{align}
  \mathcal{I}^\mathrm{T}_{(\Psi)} = 
    \begin{bmatrix}
      \frac{1}{\lambda_1^2} & \cdots & 0 \\
      \vdots & \ddots & \vdots \\
      0 & \cdots & \frac{1}{\lambda_H^2} \\ 
    \end{bmatrix} \enspace .
\end{align}

Now we can check that the parametrization criterion holds:
\begin{align}
  \begin{bmatrix}
    \frac{1}{\lambda_1^2} & \cdots & \frac{1}{\lambda_H^2} \\ 
  \end{bmatrix} f(\Phi, \Psi) &= 0  \\
  \implies\; f(\Phi, \Psi) &= 0  \\
  \implies\; \zeta(\Psi)^\mathrm{T} f(\Phi, \Psi) &= 0 \enspace .
\end{align}

\textit{Likelihood distribution parametrization criterion}. Let $\eta(\vz, \Theta)$ be a function that maps model parameters to the natural parameters of $ p_\Theta(\vx | \vz)$, 
$\mathcal{J}_{(\vz, \Theta)} = [ \frac{\partial \eta_i(\vz, \Theta)}{\Theta_j} ] $ is the Jacobian matrix of the mapping function. 
Then the following criterion should be met for the ELBO integral to be equal to entropy at convergence for any function $g(\vz, \Phi, \Theta)$ and a subset of parameters $\theta \in \Theta$:
\begin{align}
  \int \mathcal{J}^\mathrm{T}_{(\vz, \theta)} g(\vz, \Phi, \Theta) d\vz = 0 \;\implies\; \int \eta(\vz, \Theta)^\mathrm{T} g(\vz, \Phi, \Theta) = 0 \enspace .
\end{align} 

Let's check if Gaussian likelihood in the sparse coding model fulfills this criterion. The function to map $z$ and $\Theta=\{\Wt, \sigma^2 \}$ to natural parameters reads:
\begin{align}
  \eta(z, \Theta) = 
    \begin{bmatrix}
      \frac{\Wt\vz}{\sigma^2} \\ 
      -\frac{1}{2\sigma^2} \\ 
    \end{bmatrix} 
\end{align}

We choose the subset $\theta = \{ \sigma^2 \}$. Then the Jacobian of the mapping w.r.t. $\theta$ reads:
\begin{align}
  \mathcal{J}^\mathrm{T}_{(\vz, \theta)} = 
    \begin{bmatrix}
      -\frac{\Wt_{1,:}\vz}{\sigma^4}, & 
      \cdots, &
      -\frac{\Wt_{H,:}\vz}{\sigma^4}, &       
      \frac{1}{2\sigma^4} \\ 
    \end{bmatrix} \enspace .
\end{align}

Check the criterion:
\begin{align}
  0 
  &= \int \mathcal{J}^T_{(\vz, \theta)} g(\vz, \Phi, \Theta) d\vz \\
  &= \int \begin{bmatrix}
    -\frac{\Wt_{1,:}\vz}{\sigma^4}, & 
    \cdots, &
    -\frac{\Wt_{H,:}\vz}{\sigma^4}, &       
    \frac{1}{2\sigma^4} \\ 
    \end{bmatrix}
    g(\vz, \Phi, \Theta) d\vz \\
  &= -\frac{1}{\sigma^2} \int \begin{bmatrix}
    \frac{\Wt_{1,:}\vz}{\sigma^2}, & 
    \cdots, &
    \frac{\Wt_{H,:}\vz}{\sigma^2}, &       
    -\frac{1}{2\sigma^2} \\ 
    \end{bmatrix}
    g(\vz, \Phi, \Theta) d\vz \enspace .
\end{align}

We can recognize the entries of the mapping function $\eta(\vz, \Theta)$ in the argument of the integral. We can, therefore, rewrite and conclude:
\exclude{
Check the criterion:
\begin{align}
  \begin{cases}
    \int -\frac{\Wt_{1,:}\vz}{\sigma^4} g(\vz, \Phi, \Theta) d\vz &= 0 \\
    \cdots & \\
    \int -\frac{\Wt_{H,:}\vz}{\sigma^4} g(\vz, \Phi, \Theta) d\vz &= 0 \\
    \int \frac{1}{2\sigma^4} g(\vz, \Phi, \Theta) d\vz &= 0
  \end{cases}
\end{align}
 
\begin{align}
\implies
  \begin{cases}
    \frac{1}{\sigma^2} \int -\frac{\Wt_{1,:}\vz}{\sigma^2} g(\vz, \Phi, \Theta) d\vz &= 0 \\
    \cdots & \\
    \frac{1}{\sigma^2} \int -\frac{\Wt_{H,:}\vz}{\sigma^2} g(\vz, \Phi, \Theta) d\vz &= 0 \\
    \frac{1}{\sigma^2} \int \frac{1}{2\sigma^2} g(\vz, \Phi, \Theta) d\vz &= 0
  \end{cases} 
\end{align}

As the prefactor is unequal zero, all integrals are zero. 
We can recognize the entries of the mapping function $\eta(\vz, \Theta)$ in arguments of the integrals. We can, therefore, rewrite and conclude:
}
\begin{align}
  \frac{1}{\sigma^2} \int \eta(\vz, \Theta)^\mathrm{T} g(\vz, \Phi, \Theta) d\vz &= 0 \\
  \implies \int \eta(\vz, \Theta)^\mathrm{T} g(\vz, \Phi, \Theta) d\vz &= 0  \enspace ,
\end{align}
where the last step follows from $\sigma^2$ being unequal to zero.
Thus, the parametrization criterion is fulfilled.

\subsection{Likelihood Convergence Proof (Gaussian)}
\label{app:gaussian-likelihood-convergence}

To have this paper self-contained we here reiterate the argument that the log-likelihood term $\mathcal{L}_2(\Phi, \Theta)$ in the ELBO (see  \cref{theo:laplace-prior}) with optimal observation noise reduces to the negative entropy of the likelihood (see, e.g., \cite{DammEtAl2023}, their Theorem 1 for a slightly different derivation).

The expectation of the log-likelihood under the variational posterior, denoted as $\mathcal{L}_2(\Phi, \Theta)$ in \cref{theo:laplace-prior}, reads
\begin{equation}
	\label{eq:appL2}
	\mathcal{L}_2(\Phi, \Theta) = -\frac{1}{N}\sum_n \Big( \frac{1}{2\sigma^2} \int q^{(n)}_{\Phi}\!(\vz)\, \| \vxn - \tilde W\vz \|^2 \mathrm{d}\vz 
	\Big) - \frac{D}{2} \log\!\big( 2\pi\sigma^2 \big) \enspace .
\end{equation}
Note that this is the only term in the ELBO, given in \cref{EqnELBO}, that depends on the observation noise $\sigma^2$. Consequently, the derivative of the ELBO, denoted as $\mathcal{L}$, w.r.t. $\sigma^2$ is given by

\begin{equation}
	\frac{\d \mathcal{L}(\Phi, \Theta)}{\d \sigma^2} = \frac{\d \mathcal{L}_2(\Phi, \Theta)}{\d \sigma^2} = 
	-\frac{1}{N}\sum_n \left( \frac{1}{2 \sigma^4} \int {q_\Phi^{(n)}(\vz)}  \| \vxn - \tilde W\vz \|^2 \mathrm{d}\vz  \right) - \frac{D}{2 \sigma^2} \enspace .
\end{equation}
As $\sigma^2 > 0 $, we conclude the whenever $\frac{\d \mathcal{L}(\Phi, \Theta)}{\d \sigma^2} = 0$ the following holds 
\begin{align}
	\frac{1}{N}\sum_n& \left( \frac{1}{2 \sigma^2} \int {q_\Phi^{(n)}(\vz)}  \| \vxn - \tilde W\vz \|^2 \mathrm{d}\vz  \right) - \frac{D}{2} = 0\\
	\implies \frac{1}{N}\sum_n& \left( \frac{1}{2 \sigma^2} \int {q_\Phi^{(n)}(\vz)}  \| \vxn - \tilde W\vz \|^2 \mathrm{d}\vz  \right) = \frac{D}{2}
	\label{eq:app_sigma2_opt} \enspace .
\end{align}
So, at optimality the high-dimensional integral in \cref{eq:appL2} has a particularly simple solution and by plugging \cref{eq:app_sigma2_opt} into \cref{eq:appL2} we obtain
\begin{equation}
	\label{eq:appL2Entropy}
	\mathcal{L}_2(\Phi, \Theta) = - \frac{D}{2} - \frac{D}{2} \log\!\big( 2\pi\sigma^2 \big) = 
	- \frac{D}{2} \log\!\big( 2 e \pi\sigma^2 \big)
	= -\mathcal{H}[p_\Theta(\vx | \vz)]
	\enspace,
\end{equation}
that is, $\mathcal{L}_2$ becomes equal to the (negative) entropy of $p_\Theta(\vx \vert \vz)$, which concludes the argument.

\ \\
\section{ELBOS FOR LAPLACE PRIOR SPARSE CODING}
\label{app:ELBOs_for_Laplace_Prio_SC}
\subsection{Deriving Analytic ELBO for Laplace Prior Sparse Coding Model}
\label{app:derive-analytic-elbo}

Here we discuss sparse coding defined as a linear latent variable model with Laplace prior defined as: 
\begin{align}
  p_\Theta(\mathbf{\vz}) &= \prod_{h=1}^H \frac{1}{2\lambda_h} \exp\left(-\frac{|z_h|}{\lambda_h}\right) \enspace , \\
  p_\Theta(\vx | \vz) &= \mathcal{N} (\vx | \Wt\vz, \sigma^2\mathbb{I}) \enspace , \label{eq:likelihood} \\
  || \Wt_{:,h} || &= 1 \enspace .
\end{align}

For convenience we will use a more extended notation $p_\Theta(\vxn | \vz, \Wt, \sigma^2)$ for the noise distribution in \cref{eq:likelihood}. The ELBO with variational distribution $q_{\Phi}^{(n)}(\vz)$ for $N$ data points reads:
\begin{align}
  \log p_\Theta(\vx) \geq \ELBO= \frac{1}{N} \sum_{n=1}^N \int q_\Phi^{(n)}(\vz) \log \frac{ p_\Theta(\vxn | \vz, \Wt, \sigma^2) \,p_\Theta(\vz)}{q_\Phi^{(n)}(\vz)} d\vz \enspace .
\end{align}

We can rewrite it as:
\begin{align}
  \ELBO &= \frac{1}{N} \sum_{n=1}^N \int q_\Phi^{(n)}(\vz) \log p_\Theta(\vxn | \vz, \Wt, \sigma^2) d\vz \\
  &+ \frac{1}{N} \sum_{n=1}^N \int q_\Phi^{(n)}(\vz) \log p_\Theta(\vz) d\vz \label{EqnLambdaFFTerm}\\
  &- \frac{1}{N} \sum_{n=1}^N \int q_\Phi^{(n)}(\vz) \log  q_\Phi^{(n)} d\vz \enspace .
\end{align}

At stationary points for parameters $\{\vlambda, \sigma^2 \}$, it converges to the following expression of entropy terms:
\begin{align}
  \HELBO = - \mathcal{H}[ p_\Theta(\vx | \vz, \sigma^2)] - \mathcal{H}[ p_\Theta(\vz | \{\lambda_i\})] + \frac{1}{N} \mathcal{H}[q_\Phi^{(n)}(\vz)] \enspace .
\end{align}

For the sparse coding model, it reads:
\begin{align}
  \HELBO= -\frac{D}{2} \log (2\pi e \sigma^2) -\sum_{h=1}^H \log (2 \lambda_h e)  + \frac{1}{N} \sum_n\mathcal{H}[q_\Phi^{(n)}(\vz)] \enspace . 
  \label{eq:sparse:elbo}
\end{align}

Detailed, from the derivation of the convergence to entropies, recall that at stationary points it holds that:
\begin{align}
  \frac{1}{N} \sum_{n=1}^N \int q_\Phi^{(n)}(\vz) \log p_\Theta(\vxn | \vz, \Wt, \sigma^2) d\vz &= \frac{D}{2} \log (2\pi e \sigma^2) \label{eq:sigmasqr}\\
  \frac{1}{N} \sum_{n=1}^N \int q_\Phi^{(n)}(\vz) \log p_\Theta(\vz | \vlambda) d\vz &= \sum_{h=1}^H \log (2 \lambda_h e) \enspace .
\end{align}

Now we can analytically solve the integrals and obtain expressions for optimal $\lambda_h$ and $\sigma^2$ at stationary points.

We use a Gaussian distribution with full covariance as the variational distribution for each data point $\vxn$:
\begin{align}
  q_\Phi^{(n)}(\vz) = \mathcal{N}(\vz | \vnu^{(n)}, \TT^{(n)}) \enspace .
\end{align}

Let us start to with \cref{eq:sigmasqr} and solve the integral analytically to obtain $\sigma^2$: %
\begin{align}
  & \frac{1}{N} \sum_{n=1}^N \int q_\Phi^{(n)}(\vz) \log p_\Theta(\vxn | \vz, \Wt, \sigma^2) d\vz \\
  &= -\log [Z(\mathbb{I}\sigma^2)] - \frac{1}{N} \sum_{n=1}^N \frac{1}{2\sigma^2} \E_{q^{(n)}(\vz)} \left[ (\Wt\vz-\vx^{(n)})^\mathrm{T}(\Wt\vz-\vx^{(n)}) \right] \\
  &= -\frac{D}{2} \log(2\pi\sigma^2) - \frac{1}{2\sigma^2} \frac{1}{N} \sum_{n=1}^N \left[\tr(\Wt^\mathrm{T}\Wt\TT^{(n)}) + (\Wt\vnu^{(n)}-\vx^{(n)})^\mathrm{T}(\Wt\vnu^{(n)}-\vx^{(n)}) \right]   \enspace .
\end{align}
The last equation can be obtained by carefully expanding the quadratic form and taking the corresponding expectations w.r.t. the Gaussian density $q^{(n)}(\vz)$. %

Taking derivative w.r.t. $\sigma^2$ and setting it to zero:
\begin{align}
  0 
  &= \frac{\partial \frac{1}{N} \sum_{n=1}^N \int q_\Phi^{(n)}(\vz) \log p_\Theta(\vxn | \vz, \Wt, \sigma^2) d\vz}{\partial \sigma^2} \\
  &= -\frac{D}{2\sigma^2} + \frac{1}{2(\sigma^2)^2} \frac{1}{N} \sum_{n=1}^N \left[\tr(\Wt^\mathrm{T}\Wt\TT^{(n)}) + (\Wt\vnu^{(n)}-\vx^{(n)})^\mathrm{T}(\Wt\vnu^{(n)}-\vx^{(n)}) \right] \enspace .
\end{align}

Solve it w.r.t. $\sigma^2$:
\begin{align}
  \sigma^2 
  = \frac{1}{DN} \sum_{n=1}^N \left[\tr(\Wt^\mathrm{T}\Wt\TT^{(n)}) + (\Wt\vnu^{(n)}-\vx^{(n)})^\mathrm{T}(\Wt\vnu^{(n)}-\vx^{(n)}) \right] \enspace .
  \label{eq:sparse:sigmasqr}
\end{align}

Next, we solve for the optimal prior scales $\lambda_h$. As \cref{EqnLambdaFFTerm} is the only term of the ELBO that depends on $\vlambda$, the condition for a stationary point for $\lambda_h$ yields:
\begin{align}
  0 
  &= \frac{\partial \frac{1}{N} \sum_{n=1}^N \int q_\Phi^{(n)}(\vz) \log p_\Theta(\vz | \vlambda) d\vz}{\partial \lambda_h} \\
  &= \frac{1}{N} \sum_{n=1}^N \int q_\Phi^{(n)}(\vz) \left(-\frac{1}{\lambda_h} + \frac{|z_h|}{\lambda_h^2}\right) d\vz \\
  &= \frac{1}{\lambda_h^2}\,\frac{1}{N} \sum_{n=1}^N \int q_\Phi^{(n)}(\vz) (|z_h|- \lambda_h) d\vz \\ 
  \implies\ 0 &= \frac{1}{N} \sum_{n=1}^N \int q_\Phi^{(n)}(\vz) |z_h| d\vz - \frac{1}{N} \sum_{n=1}^N \int q_\Phi^{(n)}(\vz) \lambda_h d\vz \\
  &= \frac{1}{N} \sum_{n=1}^N \int q_\Phi^{(n)}(\vz) |z_h| d\vz - \lambda_h \label{eq:highdimintegral} \\
  \implies\ \lambda_h &= \frac{1}{N} \sum_{n=1}^N \int \mathcal{N}(\vz\,|\, \vnu^{(n)}, \TT^{(n)})\, |z_h|\,d\vz \label{eq:multidimintegral} \\
              &= \frac{1}{N} \sum_{n=1}^N \int \mathcal{N}(z_h \,|\, \nu_h^{(n)}, \TT_{hh}^{(n)})\, |z_h|\, dz_h \enspace . \label{eq:onedimintegral} %
\end{align}
The $h$-dimensional integral (\ref{eq:multidimintegral}) w.r.t. $\vz$ can be simplified to (\ref{eq:onedimintegral}) because we can rewrite the Gaussian distribution as a product of a marginal and a conditional Gaussian $q_\Phi^{(n)}(\vz) = q_\Phi^{(n)}(\vz_{\setminus h} | z_h) q_\Phi^{(n)}(z_h)$:
\begin{align}
  &= \int \mathcal{N}(z_{\setminus h} | \vnu^{(n)}_{\setminus h}(z_h), \TT^{(n)}_{\setminus h}(z_h)) \mathcal{N}(z_h | \nu^{(n)}_h, \TT^{(n)}_{hh}) |z_h| d\vz \\
  &= \int \mathcal{N}(z_h | \nu^{(n)}_h, \TT^{(n)}_{hh}) |z_h| \int \mathcal{N}(z_{\setminus h} | \vnu^{(n)}_{\setminus h}(z_h), \TT^{(n)}_{\setminus h}(z_h))  d\vz_{\setminus h}  dz_h \\
  &= \int \mathcal{N}(z_h | \nu^{(n)}_h, \TT^{(n)}_{hh}) |z_h| dz_h \enspace .
\end{align}

To solve this integral we now make use of another integral that is known to have an analytic solution:
\begin{align}
  \int_0^{+\infty} z\, \exp\big(-(az+b)^2\big)\, dz = \frac{\sqrt{\pi}\, b}{2a^2} \big(\erf(b)-1\big) + \frac{(-b^2)}{2a^2} \text{ for } a > 0 \enspace .
\end{align}
The analytic solution of the integral is, e.g., stated in Eq.\,2.1.2 by \citep[][]{korotkov2020integrals}. The book is itself based on two earlier books by the same authors \citep[][]{KorotkovKorotkov2012} and \citep[][]{Korotkov2002}.
Rewriting the integral to a proper Gaussian integral by substituting mean and covariance and by multiplying by the normalizing coefficient gives:
\begin{align}
  \int_0^{+\infty} z\, \mathcal{N}(z | \nu, \sigma^2)\, dz = \frac{\sigma}{\sqrt{2\pi}} \exp\left(-\frac{1}{2} \frac{\nu^2}{\sigma^2}\right) -\frac{\nu}{2} \left[ \erf\left(-\frac{\nu}{\sqrt{2}\sigma}\right) -1 \right] \enspace .
\end{align}

The integral over the complementary set of the support reads:
\begin{align}
  \int_{-\infty}^0 (-z)\, \mathcal{N}(z | \nu, \sigma^2)\, dz 
  &= \int_0^{+\infty} z\, \mathcal{N}(z | -\nu, \sigma^2) \,dz \\
  &= \frac{\sigma}{\sqrt{2\pi}} \exp\left(-\frac{1}{2} \frac{\nu^2}{\sigma^2}\right) +\frac{\nu}{2} \left[ \erf\left(\frac{\nu}{\sqrt{2}\sigma}\right) -1 \right] \enspace . 
\end{align}
\ \vspace{-2ex}\\
The full integral over the magnitude of $z$ therefore reads:
\begin{align}
  \int \mathcal{N}(z | \nu, \sigma^2)\, |z|\, dz
  &= \frac{2\sigma}{\sqrt{2\pi}} \exp\left(-\frac{1}{2} \frac{\nu^2}{\sigma^2}\right) +\frac{\nu}{2} \left[ \erf\left(\frac{\nu}{\sqrt{2}\sigma}\right) - \erf\left(-\frac{\nu}{\sqrt{2}\sigma}\right) \right] \\
  &= \frac{2\sigma}{\sqrt{2\pi}} \exp\left(-\frac{1}{2} \frac{\nu^2}{\sigma^2}\right) + \nu \erf\left(\frac{\nu}{\sqrt{2}\sigma}\right) \enspace .
\end{align}

Therefore we obtain for $\lambda_h$ the expression:
\begin{align}
  \lambda_h 
  &= \frac{1}{N} \sum_{n=1}^N \int \mathcal{N}(z_h | \nu_h^{(n)}, \TT_{hh}^{(n)})\, |z_h|\, dz_h \\
  &= \frac{1}{N} \sum_{n=1}^N \left[ \frac{2\sqrt{\TT_{hh}^{(n)}}}{\sqrt{2\pi}} \exp\left(-\frac{1}{2} \frac{(\nu_h^{(n)})^2}{\TT_{hh}^{(n)}}\right) + \nu_h^{(n)} \erf\left(\frac{\nu_h^{(n)}}{\sqrt{2 \TT_{hh}^{(n)}}}\right) \right] \\
  &= \frac{1}{N} \sum_{n=1}^N \sqrt{\TT_{hh}^{(n)}} \,{\cal M}\left( \frac{\nu_h^{(n)}}{\sqrt{\TT_{hh}^{(n)}}} \right) \enspace , \label{eq:sparse:alpha} 
\end{align}
where ${\cal M}(a) = \sqrt{ \frac{2}{\pi} } \exp\left(-\frac{1}{2}\,a^2 \right) + a \erf\left(\frac{a}{\sqrt{2}}\right)$ as defined in the main text, \cref{eq:DefFuncM}.
So the important observation is that the integral \cref{eq:multidimintegral} has an analytic solution, \cref{eq:sparse:alpha}. 
In this context, we remark that integrals such as \cref{eq:multidimintegral} emerged in other contexts of probabilistic machine learning. Concretely, \citet{ChallisBarber2013} investigated integrals of Gaussians with different `potential functions' including integrals with potential functions $\exp(-|x|)$ (while also other potential functions were treated). 
The same authors point out \mbox{\citep[][]{BarberBishop1998,KussRasmussen2005}} for procedures to reduce high-dimensional to one-dimensional integrals analogously to how \cref{eq:onedimintegral} is obtained from \cref{eq:multidimintegral} (but marginalizations involving Gaussians are also generally well-known).

All models treated in \citep[][]{ChallisBarber2013} (theoretically and empirically) consider fully Bayesian learning using sparse weight matrices. Consequently, the used approximation bound is different from the here considered ELBOs for sparse coding (for both the entropy-based version \cref{eq:sparse:elbo} as well as the classical ELBO \cref{eq:intloglikelihood} to \cref{eq:intlogproposal}).
The emerging problems are closely related, however, and the bound treated by \citet{ChallisBarber2013} could be reformulated to relate to the probablistic sparse coding problem using the classical ELBO \cref{eq:intloglikelihood} to \cref{eq:intlogproposal} (which is also explicitly stated by the authors in the introduction). The integral that emerges for Laplace potentials in their site potential term of the bound is the same as the integral required to solve for $\lambda_h$ in our context (see \cref{eq:multidimintegral}); and \citet{ChallisBarber2013} also provide an analytic soltuion for the integral.
In the context of fully Bayesian approaches, entropy convergence results could, visa versa, also be applied to the bound of \citet{ChallisBarber2013} albeit some algebraic transformations would be required. Convergence to entropy sums could then potentially be useful for models such as Gaussian process regression etc. 

In principle, the analytic integral solutions that emerge in standard probabilistic sparse coding \cref{EqnPSC} are known since still earlier. For instance, \citet[see][]{Korotkov2002,KorotkovKorotkov2012} provided
analytic solutions of integrals that can be used for the here emerging integrals (and we used these solutions in our derivation above). Hence, analytic solutions could have been used, e.g., for work by \citet[][]{seeger_bayesian_2008} or \citep[][]{barello_sparse_coding_2018}, and may prove useful in future work in these direction.

\exclude{
I THINK WE SAY EVERYTHING NOW IN THE MAIN TEXT.\\\
\ \\
Now we can write the ELBO equation when $\lambda_h$ and $\sigma^2$ are optimal, by plugging in \cref{eq:sparse:alpha} and \cref{eq:sparse:sigmasqr} into \cref{eq:sparse:elbo}:
\begin{align}
 \lambda_{\mathrm{opt}, h} 
  &= \frac{1}{N} \sum_n \sqrt{\TT_{hh}^{(n)}} {\cal M}\left( \frac{\nu_h^{(n)}}{\sqrt{\TT_{hh}^{(n)}}} \right) \\
  \sigmaopt 
  &= \frac{1}{DN} \sum_n \left[\tr(\Wt^\mathrm{T}\Wt\TT^{(n)}) + (\Wt\vnu^{(n)}-\vxn)^\mathrm{T}(\Wt\vnu^{(n)}-\vxn) \right] \\
  \HELBO (\Phi,\Wt)
  &= \frac{1}{N} \sum_{n=1}^N  \frac{1}{2} \log(|2\pi e \TT^{(n)}|) 
     -\sum_{h=1}^H \log (2\lambda_{\mathrm{opt}, h} e) 
     -\frac{D}{2} \log (2\pi e \sigmaopt) 
  \label{eq:helbo:objective}
\end{align}

We have to maximize the $\HELBO (\Phi,\Wt)$ w.r.t. $\Wt, \{\vnu^{(n)}, \TT^{(n)}\}$ respecting the corresponding constraints for every variable: a) the columns of $\Wt$ must be of unit length, and b) $\TT^{(n)}$ are S.P.D. matrices (use factorized triangular representation e.g.). 
}

\subsection{Sparse Coding. ELBO for Other Versions of Gaussian Variational Distributions}
\label{app:derive-other-elbo}
If the variational posterior is an uncorrelated Gaussian $q^{(n)}(\vz) = \mathcal{N}(\vz | \vnu^{(n)}, \diag(\taunsqr{1}, \ldots, \taunsqr{H}) )$, the entropy-based ELBO objective can be further simplified. First, we exploit that the columns of $\Wt$ are normalized, and consider the term
\begin{align}
    \tr(\Wt^\mathrm{T}\Wt \diag( \taunsqr{1}, \ldots, \taunsqr{H}) ) 
    &= \diag(\Wt^\mathrm{T}\Wt)^T  \vectorize ( \taunsqr{1}, \ldots, \taunsqr{H} ) \\
    &= \sum_{h=1}^H \Wt_{:,h}^\mathrm{T}\Wt_{:,h} \taunsqr{h} \\
    &= \sum_{h=1}^H \taunsqr{h} \enspace ,
\end{align}
which removes the dependency on $\Wt$ here. The optimal $\lambda_{\mathrm{opt}, h}$ and $\sigmaopt$ then read:
\begin{align}
  \lambda_{\mathrm{opt}, h}(\Phi)
  &= \frac{1}{N} \sum_{n=1}^N \taun{h} {\cal M}\left( \frac{\nu_h^{(n)}}{\taun{h}} \right) \\
  \sigmaopt(\Phi,\Wt) \,
  &= \, \frac{H}{D}\, \bar{\tau}^2 \, + \, \frac{1}{D}\frac{1}{N} \sum_{n=1}^N (\Wt\vnu^{(n)}-\vxn)^\mathrm{T}(\Wt\vnu^{(n)}-\vxn), 
  \mbox{\ where\ } \bar{\tau}^2 = \frac{1}{N} \sum_{n=1}^N \frac{1}{H} \sum_{h=1}^H \taunsqr{h} \enspace , 
  \label{eq:tau_avg}
\end{align}
which gives us a simplified entropy-based objective:
\begin{align}
  \HELBO (\Phi,\Wt)
  &= \frac{1}{N} \sum_{n=1}^N  \sum_{h=1}^H \frac{1}{2} \log\big(2\pi e \big(\tau^{(n)}_h\big)^2\big) 
     -\sum_{h=1}^H \log (2 e\lambda_{\mathrm{opt}, h}(\Phi) ) 
     -\frac{D}{2} \log (2\pi e \sigmaopt(\Phi,\Wt)) \\
  &= \frac{1}{N} \sum_{n=1}^N  \sum_{h=1}^H \frac{1}{2} \log\big(2\pi e \big(\tau^{(n)}_h\big)^2\big) 
     -\sum_{h=1}^H \log \left(2 e \frac{1}{N} \sum_{n=1}^N \taun{h} {\cal M}\left( \frac{\nu_h^{(n)}}{\taun{h}} \right) \right) \\
  &-\frac{D}{2} \log \left(2\pi e \left[ \frac{H}{D}\, \bar{\tau}^2 \, 
    + \, \frac{1}{D}\frac{1}{N} \sum_{n=1}^N (\Wt\vnu^{(n)}-\vxn)^\mathrm{T}  (\Wt\vnu^{(n)}-\vxn) \right] \right) \enspace .
     \label{eqn:helbodiag}
\end{align}

If the objective is annealed as suggested in \cref{eq:three-entropies-annealing}, then the objective reads:
\begin{align}
  \HELBO (\Phi,\Wt)
  &= \frac{1}{N} \sum_{n=1}^N  \sum_{h=1}^H \frac{1}{2} \log\big(2\pi e \big(\tau^{(n)}_h\big)^2\big) 
     - \,\gamma\, \sum_{h=1}^H \log \big(2 e\lambda_{\mathrm{opt}, h}(\Phi) \big) 
     - \,\delta\, \frac{D}{2} \log \big(2\pi e \sigmaopt(\Phi,\Wt)\big) \enspace . 
     \label{eqn:helbodiag-anneal}
\end{align}

\subsection{Sparse Coding. Classical Variational Inference Objective}
\label{app:derive-classical-elbo}
Having obtained the analytic solution of the ELBO in \cref{eq:analytic-elbo}, it could be asked how much the results rely on the entropy convergence results. For this, we here consider the original ELBO of the sparse coding model defined as in \cref{EqnPSC}.
The ELBO with variational distribution $q^{(n)}(\vz)$ for $N$ data points reads:
\begin{align}
  \log p_\Theta(\vx) \geq \ELBO(\Phi,\Theta) = \frac{1}{N} \sum_{n=1}^N \int q_\Phi^{(n)}(\vz) \log \frac{ p_\Theta(\vxn | \vz, W, \sigma^2)  p_\Theta(\vz)}{q_\Phi^{(n)}(\vz)} d\vz \enspace .
\end{align}

We can rewrite it as:
\begin{align}
  \ELBO(\Phi,\Theta) &= \frac{1}{N} \sum_{n=1}^N \int q_\Phi^{(n)}(\vz) \log p_\Theta(\vxn | \vz, W, \sigma^2) d\vz \label{eq:intloglikelihood} \\
  &+ \frac{1}{N} \sum_{n=1}^N \int q_\Phi^{(n)}(\vz) \log p_\Theta(\vz) d\vz \label{eq:intlogprior} \\
  &- \frac{1}{N} \sum_{n=1}^N \int q_\Phi^{(n)}(\vz) \log q_\Phi^{(n)}(\vz) d\vz  \enspace . \label{eq:intlogproposal}
\end{align}

Similarly, as for the entropy-based ELBO, we use full covariance Gaussian $q^{(n)}(\vz) = \mathcal{N}(\vz | \vnu^{(n)}, \TT^{(n)})$ as a variational posterior distribution. The integral over the likelihood function (Eq. \ref{eq:intloglikelihood}) then reads:
\begin{align}
  & \int q_\Phi^{(n)}(\vz) \log p_\Theta(\vxn | \vz, W, \sigma^2) d\vz \\
  &= -\log [Z(\mathbb{I}\sigma^2)] - \frac{1}{2\sigma^2} \E_{q(\vz)} \left[ (W\vz-\vxn)^\mathrm{T}(W\vz-\vxn) \right] \\
  &= -\frac{D}{2} \log(2\pi\sigma^2) - \frac{1}{2\sigma^2} \left[\tr(W^\mathrm{T}W\TT^{(n)}) + (W\vnu^{(n)}-\vxn)^\mathrm{T}(W\vnu^{(n)}-\vxn) \right]  \enspace .
\end{align}

Now consider the integral in \cref{eq:intlogprior}. We can rewrite it as follows:
\begin{align}
  \int q_\Phi^{(n)}(\vz) \log p_\Theta(\vz) d\vz
  &= \int \mathcal{N}(\vz | \vnu^{(n)}, \TT^{(n)}) \Big(H \log\Big(\frac{1}{2}\Big) - \sum_{h=1}^H |z_h|\Big) d\vz \\
  &= H \log\Big(\frac{1}{2}\Big) - \int \mathcal{N}(\vz | \vnu^{(n)}, \TT^{(n)})  \sum_{h=1}^H |z_h| d\vz \\
  &= H \log\Big(\frac{1}{2}\Big) - \sum_{h=1}^H \int \mathcal{N}(\vz_h | \vnu^{(n)}_h, \TT^{(n)}_{hh})  |z_h| d\vz_h \\
  &= H \log\Big(\frac{1}{2}\Big) - \sum_{h=1}^H \left[ \frac{2\sqrt{\TT_{hh}^{(n)}}}{\sqrt{2\pi}} \exp\left(-\frac{1}{2} \frac{(\nu_h^{(n)})^2}{\TT_{hh}^{(n)}}\right) + \nu_h^{(n)} \erf\left(\frac{\nu_h^{(n)}}{\sqrt{2 \TT_{hh}^{(n)}}}\right) \right] \enspace .
\end{align}
That is, we can use the result obtained for the entropy ELBO and also for the classical ELBO.

\Cref{eq:intlogproposal} is just a Gaussian entropy:
\begin{align}
  -\int q_\Phi^{(n)}(\vz) \log q(\vz) d\vz 
  &= \mathcal{H}[q_\Phi^{(n)}(\vz)] \\
  &= \frac{1}{2} \log(|2\pi e \TT^{(n)}|) \enspace .
\end{align}

Thus, the classical ELBO objective can be reformulated as follows:
\begin{align}
  \ELBO(\Phi,\Theta) 
  = &-\frac{D}{2} \log(2\pi\sigma^2) - \frac{1}{2\sigma^2} \frac{1}{N} \sum_{n=1}^N \left[\tr(W^\mathrm{T}W\TT^{(n)}) + (W\vnu^{(n)}-\vxn)^\mathrm{T}(W\vnu^{(n)}-\vxn) \right]  \\
  &+ H \log\Big(\frac{1}{2}\Big) - \frac{1}{N} \sum_{n=1}^N \sum_{h=1}^H 
  \sqrt{\TT_{hh}^{(n)}} \,{\cal M}\left( \frac{\nu_h^{(n)}}{\sqrt{\TT_{hh}^{(n)}}} \right) \\
  &+  \frac{1}{N} \sum_{n=1}^N \frac{1}{2} \log(|2\pi e \TT^{(n)}|) \enspace .
\end{align}

\newcommand{\FFEE}{{\cal F}^{\mathrm{EE}}} %
\newcommand{\FFEL}{{\cal F}^{\mathrm{ELBO}}} %
\newcommand{\PhiMax}{{\Phi}^{\mathrm{max}}} %
\newcommand{\ThetaNew}{{\Theta}^{\mathrm{new}}} %
\newcommand{\AtPoint}[1]{\Big|_{\hspace{0ex}\tiny #1}}
\newcommand{\AtPoints}[2]{\Big|_{\hspace{-1ex}\tiny\begin{array}{c} \ \\ #1\\ #2 \end{array}}}

\newcommand{\lambdaVec}{\vec{\lambda}}
\newcommand{\lambdaVecT}{\vec{\lambda}^{\mathrm{T}}}
\newcommand{\lambdaVecNew}{\vec{\lambda}^{\mathrm{new}}}

\newcommand{\vtau}{\vec{\tau}}
\newcommand{\tauVec}{\vec{\tau}}
\newcommand{\sigmaTilde}{\tilde{\sigma}}

\exclude{
\section{Optimization of Variational Parameters Based on Entropy sums}
\mycomment{Remove this section? But preserve.}
Using Theorems 1 and 2 we can define a new objective for learning that is based on entropies. \\
We use $\sigmaTilde=\sigma^2$. We can then define 
\begin{align}
\FFEE_{(\Phi,\Theta)} = \FFEE_{(\tauVec,\lambdaVec,\sigmaTilde)} = 
	\frac{1}{N} \sum_n \mathcal{H}[q_\Phi^{(n)}(\vz)] 
	- \mathcal{H}[p_\Theta(\vz)] 
	- \mathcal{H}[p_\Theta(\vx | \vz)] \\
= \frac{1}{N}\sum_n\sum_{h=1} \frac{1}{2}\log\!\big( 2\pi{}e{}\tau^{(n)}_h \big) - \sum_{h=1} \log\!\big( 2e{}\lambda_h \big) - \frac{D}{2} \log\!\big( 2\pi{}e\sigmaTilde \big)
\end{align}
After the M-step\footnote{In our case the M-step amounts to taking the gradient $\frac{\partial}{\partial \Wt}\sigmaTilde(\Phi,\Wt)$ with $\sigmaTilde(\Phi,\Wt)$ as given by Theorem 2,
while respecting the constraints of $\Wt$.}, the new model parameters are given by $\ThetaNew$, i.e., $\Wt^{\mathrm{new}}$ was computed using a gradient step and $\lambdaVecNew$ and $\sigma^2_{new}$ were determined using Theorem 2 based on $\Phi^{old}$ and $\Wt^{\mathrm{new}}$.

Now the following formula may apply (this has still to be checked thoroughly): The optimal variational parameters $\Phi$ are given by the solution of the following equation:
\begin{align}
\textstyle \frac{\partial}{\partial \Phi} \FFEE_{(\Phi,\Theta)}\AtPoint{\Theta=\ThetaNew}
+ \big( \frac{\partial}{\partial \lambdaVecT} \FFEE_{(\Phi,\Theta)} \big)\AtPoint{\Theta=\ThetaNew} \Big( \frac{\partial}{\partial \Phi} \lambdaVec\big(\Phi\big)    \Big)
+ \big( \frac{\partial}{\partial \sigma^2} \FFEE_{(\Phi,\Theta)} \big)\AtPoint{\Theta=\ThetaNew} \Big( \frac{\partial}{\partial \Phi} \sigma^2\big(\Phi,\Wt\big)\Big)\AtPoint{\Wt=\Wt^{new}}     = 0\nonumber
\end{align}
The formula looks like the conventional total differentiation of the function $\FFEE_{\big(\tauVec,\lambdaVec(\Phi),\sigmaTilde(\Phi,\Wt)\big)}$ with $\lambdaVec(\Phi)$ and $\sigmaTilde(\Phi,\Wt)$ as given by Theorem 2. There is one important difference, however. In the formula, we evaluate the first factor of each summand always at the fixed parameters
$\ThetaNew=(\Wt^{new},\lambdaVec^{new},\sigmaTilde^{new})$, i.e., we hold those values fixed at the result of the M-step (also for $\lambdaVec$ and $\sigmaTilde$ that would for the "normal" total derivative depend on $\Phi$ and $\Wt$). The solutions for $\Phi$ should coincide with the solutions for $0=\frac{\partial}{\partial \Phi}\FFEL(\Phi,\Theta)$, i.e., with the derivative of the normal ELBO w.r.t.\ the variational parameters. However, the new formula is (A) entropy-based, and (B) is presumably resulting in nicer equations. 

COMMENT: 
It seems that for $\nu$-updates we indeed get the same objective. With Dmytro it was discussed that he double-checks, and checks how the updates for $\tau_h$ would look like (for conventional and using the new formula). The formula should be valid more generally but I still have to proof it in detail. But it may also be not correct (don't know for sure without a proper proof). Please check everybody. Thanks!
}

\ \\
\section{PROOF OF THEOREM 3}
\label{Theorem3:Proof}

\setcounter{theorem}{2}

In this section, we lay out the details of the proof for \cref{theo:analytic-elbo}.
For completeness, we first restate the entropy-based ELBO (given equivalently in \cref{eq:analytic-elbo}), which reads 
\begin{align*}
    \begin{split}
        \HELBO(\Phi,\Wt) = &\, \frac{1}{N} \sum_{n=1}^{N} \frac{1}{2} \log\big(\,|\,2\pi\,e\,\tau_h^{(n)} \,|\,\big)      \\
        & - \sum_{h=1}^H \log \Big( 2\,e\,  \frac{1}{N} \sum_{n=1}^N \tau_h^{(n)} {\cal M}\Big( \frac{\nu_h^{(n)}}{\tau_h^{(n)}} \Big)    \Big) \\
        & -\frac{D}{2} \log \Big( 2\pi{}e \frac{1}{DN} \sum_{n=1}^N \E_{q^{(n)}_{\Phi}\!(\vz)} \| \vxn - \tilde W\vz \|^2  \Big) \enspace ,
    \end{split}
\end{align*}
where we again use $\tau_h^2 = \TT_{hh}$ as a short-hand for the diagonal elements of the covariance (and, accordingly, $\tau_h = \sqrt{\TT_{hh}}$ for their positive square root).
To prove \cref{theo:analytic-elbo} we rely on the following Lemma which establishes the equality of gradients on the manifold of optimal scales and variances, i.e., all points in parameter space that satisfy \cref{eq:condition_stationary_points} for non-amortized and amortized parametrizations.

\begin{lemma}[Equality of gradients on manifold of optimal scales and variance]
    \label{lemma:gradient_equality}
    Consider the learning objectives $\ELBO(\Phi, \Theta)$, given in \cref{EqnELBO}, and $\HELBO(\Phi, \Wt)$, given in \cref{eq:analytic-elbo}, for the probabilistic sparse coding model formulated in \cref{EqnPSC2} with $\Theta=(\Wt,\sigma^2, \vlambda) \in  \RRRnorm^{D\times{}H}  \times \RRR_+ \times \RRR_+^H$, and variational parameters $\Phi = (\Phi_\nu, \Phi_\TT)$ that parameterize mean 
    $ \vnu^{(n)} \in \RRR^H$ and covariance $\TT^{(n)} \in \RRR^{H\times H}$ (in amortized or non-amortized fashion).
    We assume $\Phi_\nu \cap \Phi_\TT = \emptyset$.
    
    Then, whenever \cref{eq:condition_stationary_points} holds, it holds that
    \begin{align}
    \begin{split}
        \label{eq:Equality_of_Gradients}
        \nabla_{\Phi} \ELBO(\Phi,\Theta) &= \nabla_{\Phi} \HELBO(\Phi,\Wt) \enspace , \\
        \nabla_{\Wt} \ELBO(\Phi,\Theta)  &= \nabla_{\Wt} \HELBO(\Phi,\Wt) \enspace .
    \end{split}
    \end{align}
\end{lemma}
\begin{proof}
    We need to prove that the gradients for $\Theta$ and $\Phi$ for both objectives are equal at all points in parameter space whenever \cref{eq:condition_stationary_points} is fulfilled, i.e., whenever the scales $\vlambda$ and the variance $\sigma^2$ are optimal.
    
    Recall that the ELBO, given in \cref{EqnELBO}, can be written as (using the notation of \cref{theo:laplace-prior})
    \begin{align}
        \ELBO(\Phi, \Theta) = \overbrace{\frac{1}{N} \sum_{n=1}^N \mathcal{H}[q_\Phi^{(n)}(\vz)] + \ELBO_1(\Phi, \Theta)}^{\text{regularization (neg. KL-divergence)}} \overbrace{+ \vphantom{\frac{1}{N} \sum_{n=1}^N} \ELBO_2(\Phi, \Theta)}^{\text{reconstruction}} %
    \intertext{
    and whenever the condition of \cref{theo:laplace-prior} (i.e., \cref{eq:condition_stationary_points}) is satisfied, we observe the term-wise convergence to entropies such that the ELBO decomposes into three entropies
    }
		= \underbrace{\frac{1}{N} \sum_{n=1}^N \mathcal{H}[q_\Phi^{(n)}(\vz)] 
		- \mathcal{H}[p_\Theta(\vz)]}_{\text{regularization (neg. KL-divergence)}}
		\underbrace{- \vphantom{\frac{1}{N} \sum_{n=1}^N}\mathcal{H}[p_\Theta(\vx | \vz)]}_{\text{reconstruction}} \enspace .
    \end{align}
    Note that the entropy-based objective $\HELBO$, given in \cref{eq:analytic-elbo}, is merely the sum of entropies above with analytically optimal scale and variance parameters obtained from \cref{theo:optimal-parameters}.

    We need to investigate the gradients with respect to the parameters of the model $\Theta=(\Wt,\sigma^2, \vlambda)$ and the variational parameters $\Phi=(\Phi_\nu, \Phi_\TT)$.
    We start by addressing the model parameters $\Theta$.
    
    \textbf{Model Parameters $\Theta$:}
    Considering $\Theta$, only the parameters $w \in \tilde W$ are of interest.\footnote{%
    The remaining parameters $\sigma, \vlambda$ will be learned in case of $\ELBO$, or set to optimality in case of $\HELBO$. However, as \cref{eq:condition_stationary_points} holds in both cases they always yield zero gradients.}
    Thus, only the reconstruction terms, $\ELBO_2$ and its counterpart $- \mathcal{H}[p_\Theta(\vx | \vz)]$, contribute to the gradients for $\tilde W$.
    We consider a general parameter $w \in \tilde W$. Regarding the standard ELBO $\ELBO$, the gradient is given by
	\begin{align}
		\frac{\partial}{\partial w} \ELBO(\Phi, \Theta)  &= \frac{\partial}{\partial w} \ELBO_2(\Theta, \Phi) \\
		&= - \frac{1}{2 \sigma^2} \frac{\partial}{\partial w} \left[\frac{1}{N} \sum_{n=1}^N \E_{q^{(n)}_{\Phi}(\vz)} \big\| \vxn - \tilde W \vz\big\|^2 \right] \\
		&= - \frac{1}{\sigma^2} \frac{1}{N} \sum_{n=1}^N \E_{q^{(n)}_{\Phi}(\vz)} \left(\vxn - \tilde W \vz\right) \frac{\partial}{\partial w} \tilde W \vz \enspace.
    \intertext{Similarly, for the entropy-based ELBO $\HELBO$ we obtain}
		\frac{\partial}{\partial w} \HELBO(\Phi, \Wt) 
		&= - \frac{\partial}{\partial w} \mathcal{H}[ p_{\Theta}(\vx \vert \vz)] \\
		&= - \frac{D}{2} \frac{\partial}{\partial w} \log\big(2 \pi e 
		\frac{1}{ND} \sum_{n=1}^N \E_{q^{(n)}_{\Phi}(\vz)} \big\| \vxn - \tilde W \vz \big\|^2 \big) \\
		&= - D \frac{\frac{1}{N} \sum_{n=1}^N \E_{q^{(n)}_{\Phi}(\vz)} \left(\vxn - \tilde W \vz\right)}{\frac{1}{N} \sum_{n=1}^N \E_{q^{(n)}_{\Phi}(\vz)} \big\| \vxn - \tilde W \vz\big\|^2 } \frac{\partial}{\partial w} \tilde W \vz \enspace ,
		\intertext{and with $\sigmaopt(\Phi, \tilde W) = \frac{1}{ND} \sum_{n=1}^N \E_{q^{(n)}_{\Phi}(\vz)} \big\| \vxn - \tilde W \vz\big\|^2$ (as derived in \cref{theo:optimal-parameters}) we conclude}
		&= - \frac{1}{\sigmaopt} \frac{1}{N} \sum_{n=1}^N \E_{q^{(n)}_{\Phi}(\vz)} \left(\vxn - \tilde W \vz\right) \frac{\partial}{\partial w} \tilde W \vz \enspace .
	\end{align}	

    Observe that both objectives yield the same gradient information for any $w \in \Wt$, just scaled by a ratio that reflects how far $\sigma^2$ is from its optimal value $\sigmaopt$, such that
    \begin{equation}
         \nabla_{\Wt}
         \HELBO(\Phi, \Wt) = \frac{\sigma^2}{\sigmaopt} 
         \nabla_{\Wt}
         \ELBO(\Phi, (\vlambda, \Wt, \sigma^2)) \enspace .
         \label{eq:gradW}
    \end{equation}
    Importantly, both objectives give rise to the same gradients for $\tilde W$ whenever $\sigma^2 = \sigmaopt$, i.e., when \cref{eq:condition_stationary_points} is satisfied.\footnote{Note that the same constraints on $\Wt$ are imposed for both objectives. That is, any additive regularization terms of the form $\mathcal{L} + R(\Wt)$ would consequently yield the very same gradient updates for $\mathcal{L} \in \{ \ELBO, \HELBO \}$.}

    \textbf{Variational Parameters $\Phi$:}
    It remains to investigate how the gradients for the variational parameters $\Phi$ are affected when training with $\ELBO$ or $\HELBO$.
    We consider $\phi \in \Phi$ for which the gradient decomposes into three terms
    \begin{align}
    \frac{\partial}{\partial \phi} \ELBO(\Phi, \Theta) &=
    \frac{\partial}{\partial \phi} \frac{1}{N} \sum_{n=1}^N \mathcal{H}[q_\Phi^{(n)}(\vz)] + \frac{\partial}{\partial \phi} \ELBO_1(\Phi, \Theta) + \frac{\partial}{\partial \phi} \sum_{n=1}^N \ELBO_2(\Phi, \Theta)\label{eq:dellphi_ELBO} \enspace , \\
    \frac{\partial}{\partial \phi} \HELBO(\Phi, \Wt) &=
    \frac{\partial}{\partial \phi}\frac{1}{N} \sum_{n=1}^N \mathcal{H}[q_\Phi^{(n)}(\vz)] 
    - \frac{\partial}{\partial \phi} \mathcal{H}[p_\Theta(\vz)]
    - \frac{\partial}{\partial \phi} \mathcal{H}[p_\Theta(\vx | \vz)] \enspace . \label{eq:dellphi_HELBO} %
    \end{align}
    The average encoder entropy (the first term in \cref{eq:dellphi_HELBO,eq:dellphi_ELBO}, respectively) is part of both objectives and consequently provides the same gradient information for any $\phi \in \Phi$ (regardless of the concrete parametrization in terms of $\Phi$).
    We continue with the gradient updates arising from the reconstruction score, i.e., $\ELBO_2(\Phi, \Theta)$ vs. $- \mathcal{H}[p_\Theta(\vx | \vz)]$, the last term in \cref{eq:dellphi_HELBO,eq:dellphi_ELBO}, respectively.
    Starting with the latter we get %
    \begin{align}
		- \frac{\partial}{\partial \phi} \mathcal{H}[ p_{\Theta}(\vx \vert \vz)] %
		&= - \frac{D}{2} \frac{\partial}{\partial \phi} \log\big(2 \pi e 
		\sigmaopt(\Phi, \tilde W) \big) \\
		&= - \frac{D}{2} \frac{1}{\sigmaopt(\Phi, \tilde W)} \frac{\partial}{\partial \phi}  \sigmaopt(\Phi, \tilde W)\\
		\intertext{and by invoking $\sigmaopt(\Phi, \tilde W) = \frac{1}{ND} \sum_{n=1}^N \E_{q_\Phi(\vz \vert \vxn)} \big\| \vxn - \tilde W\vz)\big\|^2$,}
		&= - \frac{1}{2\sigmaopt}  \frac{\partial}{\partial \phi} \left[\frac{1}{N} \sum_{n=1}^N \E_{q(\vz \vert \vxn)} \big\| \vxn - \tilde W \vz\big\|^2 \right] \enspace .
	\end{align}
    Considering the corresponding counterpart, $\ELBO_2(\Phi, \Theta)$ in the classical ELBO, the gradient is directly given by
    \begin{align}
		\frac{\partial}{\partial \phi}  \ELBO_2(\Phi, \Theta) &= - \frac{1}{2\sigma^2}  \frac{\partial}{\partial \phi} \left[\frac{1}{N} \sum_{n=1}^N \E_{q(\vz \vert \vxn)} \big\| \vxn - \tilde W \vz\big\|^2 \right] \enspace,		
    \end{align}
    such that gradient updates are again scaled depending on how close $\sigma^2$ is to $\sigmaopt$
    \begin{equation}
         - \frac{\partial}{\partial \phi} \mathcal{H}[ p_{\Theta}(\vx \vert \vz)] = \frac{\sigma^2}{\sigmaopt} \frac{\partial}{\partial \phi}  \ELBO_2(\Phi, \Theta) \enspace .
    \end{equation}
    We are left with the middle terms in \cref{eq:dellphi_ELBO,eq:dellphi_HELBO}, i.e., $ \frac{\partial}{\partial \phi} \ELBO_1(\Phi, \Theta)$ vs. $ - \frac{\partial}{\partial \phi} \mathcal{H}[p_\Theta(\vz)]$.
    To enable the gradient computations we first need to derive a closed-form expression for $\ELBO_1(\Phi, \Theta)$ with $q_\Phi^{(n)}(\vz) = \mathcal{N}(\vz; \vnu^{(n)}, \TT^{(n)}\big)$, again with diagonal elements $\tau_h^2 = \TT_{hh}$, and a Laplace prior with learnable scales $\lambda_h$, i.e., $ p_\Theta(\mathbf{\vz}) = \prod_{h=1}^H \frac{1}{2\lambda_h} \exp\left(-\frac{|z_h|}{\lambda_h}\right)$. 
    
    Recall that $\ELBO_1(\Phi, \Theta) = \textstyle\frac{1}{N} \sum_n \int q^{(n)}_{\Phi}(\vz) \log p_{\Theta}(\vz) \d\vz$ for which the individual summands evaluate to %
\begin{align}
    \label{eq:L1_closed_form_start}
    \int q_\Phi^{(n)}(\vz) \log p_\Theta(\vz) \d\vz %
    &= \int \mathcal{N}(\vz | \vnu^{(n)}, \TT^{(n)}) \Big(\sum_{h=1}^H \log\Big(\frac{1}{2 \lambda_h}\Big) - \sum_{h=1} \frac{|z_h|}{\lambda_h}\Big) \d\vz \\
    &= \sum_{h=1}^H \log\Big(\frac{1}{2 \lambda_h}\Big) - \int \mathcal{N}(\vz | \vnu^{(n)}, \TT^{(n)})  \sum_{h=1} \frac{|z_h|}{\lambda_h} \d\vz \\
    &= - \sum_{h=1}^H \log(2 \lambda_h) - \sum_{h=1}^H \frac{1}{\lambda_h}\int \mathcal{N}(z_h | \nu^{(n)}_h, \TT^{(n)}_{hh})  |z_h| \d z_h \\
    &= - \sum_{h=1}^H \log(2 \lambda_h) - \sum_{h=1}^H \frac{1}{\lambda_h} \left[ \sqrt{ \frac{2}{\pi} } \tau_h^{(n)} \exp\left(-\frac{1}{2} \Bigg(\frac{\nu_h^{(n)}}{\tau_h^{(n)}}\Bigg)^2 \right) + \nu_h^{(n)} \erf\left(\frac{\nu_h^{(n)}}{\sqrt{2} \tau_h^{(n)}}\right) \right] \enspace. \label{eq:L1_closed_form_end}
    \intertext{With help of the statistic ${\cal M}(a) = \sqrt{ \frac{2}{\pi} } \exp\left(-\frac{1}{2}\,a^2 \right) + a \erf\left(\frac{a}{\sqrt{2}}\right)$, introduced in \cref{eq:DefFuncM}, we get}
    &= - \sum_{h=1}^H \log(2 \lambda_h) - \sum_{h=1}^H \frac{\tau_h^{(n)}}{\lambda_h} 
    \mathcal{M}\Big(\frac{\nu_h^{(n)}}{\tau_h^{(n)}}\Big)
    \intertext{such that}
    \ELBO_1(\Phi, \Theta) &= 
        - \frac{1}{N}\sum_{n=1}^N \sum_{h=1}^H 
        \left[ \log(\lambda_h)
        + \frac{\tau_h^{(n)}}{\lambda_h} 
        \mathcal{M}\Big(\frac{\nu_h^{(n)}}{\tau_h^{(n)}}\Big)
        + c \right] 
    \label{eq:ELBO1_closedform}
\end{align}
for some constant $c$ (that does not influence any gradients). Note that now the functional dependency for mean and variance parameters matters for the gradient calculations, such that we need to consider $\phi \in \Phi_\nu$ and $\phi \in \Phi_\TT$ separately.

Before proceeding we need to address the different parametrization that arises for non-amortized vs. amortized approaches. 
In the non-amortized setting, the variational parameters $\Phi = (\Phi_\nu, \Phi_\TT)$ directly parameterize mean and covariance of $q_\Phi^{(n)}$ (per data point $\vxn$), i.e.,
$\Phi_\nu = (\vnu^{(1)},\ldots,\vnu^{(N)})$ and $\Phi_\TT = (\TT^{(1)},\ldots,\TT^{(N)})$.
In amortized approaches, we take $\Phi = (\Phi_\nu, \Phi_\TT)$ to parameterize the two functions\footnote{Commonly, artificial neural networks are utilized here, such that $\Phi_\nu$ denotes the parameters of the neural net that predicts the mean, and $\Phi_\TT$ the parameters of the neural net that predicts the covariance. The independence assumption in this Lemma does not allow for parameter sharing between those networks. Often, the covariance is restricted to be a diagonal matrix such that $\TT_\Phi: \RRR^D \to \RRR^{H}$.} $\vnu_\Phi: \RRR^D \to \RRR^H$ and $\TT_\Phi: \RRR^D \to \RRR^{H\times H}$ such that mean and covariance are given as the respective function outputs, i.e,
\begin{align}
   \vnu^{(n)} &= \vnu_\Phi(\vxn) \enspace , \\
   \TT^{(n)} &= \TT_\Phi(\vxn) \enspace .
\end{align}
The derivations in the sequel cover both settings, as we only need to compare the resulting gradients in terms of functions of partial derivatives $\frac{\partial \nu_h^{(n)}}{\partial \phi_\nu}$ or $\frac{\partial \tau_h^{(n)}}{\partial \phi_\tau}$, which clearly differ in amortized vs. non-amortized parametrizations, but are the same for both objectives.

\exclude{
\SD{
Without parameter sharing:
    \begin{align*}
    \frac{\partial}{\partial \phi } \ELBO_1(\Phi, \Theta) %
        &= - \frac{\partial}{\partial \phi } \left[\frac{1}{N}\sum_{n=1}^N  \sum_{h=1}^H \frac{\tau_h^{(n)}}{\lambda_h} 
		\mathcal{M}\Big(\frac{\nu_h^{(n)}}{\tau_h^{(n)}}\Big) \right] \\
    &= - \frac{1}{N}\sum_{n=1}^N  \sum_{h=1}^H \frac{\partial}{\partial \nu_h^{(n)} } \left[ \frac{\tau_h^{(n)}}{\lambda_h} \mathcal{M}\Big(\frac{\nu_h^{(n)}}{\tau_h^{(n)}}\Big)\right] \frac{\partial \nu_h^{(n)}}{\partial \phi} \\
    &\phantom{=} - \frac{1}{N}\sum_{n=1}^N  \sum_{h=1}^H \frac{\partial}{\partial \tau_h^{(n)} } \left[ \frac{\tau_h^{(n)}}{\lambda_h} 
		\mathcal{M}\Big(\frac{\nu_h^{(n)}}{\tau_h^{(n)}}\Big)\right] \frac{\partial \tau_h^{(n)}}{\partial \phi}  \\
	&= - \frac{1}{N}\sum_{n=1}^N  \sum_{h=1}^H \frac{1}{\lambda_h} 
		\erf\Big( \frac{1}{\sqrt{2}} \frac{\nu_h^{(n)}}{\tau_h^{(n)}} \Big)
		\frac{\partial \nu_h^{(n)}}{\partial \phi} %
        \\
      &\phantom{=} -  \frac{1}{N} \sum_{n=1}^N \sum_{h=1}^H \frac{1}{\lambda_{h}} \sqrt{\frac{2}{\pi}} \exp\Big( - \frac{1}{2} \frac{\nu_h^{(n)}}{\tau_h^{(n)}} \Big)
		\frac{\partial \tau_h^{(n)}}{\partial \phi_\nu}
		\intertext{where we made use of the following derivatives which invokes the fact that $\frac{\partial \mathcal{M}(a)}{\partial a} = \erf\left(\frac{a}{\sqrt{2}}\right)$,}
    \frac{\partial}{\partial \nu_h^{(n)} } \left[ \frac{\tau_h^{(n)}}{\lambda_h} \mathcal{M}\Big(\frac{\nu_h^{(n)}}{\tau_h^{(n)}}\Big)\right]
		&= \frac{1}{\lambda_h}\erf\Big( \frac{1}{\sqrt{2}} \frac{\nu_h^{(n)}}{\tau_h^{(n)}} \Big) \enspace , \\
  \frac{\partial}{\partial \tau_h^{(n)} } \left[ \frac{\tau_h^{(n)}}{\lambda_h} 
		\mathcal{M}\Big(\frac{\nu_h^{(n)}}{\tau_h^{(n)}}\Big)\right] &= 
    \left[ \mathcal{M}\Big(\frac{\nu_h^{(n)}}{\tau_h^{(n)}} \Big) + \tau_h^{(n)} \frac{\partial}{\partial \tau_h^{(n)}} \mathcal{M}\Big(\frac{\nu_h^{(n)}}{\tau_h^{(n)}}\Big) \right] \\ 
		&= \left[ \mathcal{M}\Big(\frac{\nu_h^{(n)}}{\tau_h^{(n)}}\Big) -  \frac{\nu_h^{(n)}}{\tau_h^{(n)}} \erf\Big( \frac{1}{\sqrt{2}} \frac{\nu_h^{(n)}}{\tau_h^{(n)}} \Big) \right]
		\\
		\overset{(\ref{eq:DefFuncM})}&{=} \sqrt{\frac{2}{\pi}}\exp\Big(- \frac{1}{2} \frac{\nu_h^{(n)}}{\tau_h^{(n)}} \Big)\frac{\partial \tau_h^{(n)}}{\partial \phi_\tau} \enspace. %
    \end{align*}
}}

    \underline{Variational Parameters: Mean}
 
    Let us continue with the gradient updates for the variational mean $\vnu_\Phi$, so just the mean parameters $\phi_\nu \in \Phi_\nu$ are of interest here.
    Considering $\ELBO$ first, the updates for $\vnu$ from $\ELBO_1(\Phi, \Theta)$, in the form of \cref{eq:ELBO1_closedform}, result in
    \begin{align}
    \frac{\partial}{\partial \phi_\nu } \ELBO_1(\Phi, \Theta) %
        &= - \frac{1}{N}\sum_{n=1}^N  \sum_{h=1}^H \frac{\tau_h^{(n)}}{\lambda_h} 
		\frac{\partial}{\partial \phi_\nu} \mathcal{M}\Big(\frac{\nu_h^{(n)}}{\tau_h^{(n)}}\Big)  \\
		&= - \frac{1}{N}\sum_{n=1}^N  \sum_{h=1}^H \frac{1}{\lambda_h} 
		\erf\Big( \frac{1}{\sqrt{2}} \frac{\nu_h^{(n)}}{\tau_h^{(n)}} \Big)
		\frac{\partial \nu_h^{(n)}}{\partial \phi_\nu} \label{eq:UpdateNuELBO} \enspace ,
		\intertext{where we made use of the following derivative which invokes the fact that $\frac{\partial \mathcal{M}(a)}{\partial a} = \erf\left(\frac{a}{\sqrt{2}}\right)$,}
		\frac{\partial}{\partial \phi_\nu} \mathcal{M}\Big(\frac{\nu_h^{(n)}}{\tau_h^{(n)}}\Big)
		&= \frac{1}{\tau_h^{(n)}}\erf\Big( \frac{1}{\sqrt{2}} \frac{\nu_h^{(n)}}{\tau_h^{(n)}} \Big)
		\frac{\partial \nu_h^{(n)}}{\partial \phi_\nu} \enspace . 
    \end{align}
     Now, the respective gradient updates from the prior entropy of the entropy-based objective $\HELBO$ are given as
	\begin{align}
		- \frac{\partial}{\partial \phi_\nu} \mathcal{H}[p_\Theta(\mathbf{z})] 
		&= - \frac{\partial}{\partial \phi_\nu}  \sum_{h=1}^H \log(2 e \lambda_{\mathrm{opt}, h}(\Phi)) \\
		&= - \sum_{h=1}^H \frac{1}{ \lambda_{\mathrm{opt}, h}} \frac{\partial}{\partial \phi_\nu}  \lambda_{\mathrm{opt}, h}(\Phi) \\
		&= - \frac{1}{N}\sum_{n=1}^N \sum_{h=1}^H \frac{1}{ \lambda_{\mathrm{opt}, h}(\Phi)} \erf\Big( \frac{1}{\sqrt{2}} \frac{\nu_h^{(n)}}{\tau_h^{(n)}} \Big) \frac{\partial \nu_h^{(n)}}{\partial \phi_\nu} \enspace.   \label{eq:UpdateNuEntropies}
	\end{align} 
	The line above makes use of the following derivation 
	\begin{align}
		\frac{\partial}{\partial \phi_\nu}  \lambda_{\mathrm{opt}, h}(\Phi) 
		&= \frac{\partial}{\partial \phi_\nu}  \frac{1}{N} \sum_{n=1}^N \tau_h^{(n)} \mathcal{M}\Big(\frac{\nu_h^{(n)}}{\tau_h^{(n)}}\Big) \\ 
		&= \frac{1}{N} \sum_{n=1}^N \frac{\tau_h^{(n)}}{\tau_h^{(n)}}\erf\Big( \frac{1}{\sqrt{2}} \frac{\nu_h^{(n)}}{\tau_h^{(n)}} \Big)\frac{\partial \nu_h^{(n)}}{\partial \phi_\nu} \\
		&= \frac{1}{N} \sum_{n=1}^N \erf\Big( \frac{1}{\sqrt{2}} \frac{\nu_h^{(n)}}{\tau_h^{(n)}} \Big)\frac{\partial \nu_h^{(n)}}{\partial \phi_\nu} \enspace.
	\end{align}
	
    By comparing \cref{eq:UpdateNuELBO,eq:UpdateNuEntropies} we again conclude that the gradients just differ in the scaling by $1 / \lambda_h$ vs. $1 / \lambda_{\mathrm{opt}, h}$. That is, for optimal scales $\vlambda = \vlambdaopt$ the gradients of the regularization term for the variational mean $\vnu$ coincide.
	
    \underline{Variational Parameters: Variance}
	
    A similar result holds for (the parameters of) the variational variances. %
    We consider $\phi_\tau \in \Phi_\TT$ and start with the prior entropy in $\HELBO$. The gradient w.r.t. $\phi_\tau$ reads
    \begin{align}
		- \frac{\partial}{\partial \phi_\tau} \mathcal{H}[p_\Theta(\mathbf{z})]
		&= - \frac{\partial}{\partial \phi_\tau}  \sum_{h=1}^H \log(2 e \lambda_{\mathrm{opt}, h}(\Phi)) \\
		\overset{(\ref{eq:lambda_M})}&{=} - \frac{\partial}{\partial \phi_\tau} \sum_{h=1}^H \log\Bigg(2 e  \frac{1}{N} \sum_{n=1}^N \tau_h^{(n)} \mathcal{M}\Big(\frac{\nu_h^{(n)}}{\tau_h^{(n)}}\Big)\Bigg) \\
		&= - \sum_{h=1}^H \frac{1}{\lambda_{\mathrm{opt}, h}(\Phi)} \frac{\partial}{\partial \phi_\tau} \Bigg( \frac{1}{N} \sum_{n=1}^N \tau_h^{(n)} \mathcal{M}\Big(\frac{\nu_h^{(n)}}{\tau_h^{(n)}}\Big)\Bigg) \\
		&= - \frac{1}{N} \sum_{n=1}^N \sum_{h=1}^H \frac{1}{\lambda_{\mathrm{opt}, h}(\Phi)} \sqrt{\frac{2}{\pi}} \exp\Big( - \frac{1}{2} \frac{\nu_h^{(n)}}{\tau_h^{(n)}} \Big)
		\frac{\partial \tau_h^{(n)}}{\partial \phi_\tau} \enspace ,
		\label{eq:UpdateTauEntropies}
    \end{align}
    where the last line makes use of the following derivation
	\begin{align}
    \frac{\partial}{\partial \phi_\tau} \Bigg(  \frac{1}{N} \sum_{n=1}^N \tau_h^{(n)} \mathcal{M}\Big(\frac{\nu_h^{(n)}}{\tau_h^{(n)}}\Big) \Bigg) \label{eq:dellTauLambdaOpt_Start}%
		&= \frac{1}{N} \sum_{n=1}^N \left[ \mathcal{M}\Big(\frac{\nu_h^{(n)}}{\tau_h^{(n)}} \Big) \frac{\partial \tau_h^{(n)}}{\partial \phi_\tau} + \tau_h^{(n)} \frac{\partial}{\partial \phi_\tau} \mathcal{M}\Big(\frac{\nu_h^{(n)}}{\tau_h^{(n)}}\Big) \right]\\ 
		&= \frac{1}{N} \sum_{n=1}^N \left[ \mathcal{M}\Big(\frac{\nu_h^{(n)}}{\tau_h^{(n)}}\Big) -  \frac{\nu_h^{(n)}}{\tau_h^{(n)}} \erf\Big( \frac{1}{\sqrt{2}} \frac{\nu_h^{(n)}}{\tau_h^{(n)}} \Big) \right] \frac{\partial \tau_h^{(n)}}{\partial \phi_\tau} 
		\\
		\overset{(\ref{eq:DefFuncM})}&{=} \frac{1}{N} \sum_{n=1}^N \sqrt{\frac{2}{\pi}}\exp\Big(- \frac{1}{2} \frac{\nu_h^{(n)}}{\tau_h^{(n)}} \Big)\frac{\partial \tau_h^{(n)}}{\partial \phi_\tau} \enspace. \label{eq:dellTauLambdaOpt_End}
	\end{align}
	
    Lastly, for $\ELBO$ we consider the remaining term $\ELBO_1(\Phi)$ which gradients evaluate to\footnote{Recall that for $\ELBO$, $\vlambda$ is just a learnable parameter and consequently no function of $\Phi$ (in contrast to $\HELBO$).}
	\begin{align}
		\frac{\partial}{\partial \phi_\tau} \ELBO_1(\Phi) &= -  \frac{\partial}{\partial \phi_\tau} \frac{1}{N}\sum_{n=1}^N \sum_{h=1}^H 
		\left[ 
		\log(\lambda_h)
		+ \frac{\tau_h^{(n)}}{\lambda_h} 
		\mathcal{M}\Big(\frac{\nu_h^{(n)}}{\tau_h^{(n)}}\Big)
		+c \right]
		\\
		&= - \frac{\partial}{\partial \phi_\tau} \frac{1}{N}\sum_{n=1}^N \sum_{h=1}^H 
		\left[ 
		\frac{\tau_h^{(n)}}{\lambda_h} 
		\mathcal{M}\Big(\frac{\nu_h^{(n)}}{\tau_h^{(n)}}\Big)
		\right] \\
		\intertext{and with the derivations in \cref{eq:dellTauLambdaOpt_Start} -- \cref{eq:dellTauLambdaOpt_End} we get} 
		&= - \frac{1}{N} \sum_{n=1}^N \sum_{h=1}^H \frac{1}{\lambda_{h}} \sqrt{\frac{2}{\pi}} \exp\Big( - \frac{1}{2} \frac{\nu_h^{(n)}}{\tau_h^{(n)}} \Big)
		\frac{\partial \tau_h^{(n)}}{\partial \phi_\tau}
		\label{eq:UpdateTauELBO} \enspace .
	\end{align}
     
     \definecolor{cproof}{RGB}{219, 48, 122}
	
    Again, the resulting gradients in \cref{eq:UpdateTauELBO,eq:UpdateTauEntropies} just differ in the scaling $1/\lambda_h$ vs. $1/\lambda_{\mathrm{opt}, h}$. 
    
    Note that \cref{theo:analytic-elbo} can be generalized to allow for parameter sharing, i.e., the additional assumption $\Phi_\nu \cap \Phi_\TT = \emptyset$ can be dropped. 
    However, we invoked this additional assumption as it (slightly) simplifies and shortens the equations and the overall proof. 
    With the assumption $\Phi_\nu \cap \Phi_\TT = \emptyset$, we can now summarize the term-wise gradient calculations by completing \cref{eq:dellphi_ELBO,eq:dellphi_HELBO} %
    \begin{align}
        \frac{\partial}{\partial \phi} \ELBO(\Phi, \Theta) &=
        \frac{\partial}{\partial \phi} \frac{1}{N} \sum_{n=1}^N \mathcal{H}[q_\Phi^{(n)}(\vz)] + \frac{\partial}{\partial \phi} \ELBO_1(\Phi, \Theta) + \frac{\partial}{\partial \phi} \sum_{n=1}^N \ELBO_2(\Phi, \Theta) \notag \\
        \begin{split}
            &= \frac{\partial}{\partial \phi} \frac{1}{N} \sum_{n=1}^N \mathcal{H}[q_\Phi^{(n)}(\vz)] %
        - \frac{1}{N}\sum_{n=1}^N  \sum_{h=1}^H  \textcolor{cproof}{\frac{1}{\lambda_h}} 
    		\erf\Big( \frac{1}{\sqrt{2}} \frac{\nu_h^{(n)}}{\tau_h^{(n)}} \Big)
    		\frac{\partial \nu_h^{(n)}}{\partial \phi} \\
        & \hspace{+5mm} - \frac{1}{N} \sum_{n=1}^N \sum_{h=1}^H  \textcolor{cproof}{\frac{1}{\lambda_{h}}} \sqrt{\frac{2}{\pi}} \exp\Big( - \frac{1}{2} \frac{\nu_h^{(n)}}{\tau_h^{(n)}} \Big)
    		\frac{\partial \tau_h^{(n)}}{\partial \phi} %
        - \textcolor{cproof}{\frac{1}{2\sigma^2}}  \frac{\partial}{\partial \phi} \left[\frac{1}{N} \sum_{n=1}^N \E_{q(\vz \vert \vxn)} \big\| \vxn - \tilde W \vz\big\|^2 \right] \enspace ,
        \end{split} \label{eq:dellphi_ELBO_full} \\
        \frac{\partial}{\partial \phi} \HELBO(\Phi, \Wt) &=
        \frac{\partial}{\partial \phi}\frac{1}{N} \sum_{n=1}^N \mathcal{H}[q_\Phi^{(n)}(\vz)] 
        - \frac{\partial}{\partial \phi}\mathcal{H}[p_\Theta(\vz)]
        - \frac{\partial}{\partial \phi} \mathcal{H}[p_\Theta(\vx | \vz)] \notag \\
        \begin{split}
            &= \frac{\partial}{\partial \phi} \frac{1}{N} \sum_{n=1}^N \mathcal{H}[q_\Phi^{(n)}(\vz)] %
        - \frac{1}{N}\sum_{n=1}^N \sum_{h=1}^H  \textcolor{cproof}{\frac{1}{ \lambda_{\mathrm{opt}, h}(\Phi)}} \erf\Big( \frac{1}{\sqrt{2}} \frac{\nu_h^{(n)}}{\tau_h^{(n)}} \Big) \frac{\partial \nu_h^{(n)}}{\partial \phi}  \\
        & \hspace{-8mm} - \frac{1}{N} \sum_{n=1}^N \sum_{h=1}^H  \textcolor{cproof}{\frac{1}{\lambda_{\mathrm{opt}, h}(\Phi)}} \sqrt{\frac{2}{\pi}} \exp\Big( - \frac{1}{2} \frac{\nu_h^{(n)}}{\tau_h^{(n)}} \Big)
    		\frac{\partial \tau_h^{(n)}}{\partial \phi}
        -  \textcolor{cproof}{\frac{1}{2\sigmaopt}}  \frac{\partial}{\partial \phi} \left[\frac{1}{N} \sum_{n=1}^N \E_{q(\vz \vert \vxn)} \big\| \vxn - \tilde W \vz\big\|^2 \right] \enspace.	%
        \end{split} \label{eq:dellphi_HELBO_full} %
    \end{align}

    Overall, at points in parameter space that satisfy \cref{eq:condition_stationary_points}, each pair of terms gives rise to the same gradients at stationary points as all scaling coefficients (highlighted in light red) coincide (or the respective terms already provide the very same gradient information regardless of whether \cref{eq:condition_stationary_points} is satisfied).    
    Consequently, the full gradients for all trainable parameters, i.e., the sum of all constitutive terms as given in \cref{eq:dellphi_ELBO_full,eq:dellphi_HELBO_full} for $\Phi$ and \cref{eq:gradW} for $\Wt$, are equivalent whenever \cref{eq:condition_stationary_points} holds, or simply: \cref{eq:condition_stationary_points} implies
    \begin{align*}
        \nabla_{\Phi} \ELBO(\Phi^\star,\Theta^\star) = \nabla_{\Phi} \HELBO(\Phi^\star,\Wt^\star)  \text{ and }  %
        \nabla_{\Wt} \ELBO(\Phi^\star,\Theta^\star)  = \nabla_{\Wt} \HELBO(\Phi^\star,\Wt^\star) \enspace . 
    \end{align*}
\end{proof}

We are now ready to prove \cref{theo:analytic-elbo} from the main paper. 

\begin{theorem}[Restated from main paper]
    Consider the sparse coding model formulated in \cref{EqnPSC2} with model parameters ${\Theta=(\Wt,\sigma^2,\vlambda) \in \RRRnorm^{D\times{}H} \times \RRR_+ \times \RRR_+^H}$, and variational parameters $\Phi = (\Phi_\nu, \Phi_\TT)$ that parameterize mean $ \vnu^{(n)} \in \RRR^H$ and covariance $\TT^{(n)} \in \RRR^{H\times H}$ (in amortized or non-amortized fashion) where $\Phi_\nu \cap \Phi_\TT = \emptyset$.
    
    Then, the set of stationary points of the original objective $\ELBO(\Phi,\Theta)$, given in \cref{EqnELBO}, and of the entropy-based objective $\HELBO(\Phi,\Wt)$, given in \cref{eq:analytic-elbo}, coincide.
    Furthermore, it applies at any stationary point of $\ELBO(\Phi,\Theta)$ or $\HELBO(\Phi,\Wt)$ that
	\begin{align}
		\ELBO(\Phi^\star,\Theta^\star) = \HELBO(\Phi^\star,\Wt^\star) \enspace .
	\end{align}
\end{theorem}

\setcounter{theorem}{3}

\begin{proof}

To prove that the sets of stationary points of $\ELBO$ and $\HELBO$ are equal it suffices to show the following two statements:
\begin{itemize}
    \item[\circled{A}] $(\Phi^\star,\Theta^\star)$ is a stationary point of $\ELBO(\Phi,\Theta)$ $\Rightarrow$ $(\Phi^\star,\Wt^\star)$ is a stationary point of $\HELBO(\Phi,\Wt)$,
    \item[\circled{B}] $(\Phi^\star, \Wt^\star)$ is a stationary point of $\HELBO(\Phi,W)$ $\Rightarrow$ $(\Phi^\star, (\vlambdaopt, \Wt^\star, \sigmaopt))$ is a stationary point of $\ELBO(\Phi, \Theta)$.
\end{itemize}

We start with statement \circled{A}.
Let $(\Phi^\star,\Theta^\star)$ be an arbitrary stationary point of $\ELBO(\Phi,\Theta)$. By definition of fixed points, \cref{eq:condition_stationary_points} holds such that $\Theta^\star = (\Wt^\star, \sigmaopt,\vlambdaopt)$.\footnote{Recall that any local optima for $\vlambda$ and $\sigma^2$ are in fact the (respective) global optima as both problems are convex (see \cref{theo:optimal-parameters}, which also provides the analytic solutions).}
As $(\Phi^\star,\Theta^\star)$ is a stationary point of $\ELBO(\Phi,\Theta)$ we have
\begin{align*}
    \nabla_{\Wt} \HELBO(\Phi^\star,\Wt^\star) &= \nabla_{\Wt} \ELBO(\Phi^\star,\Theta^\star) = 0 \enspace , \\
    \nabla_\Phi \HELBO(\Phi^\star,\Wt^\star)  &= \nabla_\Phi \ELBO(\Phi^\star,\Theta^\star) =0 %
\end{align*}
as the gradients w.r.t. $\Phi$ and $\Wt$ for both objectives are equal by \cref{lemma:gradient_equality} (which accounts for the technicalities of different parameterizations, that arise from amortized vs. non-amortized approaches).
Consequently, $(\Phi^\star,\Wt^\star)$ must also a stationary point of the entropy-based objective $\HELBO(\Phi,\Wt)$.

To show the opposite direction, formulated in statement \circled{B}, we assume that $(\Phi^\star, \Wt^\star)$ is a stationary point of $\HELBO(\Phi,\Wt)$.
By design of the entropy-based objective, \cref{eq:condition_stationary_points} is satisfied as scales and variance are chosen to be optimal for $\HELBO(\Phi,\Wt)$ in each iteration.
We can therefore invoke \cref{lemma:gradient_equality} again and get
\begin{align*}
     \nabla_{\Wt} \ELBO(\Phi^\star,\Theta^\star) &= \nabla_{\Wt} \HELBO(\Phi^\star,\Wt^\star) = 0 \enspace , \\ 
      \nabla_\Phi  \ELBO(\Phi^\star,\Theta^\star) &= \nabla_\Phi \HELBO(\Phi^\star,\Wt^\star) = 0 \enspace .
\end{align*}
Therefore, $(\Phi^\star, (\Wt^\star, \sigmaopt, \vlambdaopt))$ must also be a stationary point of $\ELBO(\Phi, \Theta)$.

Note that the objective functions $\ELBO$ and $\HELBO$ are continuous and continuously differentiable functions. 
From \cref{lemma:gradient_equality} we can also conclude that the Hessians of $\ELBO$ and $\HELBO$ in $\Phi$ and $\Wt$ coincide at stationary points as they admit the very same functional dependencies in $\Phi$ and $\Wt$.
This implies the same convergence behavior in the vicinity ($\epsilon$-ball) around the fixed points such that both objectives have the same stationary points (with same signature).

Eventually, by \cref{theo:laplace-prior} also the function values coincide whenever \cref{eq:condition_stationary_points} holds, which concludes the proof.
\end{proof}

\section{NUMERICAL RESULTS -- DETAILS AND ADDITIONAL RESULTS}
\label{app:numerical_results}
The numerical experiments were run on a desktop computer with Intel i9-9900k 3.6GHz CPU, 32GB RAM, and Nvidia GeForce GTX 1070 8GB. We used CUDA numerical backend for PyTorch whenever possible. The default floating point precision was set to float32. On average, optimization of one epoch of $204\,800$ image patches of size $16\times 16$, with latent dimensionality of 100, by minibatches of 512 with EM-like updates took 156s. One epoch of optimization with stochastic updates by Adam took on average 12s.

\subsection{Approximating the error function}
\label{app:approx-error-function}
While the exact $\erf(\cdot)$ evaluation requires the summing of an infinite number of terms, e.g. of its Taylor series expansion, its approximate computation is heavily optimized in common numerical libraries. To get a closed-form objective, we experimented with a simple second-order B\"urmann approximation \citep{schopf_burmanns_2014}:
\begin{align}
 \erf(x) \approx
 \frac{2}{\sqrt{\pi}} \sqrt{1-e^{-x^2}} 
	\left( 
	\frac{\sqrt{\pi}}{2} 
	+ \frac{21}{200} e^{-kx^2} 
	- \frac{341}{8000} e^{-2kx^2} 
	\right) \enspace .
\end{align}

We did not find any significant difference in optimization results when compared to $\erf(\cdot)$ implemented in numerical libraries, but our na{\"\i}ve implementation led to 20-30\% longer run time. In all our experiments we always used the $\erf(\cdot)$ implementation provided by numerical libraries.

\exclude{
	Writing the error function in the form of a convergent, infinite power series(?) serves to highlight that
	objective ... is an analytic (but not a closed-form) solution.
	\ \\
	ANALYTIC FORMULA FOR THE ELBO!!
	\ \\
	CLOSED-FORM FORMULA FOR THE ELBO (Bürman approximation, mention others)
}

\subsection{Bars dataset}
\label{app:bars-dataset}

We generated the training data according to the model defined in \cref{EqnPSC}. That is, we sampled activation vectors $\vz^{(n)}$ from a Laplace distribution with $\lambda_h = 1$ (for all $h$), linearly combined the weighted generative fields, and added Gaussian noise with standard deviation $\sigma = 0.1$. Each ground truth generative field $W_{:,h}$ contained exactly one (horizontal or vertical) bar (value 1 for `bar', value 0 as the background). 

\begin{figure}[ht]
    \centering
	\includegraphics[width=0.5\textwidth, trim={0 0.5cm 0 0.5cm}, clip]{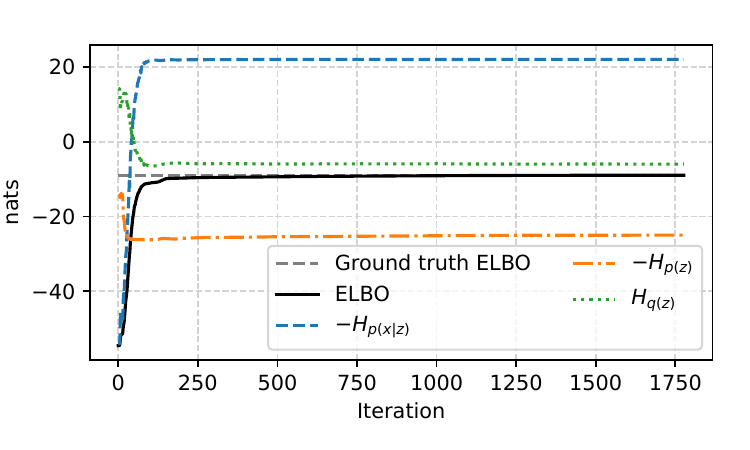}
	\caption{\textbf{Learning the artificial sparse bars dataset.} 
      While the entropy-based ELBO is monotonously increasing, different entropy terms may undergo complex dynamics during the optimization.
    }
	\label{fig:bars-dataset-elbo}
\end{figure}

For this experiment, we used full covariance Gaussian variational posterior.  We observed good convergence and complete recovery of the bars in approximately $70\%$ of runs (7 out of 10). When the model converged to a local optimum, some of the recovered generative fields usually contained two bars, and the final ELBO was slightly lower. During the optimization ELBO values quickly approach the value computed with ground truth $\Wt$, then ELBO values asymptotically converge.
\cref{fig:bars-dataset-elbo} illustrates the typical trajectories of the $\HELBO$ and different entropies during the optimization.

\subsection{Amortized learning}
\label{app:amortized-learning}
Our entropy-based objective can be combined with amortized inference and stochastic updates. We used a deep neural network that comprises two ResNet-like nonlinear mappings (parametric automorphisms) and separate linear readout maps for the mean and the diagonal covariance variational parameters of the posterior (Fig. \ref{fig:deep-encoder}), optimized by stochastic updates (Adam with $lr=10^{-3}$). We compared the convergence speed to the previously suggested (non-amortized) EM-like updates, and considered cases with and without prior entropy annealing (Fig. \ref{fig:compare-optimization-4}). The EM updates with annealing allow the ELBO to be optimized faster and reach a better optimum. We also observe a minor gap (presumably an amortization gap) due to the limited neural network capacity. All three optimization methods finally result in a set of similar generative fields (Fig. \ref{fig:compare-optimization-gabors}). 

Low-rank approximation of full covariance matrices for the variational posterior (\cref{fig:deep-encoder-low-rank}) helps to diminish the amortization gap. To construct a low-rank convariance matrix, the DNN produces a set of $r$ vectors $V \in \mathbb{R}^{H \times r}$, and a separate vecor of diagonal covariances $\sigma^2$. The covariance matrix is then computed as $\TT = VV^\mathrm{T} +  \diag(\sigma^2)$. We used $r=5$ in our experiments.

We did not observe parameters convergence in reasonable time when we trained Laplace-prior sparse coding model with unamortized Gaussian posterior and used reparameterization trick and stochastic updates with even 100 samples (no analytic solutions of the ELBO integrals).

\begin{figure}[h!]
  \centering
  \scalebox{1.0}{\begin{tikzpicture}[
        thick, 
        node distance=4mm,
        > = {Stealth[round, sep]},
        container/.style={draw, thick, rectangle, dashed, inner sep=0.3cm, rounded
            corners,fill=white,minimum height=2cm},
        mybackground/.style={execute at end picture={
            \begin{scope}[on background layer]
                \node[] at (current bounding box.north){\bottom{1cm} #1};
                \end{scope}
            }},
        base/.style={draw, rectangle, rounded corners, font=\sffamily},
        nnlayer/.style={base, minimum height=0.5cm, minimum width=2.5cm},  
        vnnlayer/.style={base, minimum height=2.5cm, minimum width=0.5cm},  
        var/.style={fill=white},
        col/.style={fill={#1!20},},
        simpleop/.style={draw, circle, fill=white, minimum height=0.5cm},  
        connection/.style={inner sep=0,outer sep=0},
    ]
    \newcommand{\map}[2]{$\mathbb{R}^{#1} \rightarrow \mathbb{R}^{#2}$}
        
    \node [var]                                  (VarX)  {$\mathbf{x}$};
    \node [vnnlayer, col=green, right=of VarX]   (Lin0)  {\rotatebox{90}{Linear \map{D}{2D}}};
    
    \node [connection, right=0.6 cm of Lin0]     (Conn1)  {};
    \node [vnnlayer, col=blue, right=of Conn1]   (ReLU11) {\rotatebox{90}{ReLU}};
    \node [vnnlayer, col=green, right=of ReLU11] (Lin11)  {\rotatebox{90}{Linear \map{2D}{3D}}};
    \node [vnnlayer, col=blue, right=of Lin11]   (ReLU12) {\rotatebox{90}{ReLU}};
    \node [vnnlayer, col=green, right=of ReLU12] (Lin12)  {\rotatebox{90}{Linear \map{3D}{2D}}};
    \node [simpleop, right=of Lin12]             (Add1)   {+};
    
    \begin{scope}[on background layer]
        \node[container, fit=(Conn1) (ReLU11) (Lin11) (ReLU12) (Lin12) (Add1)] (ResNet1) {};
    \end{scope}
    
    \path [->]  (Conn1.west)  edge (ReLU11)
                (ReLU11) edge (Lin11) 
                (Lin11)  edge (ReLU12)
                (ReLU12) edge (Lin12)
                (Lin12)  edge (Add1);
    \draw [->, rounded corners] (Conn1) |- ($(Lin11.north) + (0.0, 0.15)$) -| (Add1);
  
    \node [connection, right= 1cm of Add1]       (Conn2)  {};
    \node [vnnlayer, col=blue, right=of Conn2]   (ReLU21) {\rotatebox{90}{ReLU}};
    \node [vnnlayer, col=green, right=of ReLU21] (Lin21)  {\rotatebox{90}{Linear \map{2D}{3D}}};
    \node [vnnlayer, col=blue, right=of Lin21]   (ReLU22) {\rotatebox{90}{ReLU}};
    \node [vnnlayer, col=green, right=of ReLU22] (Lin22)  {\rotatebox{90}{Linear \map{3D}{2D}}};
    \node [simpleop, right=of Lin22]             (Add2)   {+};

    \begin{scope}[on background layer]
        \node[container, fit=(Conn2) (ReLU21) (Lin21) (ReLU22) (Lin22) (Add2)] (ResNet2) {};
    \end{scope}

    \path [->]  (Conn2.west)  edge (ReLU21)
                (ReLU21) edge (Lin21) 
                (Lin21)  edge (ReLU22)
                (ReLU22) edge (Lin22)
                (Lin22)  edge (Add2);
    \draw [->, rounded corners] (Conn2) |- ($(Lin21.north) + (0.0, 0.15)$) -| (Add2);

    \draw [->] (VarX) -- (Lin0);
    \draw [-]  (Lin0) -- (Conn1);
    \draw [-]  (Add1) -- (Conn2);
  
    \node [vnnlayer, col=green, above right=0.2cm and 0.8cm of Add2] (LinMean1)  {\rotatebox{90}{Linear \map{2D}{H}}};
    \node [var, right=of LinMean1]                                   (VarMu)  {$\mathbf{\vnu}$};

    \node [vnnlayer, col=green, below right=-0.2cm and 0.8cm of Add2] (LinL1)  {\rotatebox{90}{Linear \map{2D}{H}}};
    \node [vnnlayer, col=blue, right=of LinL1]                        (SoftMax1)  {\rotatebox{90}{SoftPlus}};
    \node [var, right=of SoftMax1]                                    (VarL)  {$\sigma^2$};

    \draw [->] (Add2) -- (LinMean1);
    \draw [->] (LinMean1) -- (VarMu);

    \draw [->] (Add2) -- (LinL1);
    \draw [->] (LinL1) -- (SoftMax1);
    \draw [->] (SoftMax1) -- (VarL);

\end{tikzpicture}}  
  \caption{\textbf{Deep encoder architecture}. First, the input data $\vx$ is linearly projected to a higher dimensional space, and then two ResNet-like transformations are applied. The variational parameters $\vnu$ and $\sigma^2$ are obtained by separate linear mappings. Posterior diagonal covariance is then constructed as $\TT = \diag(\sigma^2)$.}
  \label{fig:deep-encoder}
\end{figure}
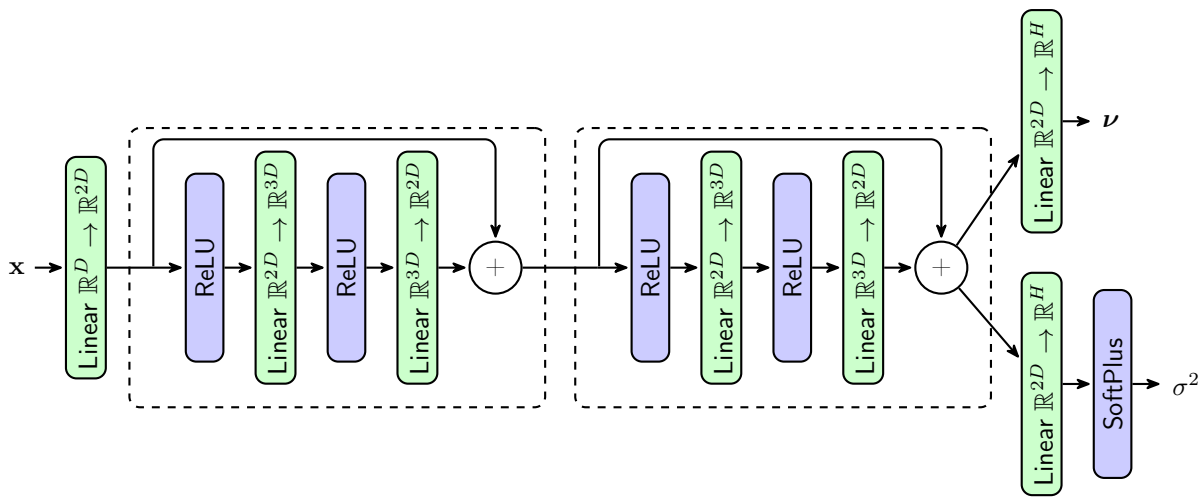

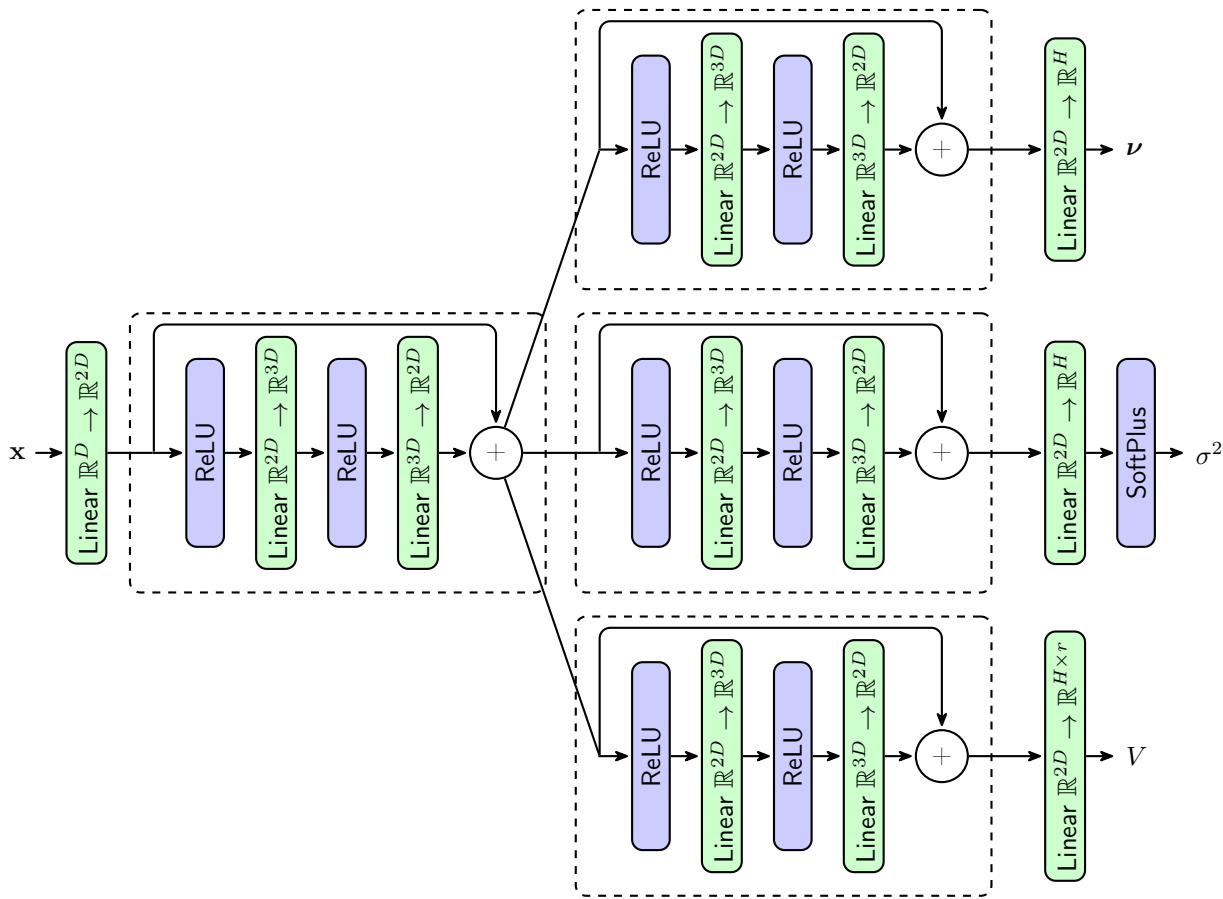
\begin{figure}[h!]
  \centering
  \scalebox{1.0}{\begin{tikzpicture}[
  thick, 
  node distance=4mm,
  > = {Stealth[round, sep]},
  container/.style={draw, thick, rectangle, dashed, inner sep=0.3cm, rounded
      corners,fill=white,minimum height=2cm},
  mybackground/.style={execute at end picture={
      \begin{scope}[on background layer]
          \node[] at (current bounding box.north){\bottom{1cm} #1};
          \end{scope}
      }},
  base/.style={draw, rectangle, rounded corners, font=\sffamily},
  nnlayer/.style={base, minimum height=0.5cm, minimum width=2.5cm},  
  vnnlayer/.style={base, minimum height=2.5cm, minimum width=0.5cm},  
  var/.style={fill=white},
  col/.style={fill={#1!20},},
  simpleop/.style={draw, circle, fill=white, minimum height=0.5cm},  
  connection/.style={inner sep=0,outer sep=0},
]
\newcommand{\map}[2]{$\mathbb{R}^{#1} \rightarrow \mathbb{R}^{#2}$}
  
\node [var]                                  (VarX)  {$\mathbf{x}$};
\node [vnnlayer, col=green, right=of VarX]   (Lin0)  {\rotatebox{90}{Linear \map{D}{2D}}};

\node [connection, right=0.6 cm of Lin0]     (Conn1)  {};
\node [vnnlayer, col=blue, right=of Conn1]   (ReLU11) {\rotatebox{90}{ReLU}};
\node [vnnlayer, col=green, right=of ReLU11] (Lin11)  {\rotatebox{90}{Linear \map{2D}{3D}}};
\node [vnnlayer, col=blue, right=of Lin11]   (ReLU12) {\rotatebox{90}{ReLU}};
\node [vnnlayer, col=green, right=of ReLU12] (Lin12)  {\rotatebox{90}{Linear \map{3D}{2D}}};
\node [simpleop, right=of Lin12]             (Add1)   {+};

\begin{scope}[on background layer]
  \node[container, fit=(Conn1) (ReLU11) (Lin11) (ReLU12) (Lin12) (Add1)] (ResNet1) {};
\end{scope}

\path [->]  (Conn1.west)  edge (ReLU11)
          (ReLU11) edge (Lin11) 
          (Lin11)  edge (ReLU12)
          (ReLU12) edge (Lin12)
          (Lin12)  edge (Add1);
\draw [->, rounded corners] (Conn1) |- ($(Lin11.north) + (0.0, 0.15)$) -| (Add1);

\node [connection, right=1cm of Add1]       (Conn2)  {};
\node [vnnlayer, col=blue, right=of Conn2]   (ReLU21) {\rotatebox{90}{ReLU}};
\node [vnnlayer, col=green, right=of ReLU21] (Lin21)  {\rotatebox{90}{Linear \map{2D}{3D}}};
\node [vnnlayer, col=blue, right=of Lin21]   (ReLU22) {\rotatebox{90}{ReLU}};
\node [vnnlayer, col=green, right=of ReLU22] (Lin22)  {\rotatebox{90}{Linear \map{3D}{2D}}};
\node [simpleop, right=of Lin22]             (Add2)   {+};

\begin{scope}[on background layer]
  \node[container, fit=(Conn2) (ReLU21) (Lin21) (ReLU22) (Lin22) (Add2)] (ResNet2) {};
\end{scope}

\path [->]  (Conn2.west)  edge (ReLU21)
          (ReLU21) edge (Lin21) 
          (Lin21)  edge (ReLU22)
          (ReLU22) edge (Lin22)
          (Lin22)  edge (Add2);
\draw [->, rounded corners] (Conn2) |- ($(Lin21.north) + (0.0, 0.15)$) -| (Add2);

\draw [->] (VarX) -- (Lin0);
\draw [-]  (Lin0) -- (Conn1);
\draw [-]  (Add1) -- (Conn2);

\node [connection, below=4cm of Conn2]       (Conn3)  {};
\node [vnnlayer, col=blue, right=of Conn3]   (ReLU31) {\rotatebox{90}{ReLU}};
\node [vnnlayer, col=green, right=of ReLU31] (Lin31)  {\rotatebox{90}{Linear \map{2D}{3D}}};
\node [vnnlayer, col=blue, right=of Lin31]   (ReLU32) {\rotatebox{90}{ReLU}};
\node [vnnlayer, col=green, right=of ReLU32] (Lin32)  {\rotatebox{90}{Linear \map{3D}{2D}}};
\node [simpleop, right=of Lin32]             (Add3)   {+};

\begin{scope}[on background layer]
  \node[container, fit=(Conn3) (ReLU31) (Lin31) (ReLU32) (Lin32) (Add3)] (ResNet3) {};
\end{scope}

\path [->]  (Conn3.west)  edge (ReLU31)
          (ReLU31) edge (Lin31) 
          (Lin31)  edge (ReLU32)
          (ReLU32) edge (Lin32)
          (Lin32)  edge (Add3);
\draw [->, rounded corners] (Conn3) |- ($(Lin31.north) + (0.0, 0.15)$) -| (Add3);

\draw [-]  (Add1) -- (Conn3);

\node [connection, above=4cm of Conn2]       (Conn4)  {};
\node [vnnlayer, col=blue, right=of Conn4]   (ReLU41) {\rotatebox{90}{ReLU}};
\node [vnnlayer, col=green, right=of ReLU41] (Lin41)  {\rotatebox{90}{Linear \map{2D}{3D}}};
\node [vnnlayer, col=blue, right=of Lin41]   (ReLU42) {\rotatebox{90}{ReLU}};
\node [vnnlayer, col=green, right=of ReLU42] (Lin42)  {\rotatebox{90}{Linear \map{3D}{2D}}};
\node [simpleop, right=of Lin42]             (Add4)   {+};

\begin{scope}[on background layer]
  \node[container, fit=(Conn4) (ReLU41) (Lin41) (ReLU42) (Lin42) (Add4)] (ResNet4) {};
\end{scope}

\path [->]  (Conn4.west)  edge (ReLU41)
          (ReLU41) edge (Lin41) 
          (Lin41)  edge (ReLU42)
          (ReLU42) edge (Lin42)
          (Lin42)  edge (Add4);
\draw [->, rounded corners] (Conn4) |- ($(Lin41.north) + (0.0, 0.15)$) -| (Add4);

\draw [-]  (Add1) -- (Conn4);

\node [vnnlayer, col=green, right=1cm of Add4] (LinMean1)  {\rotatebox{90}{Linear \map{2D}{H}}};
\node [var, right=of LinMean1]                                   (VarMu)  {$\mathbf{\vnu}$};
\draw [->] (Add4) -- (LinMean1);
\draw [->] (LinMean1) -- (VarMu);

\node [vnnlayer, col=green, right=1cm of Add2] (LinL1)  {\rotatebox{90}{Linear \map{2D}{H}}};
\node [vnnlayer, col=blue, right=of LinL1]                        (SoftMax1)  {\rotatebox{90}{SoftPlus}};
\node [var, right=of SoftMax1]                                    (VarL)  {$\sigma^2$};
\draw [->] (Add2) -- (LinL1);
\draw [->] (LinL1) -- (SoftMax1);
\draw [->] (SoftMax1) -- (VarL);

\node [vnnlayer, col=green, right=1cm of Add3] (LinLoRank1)  {\rotatebox{90}{Linear \map{2D}{H \times r}}};
\node [var, right=of LinLoRank1]                                   (VarV)  {$V$};
\draw [->] (Add3) -- (LinLoRank1);
\draw [->] (LinLoRank1) -- (VarV);

\end{tikzpicture}}  
  \caption{\textbf{Encoder architecture for variational posterior with a low-rank approximation of full covariance}. The covariance matrix is constructed as $\TT = VV^\mathrm{T} +  \diag(\sigma^2)$.
  }
  \label{fig:deep-encoder-low-rank}
\end{figure}

\begin{figure}[ht!]
  \centering
  \begin{minipage}[t]{0.3\textwidth}
    \centering
    \includegraphics[width=\linewidth, trim={2cm 2cm 2cm 2cm}, clip]{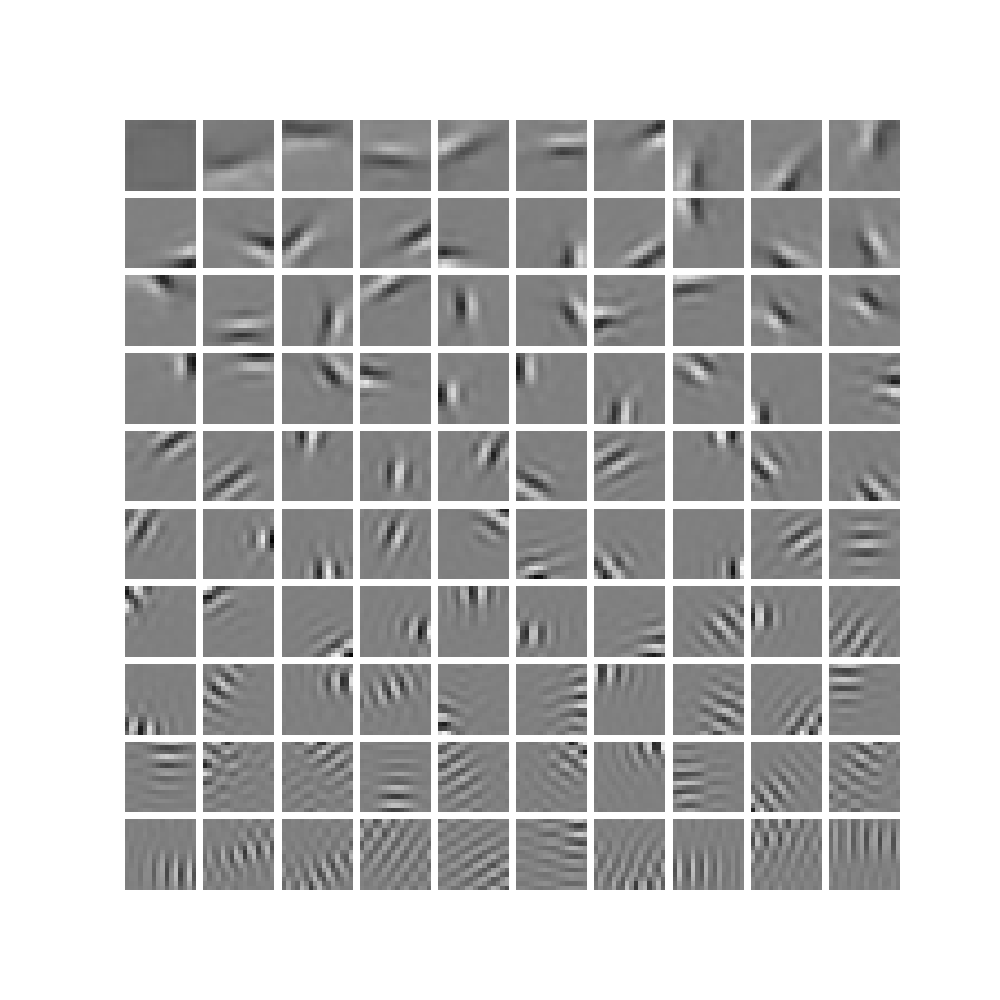}
	\subcaption{EM, annealing}
  \end{minipage}
  \hfill
  \begin{minipage}[t]{0.3\textwidth}
    \centering
    \includegraphics[width=\linewidth, trim={2cm 2cm 2cm 2cm}, clip]{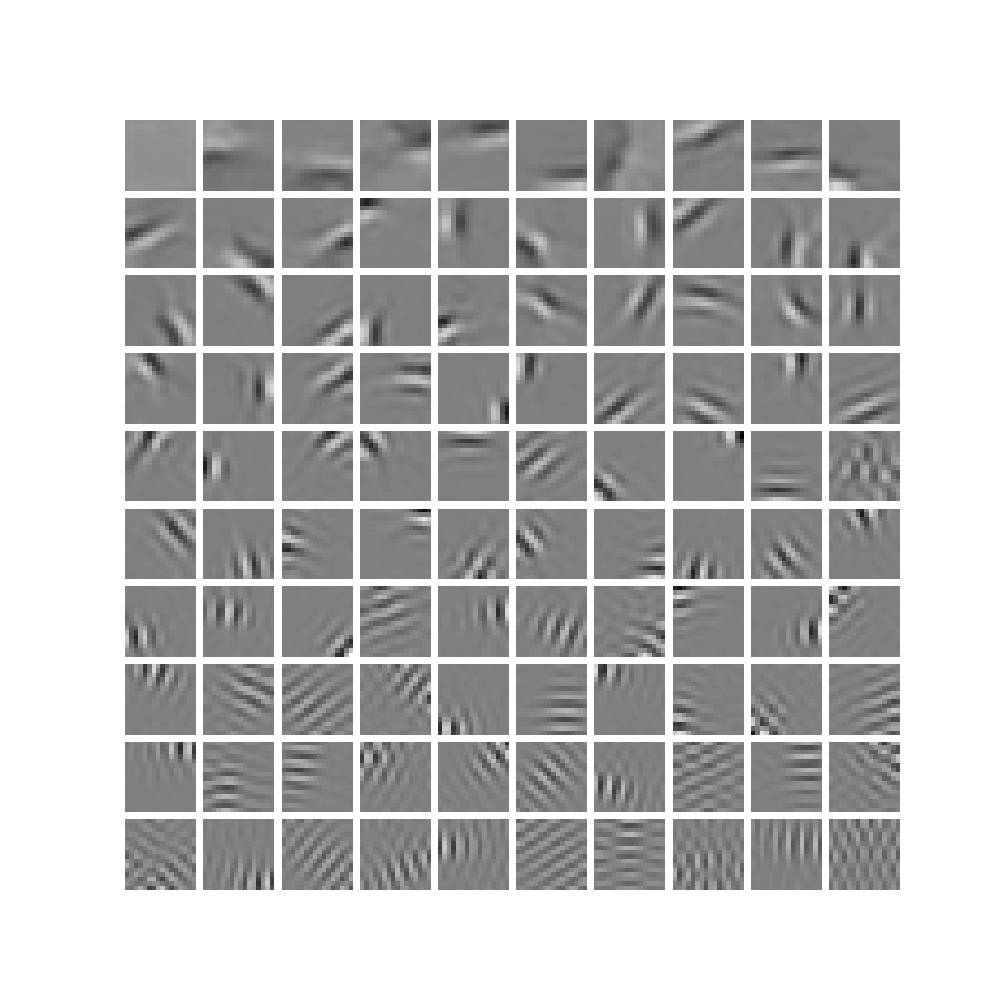}
	\subcaption{EM, no annealing}
  \end{minipage}
  \hfill
  \begin{minipage}[t]{0.3\textwidth}
    \centering
    \includegraphics[width=\linewidth, trim={2cm 2cm 2cm 2cm}, clip]{images/vanhateren/olshausen-100/compare-optimization/Ws-200-amortized.png}
	\subcaption{Adam, amortized}
  \end{minipage}
  \caption{\textbf{Learned bases with different optimization methods} where the generative fields were obtained after 200 epochs of optimization. All methods result in practically the same set of filters, but the prior entropy annealing facilitates fast convergence.}
  \label{fig:compare-optimization-gabors}
\end{figure}

\subsection{Comparing annealing schemes}
\label{app:compare-annealing}
We used a basic linear annealing scheme for prior entropy annealing: $\gamma_i = \max(1.0, 2*(5-i))$ for epoch $i$. For the likelihood entropy annealing, we set $\delta_i = \min(1.0, 1/(7-i))$.
While prior (\cref{fig:prior-annealing-all}) and likelihood (\cref{fig:beta-annealing-all}) entropy annealing result in similar generative fields after convergence, the trajectories of the optimization of generative fields and latent codes differ. With prior entropy annealing, all the latent dimensions are used for the encoding from the very beginning of the optimization. In the case of the likelihood annealing, the latent dimensions start contributing gradually to the reconstruction.

\cref{tab:annealing} provides numerical details and gives some insights into how the model parameters behave during the above-mentioned annealing. 
The table shows ELBOs, Gini coefficients, and contributions of different entropies to the entropy-based ELBO. To compute the ELBO, after every epoch, we evaluated the non-annealed ELBO for the learned parameters and the full dataset. Thus, the highest (and also the only proper) ELBO can be obtained only when a non-annealed objective is used for the optimization. We selected some of the annealing epochs, for which the contribution of the annealing coefficient causes a quantitatively similar balance of the contributing entropies to the annealed ELBO, that is, e.g., for the epochs when $\gamma=2$ and $\delta=0.5$ the corresponding contributing entropies are close. Despite the similarity in the values of the entropies, the learned generative fields are qualitatively very different (\cref{fig:prior-annealing-all} and \cref{fig:beta-annealing-all}). Next, we provide an explanation of what causes such a qualitative difference.

Notice that the Gini coefficient of the latent codes is high during the annealing, which indicates high sparsity of the posterior. Here we have to remember that the likelihood annealing trades off the reconstruction quality to Kullback-Leibler divergence between the prior and the variational posterior. With small $\delta$ we can largely ignore the reconstruction term and focus only on the Kullback-Leibler divergence. Our reparameterized model allows two ways to minimize the Kullback-Leibler divergence term: by adjusting the variational posterior parameters, and by changing the prior scales. 
We observe both phenomena in the case of the likelihood entropy annealing, which leads to noisy and non-localized generative fields and posterior collapse of some of the latent dimensions. 
That is, some of the latent dimensions do not participate in the encoding, their corresponding scale parameters $\lambda_h$ shrink to very small values, the corresponding contribution to the Kullback-Leibler divergence may become arbitrarily close to 0, and the corresponding generative fields do not contribute to the data reconstruction.
\cref{fig:beta-annealing-all} illustrates such noisy generative fields, which do not contribute to the reconstruction. Noisy and non-localized generative fields entirely disappear as the annealing ends.

\begin{figure}[!ht]
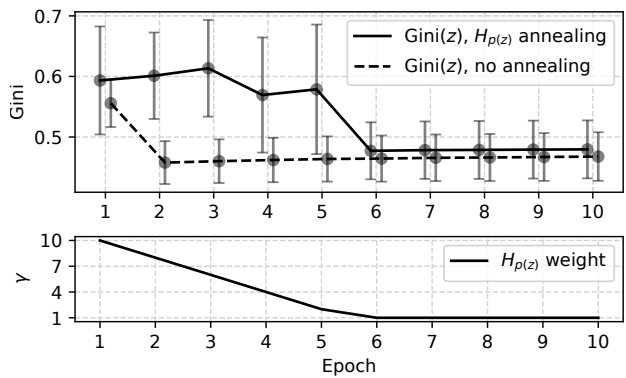

		\centering
		\includegraphics[width=0.5\linewidth, trim={0 0.6cm 0 0cm}, clip]{images/vanhateren/olshausen-100/mini-vanhateren-gini-notitle.pdf}
		\includegraphics[width=0.5\linewidth]{images/vanhateren/olshausen-100/annealing-notitle.pdf}
	\caption{\textbf{Prior entropy annealing on natural image patches dataset}.
            Gini coefficients (mean $\pm$SD bars) of the latent codes (\cref{fig:image-patches-dataset} for example generative fields) stay marginally higher if the prior entropy annealing is used even after the annealing ends after epoch 5. The bottom plot shows the annealing schedule.
            }
	\label{fig:image-patches-ELBO}
\end{figure}

The qualitative difference of prior entropy annealing can be seen by comparing the generative fields we obtain during the optimization (\cref{fig:prior-annealing-all}). Even after one epoch with high weight on the prior entropy, the generative fields already resemble localized Gabor filters. As the annealing decays, more generative fields that represent high-frequency Gabors emerge. \cref{fig:image-patches-ELBO} shows the linear annealing schedule and how the Gini coefficient changes during the prior entropy annealing.

\begin{figure}[ht!]
  \centering
  \begin{minipage}[t]{0.3\textwidth}
    \centering
    \includegraphics[width=\linewidth, trim={2cm 2cm 2cm 2cm}, clip]{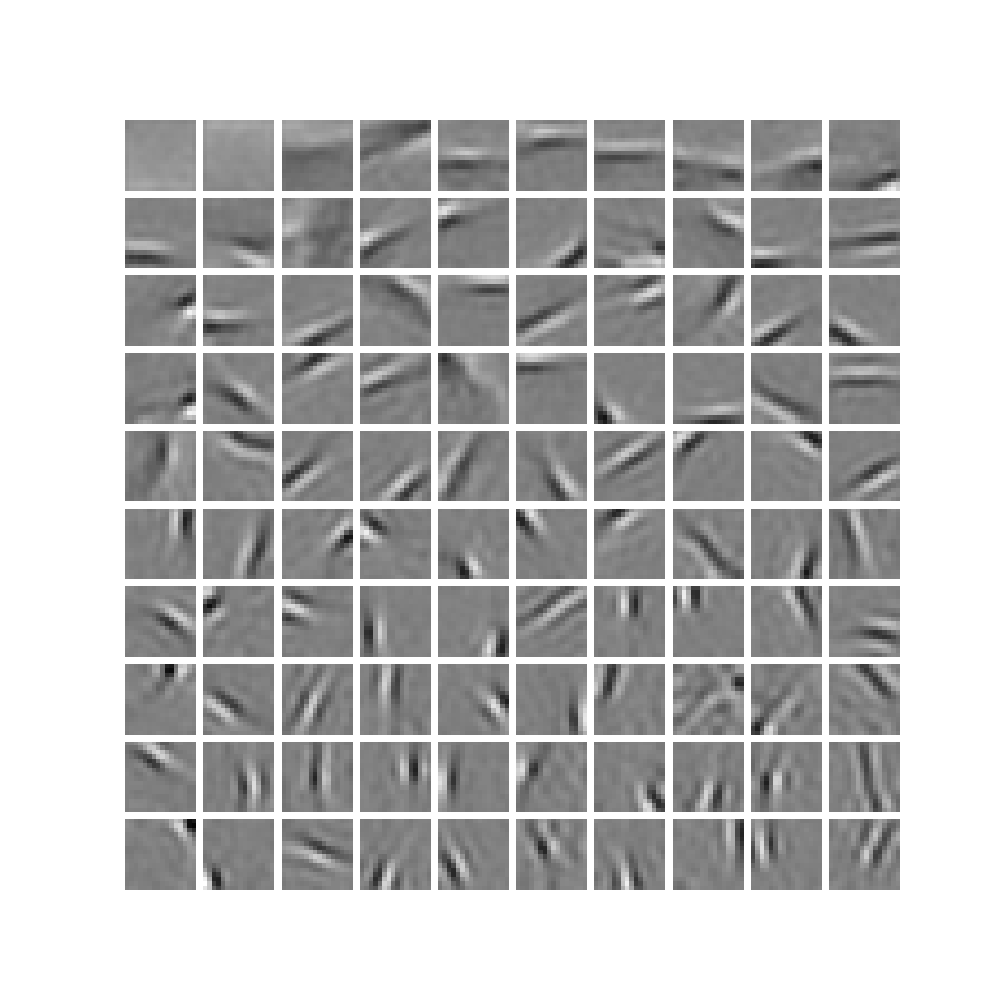}
	\subcaption{Epoch 1}
  \end{minipage}
  \hfill
  \begin{minipage}[t]{0.3\textwidth}
    \centering
    \includegraphics[width=\linewidth, trim={2cm 2cm 2cm 2cm}, clip]{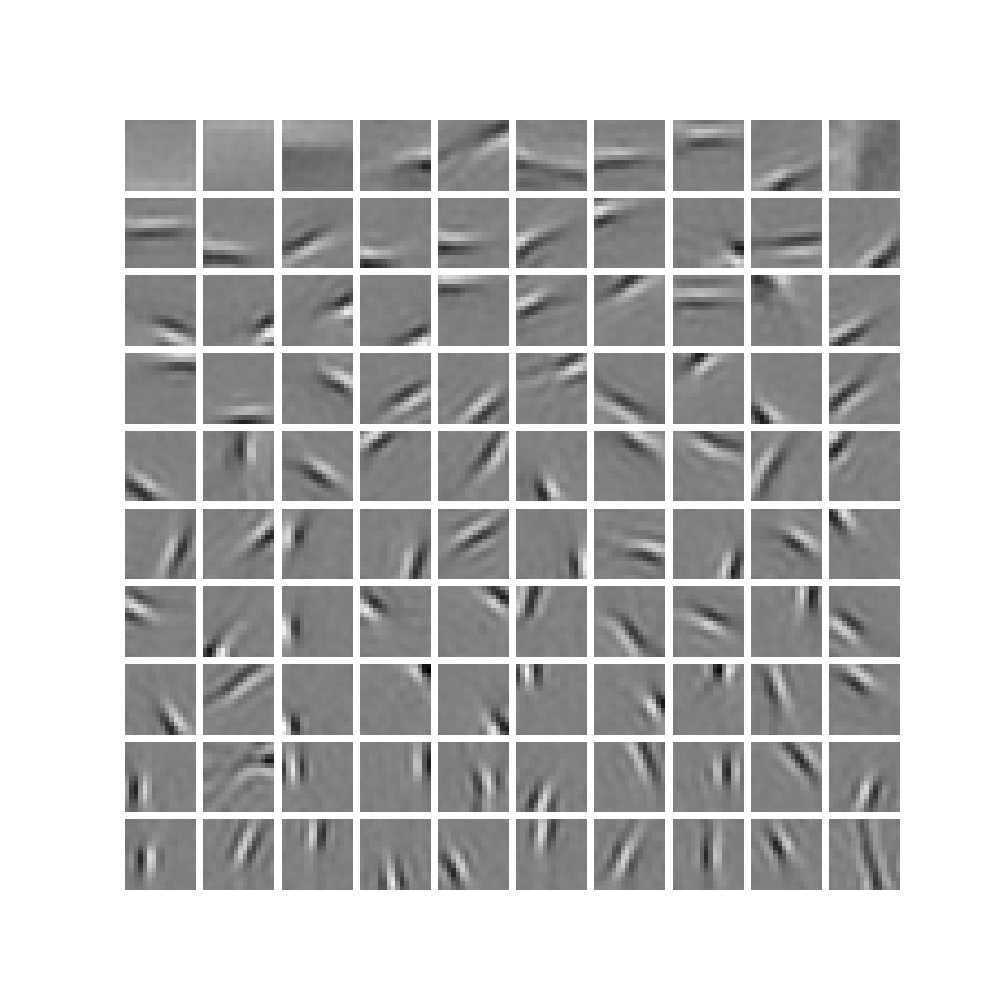}
	\subcaption{Epoch 2}
  \end{minipage}
  \hfill
  \begin{minipage}[t]{0.3\textwidth}
    \centering
    \includegraphics[width=\linewidth, trim={2cm 2cm 2cm 2cm}, clip]{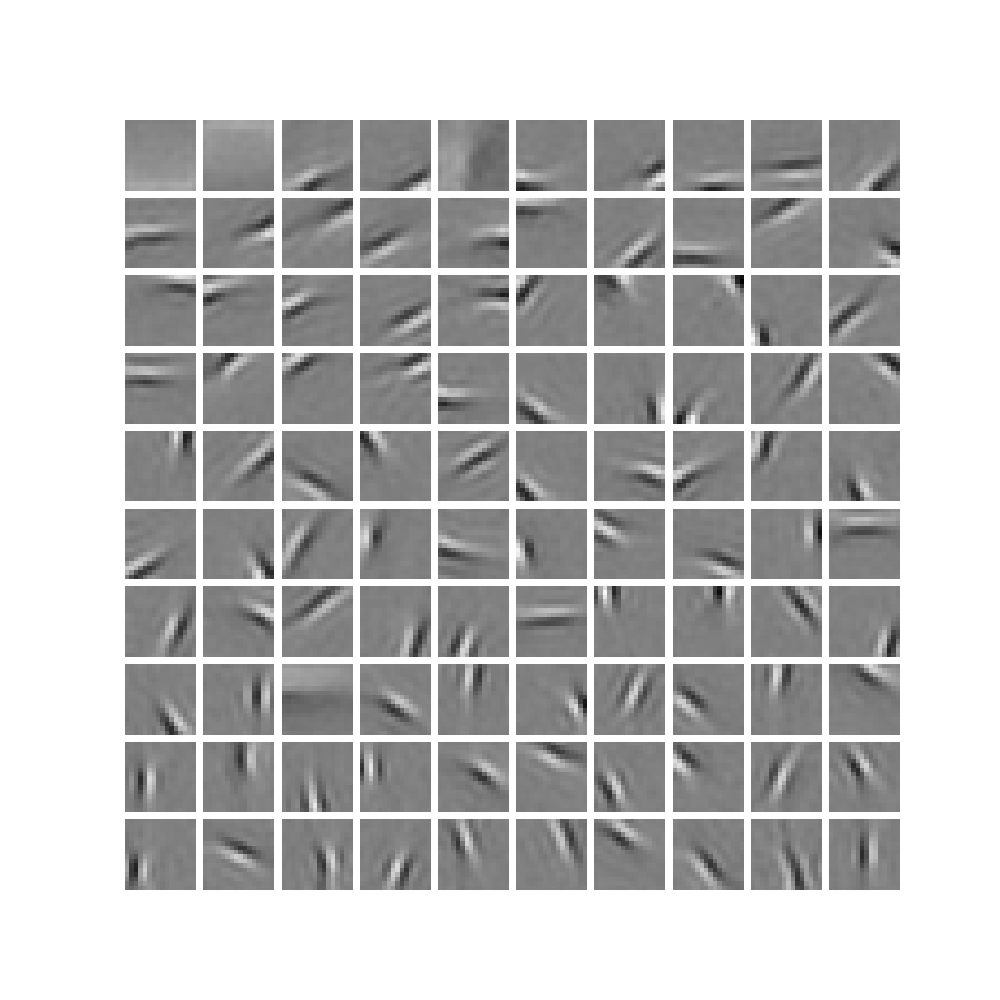}
	\subcaption{Epoch 3}
  \end{minipage} 
  \\
  \vspace{0.5cm}
  \begin{minipage}[t]{0.3\textwidth}
    \centering
    \includegraphics[width=\linewidth, trim={2cm 2cm 2cm 2cm}, clip]{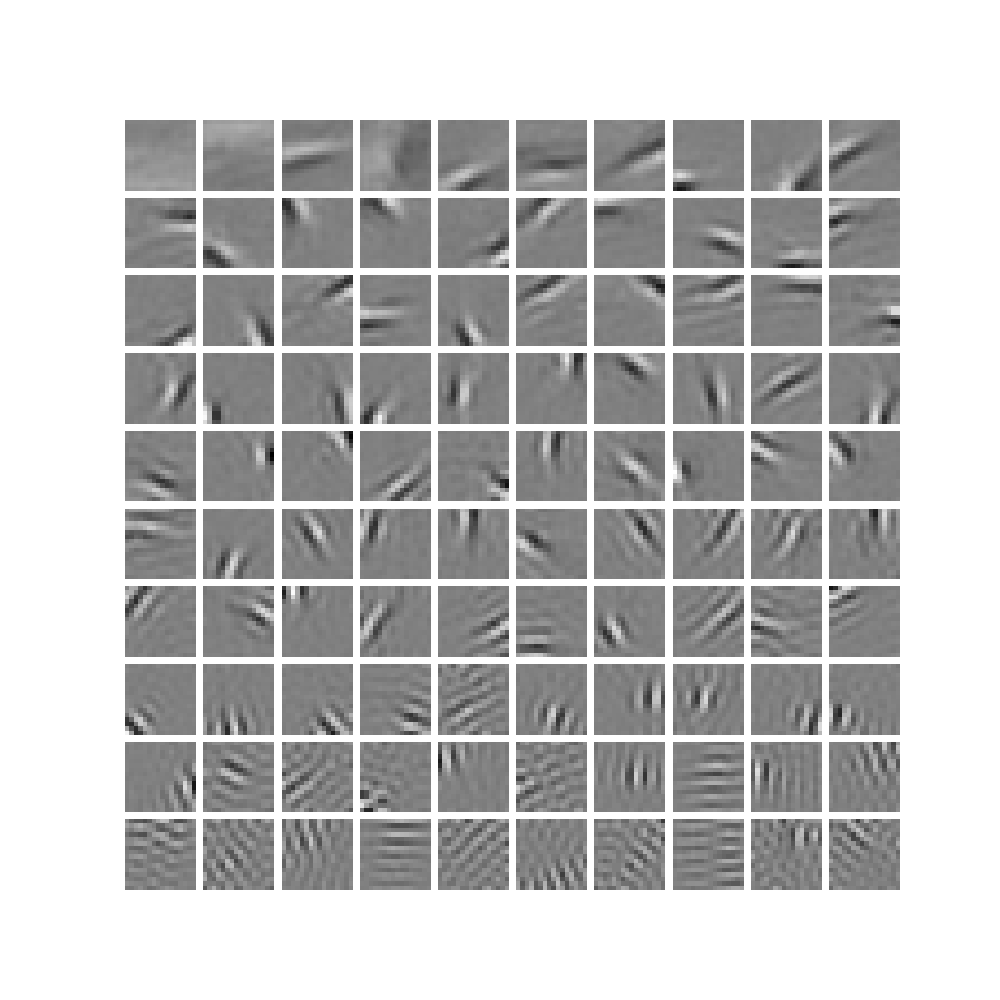}
	\subcaption{Epoch 4}
  \end{minipage}
  \hfill
  \begin{minipage}[t]{0.3\textwidth}
    \centering
    \includegraphics[width=\linewidth, trim={2cm 2cm 2cm 2cm}, clip]{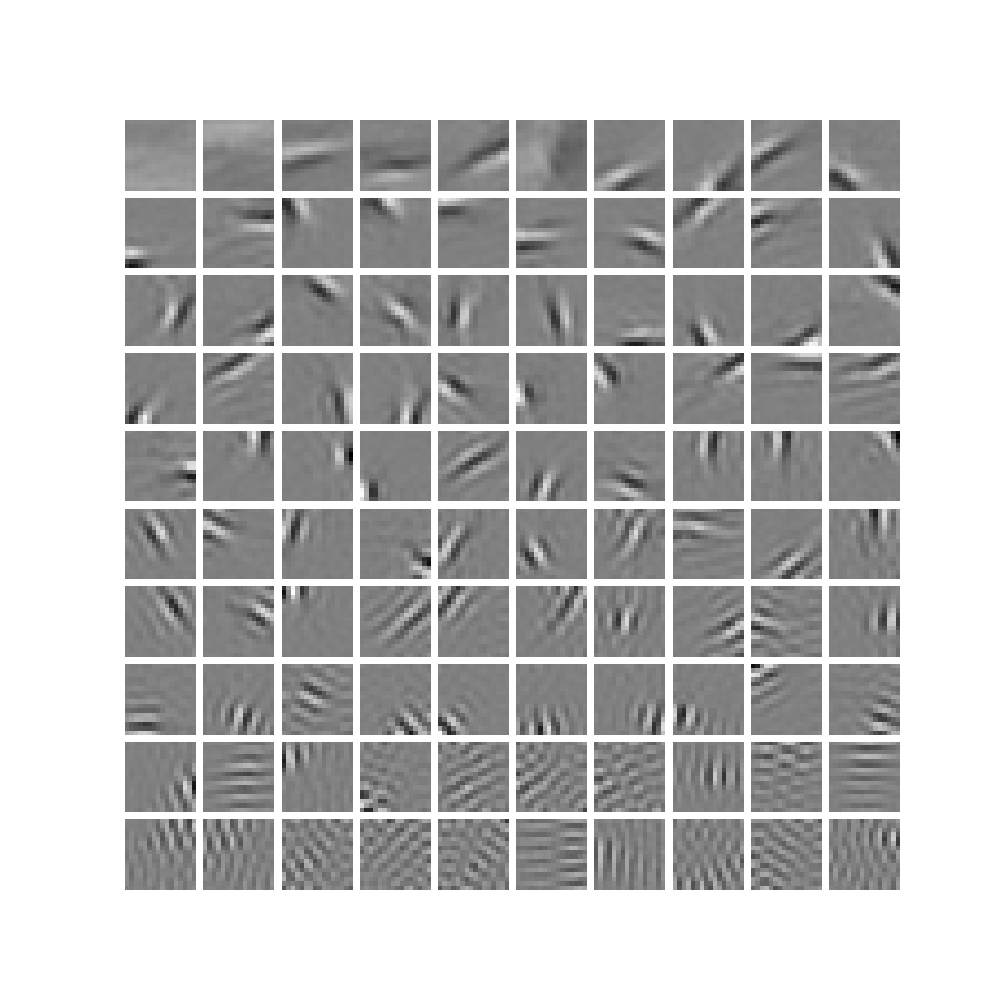}
	\subcaption{Epoch 5}
  \end{minipage}
  \hfill
  \begin{minipage}[t]{0.3\textwidth}
    \centering
    \includegraphics[width=\linewidth, trim={2cm 2cm 2cm 2cm}, clip]{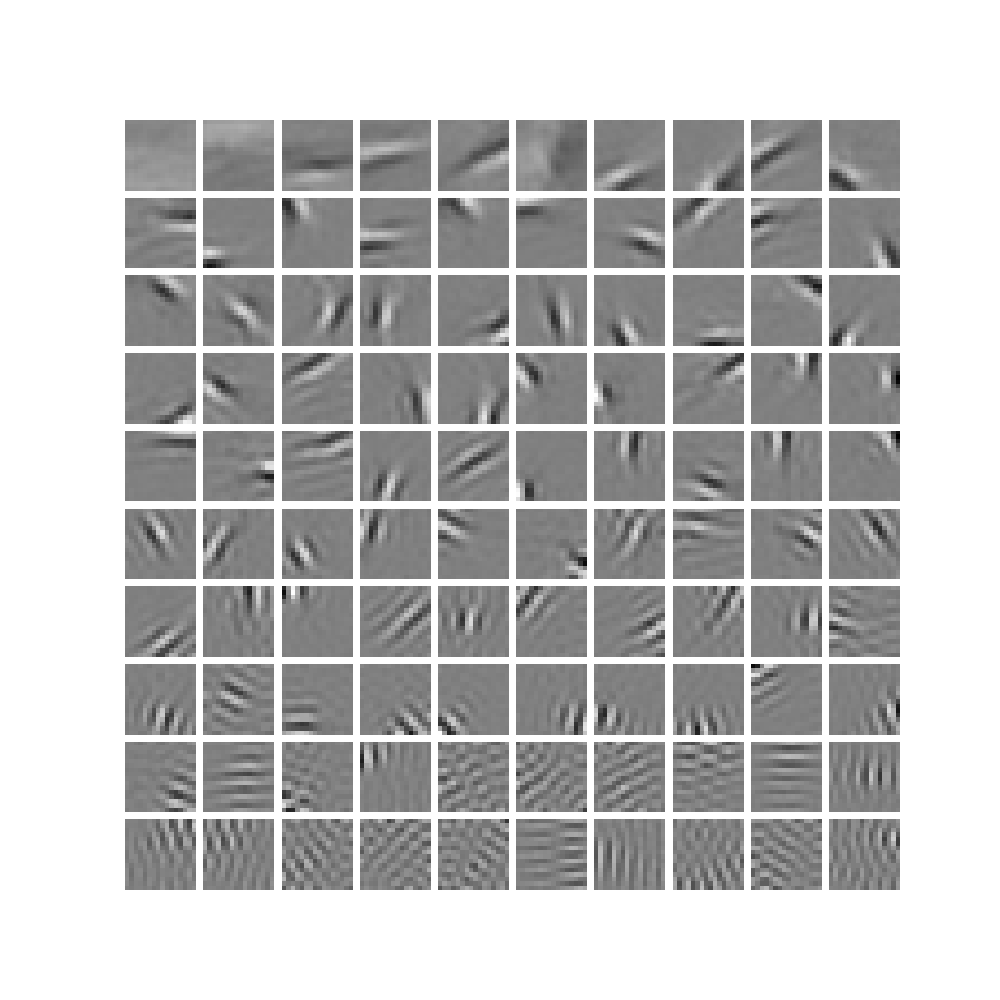}
	\subcaption{Epoch 6}
  \end{minipage}
  \\
  \vspace{0.5cm}
  \begin{minipage}[t]{0.3\textwidth}
    \centering
    \includegraphics[width=\linewidth, trim={2cm 2cm 2cm 2cm}, clip]{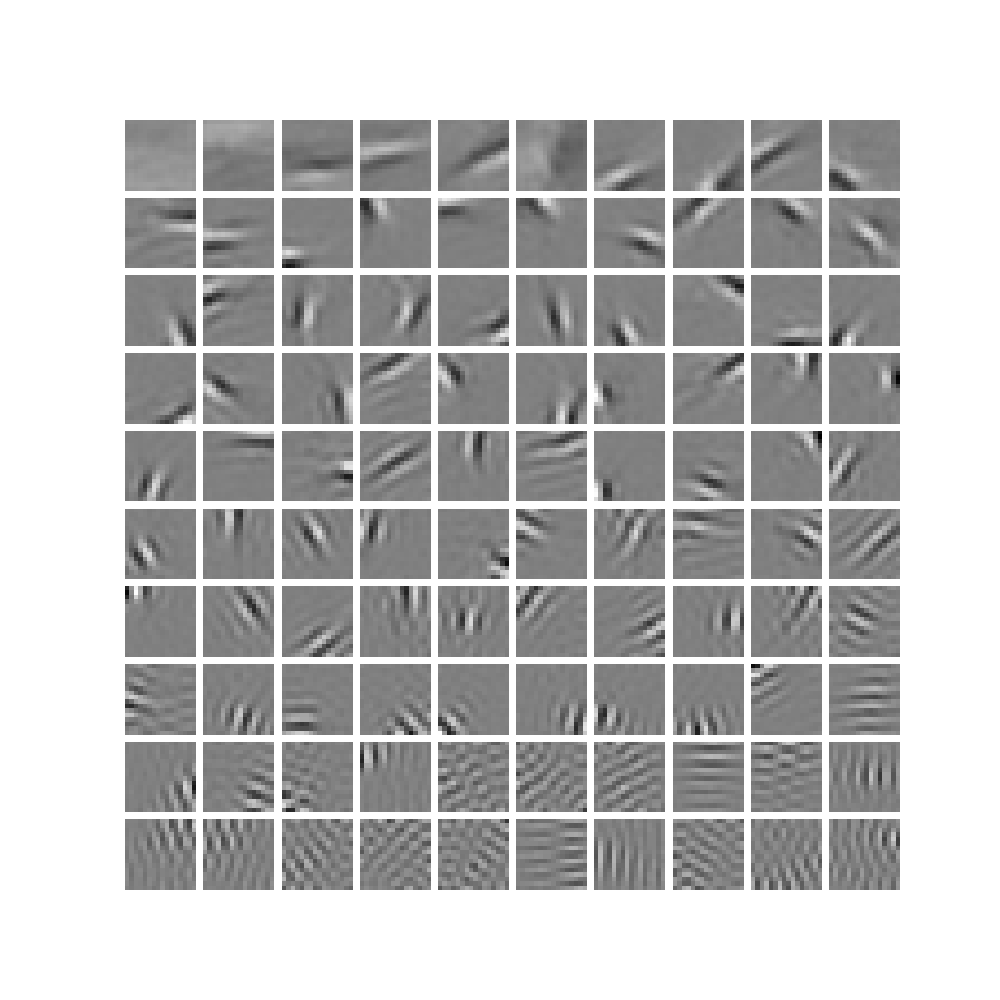}
	\subcaption{Epoch 7}
  \end{minipage}
  \hfill
  \begin{minipage}[t]{0.3\textwidth}
    \centering
    \includegraphics[width=\linewidth, trim={2cm 2cm 2cm 2cm}, clip]{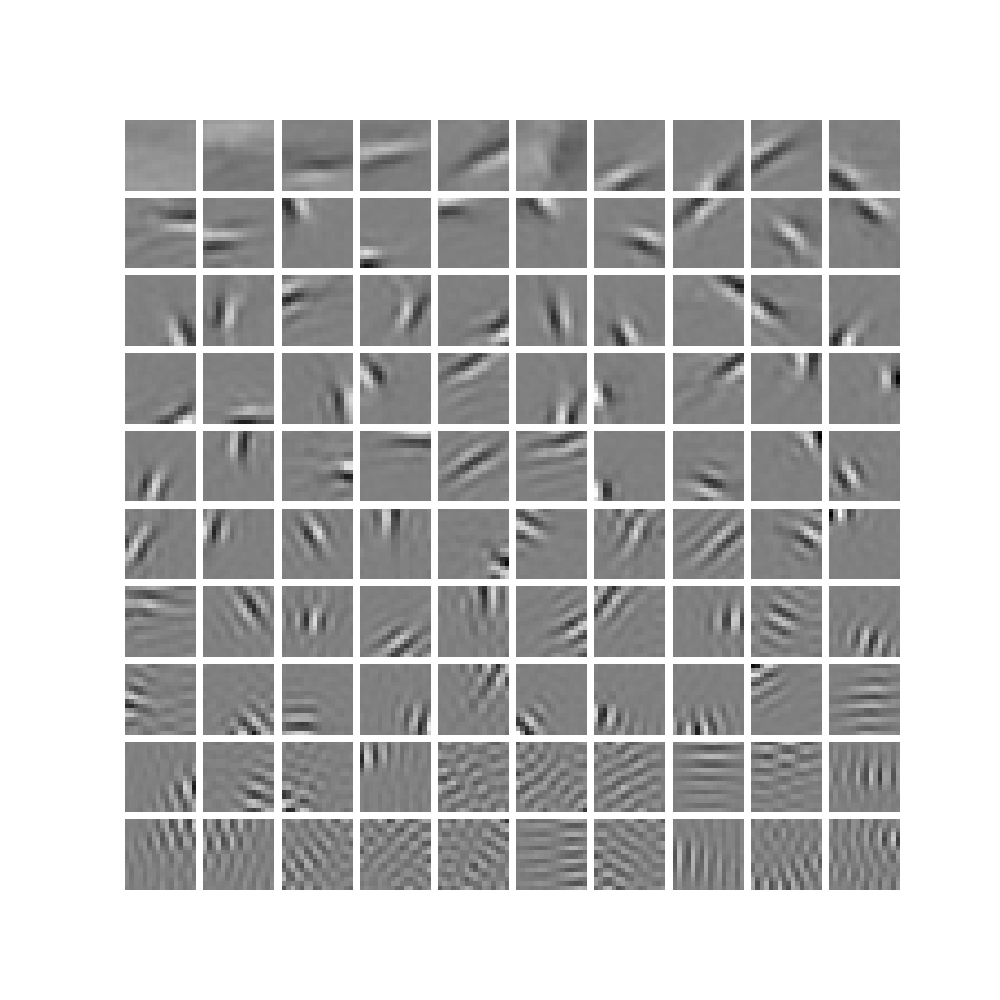}
	\subcaption{Epoch 8}
  \end{minipage}
  \hfill
  \begin{minipage}[t]{0.3\textwidth}
    \centering
    \includegraphics[width=\linewidth, trim={2cm 2cm 2cm 2cm}, clip]{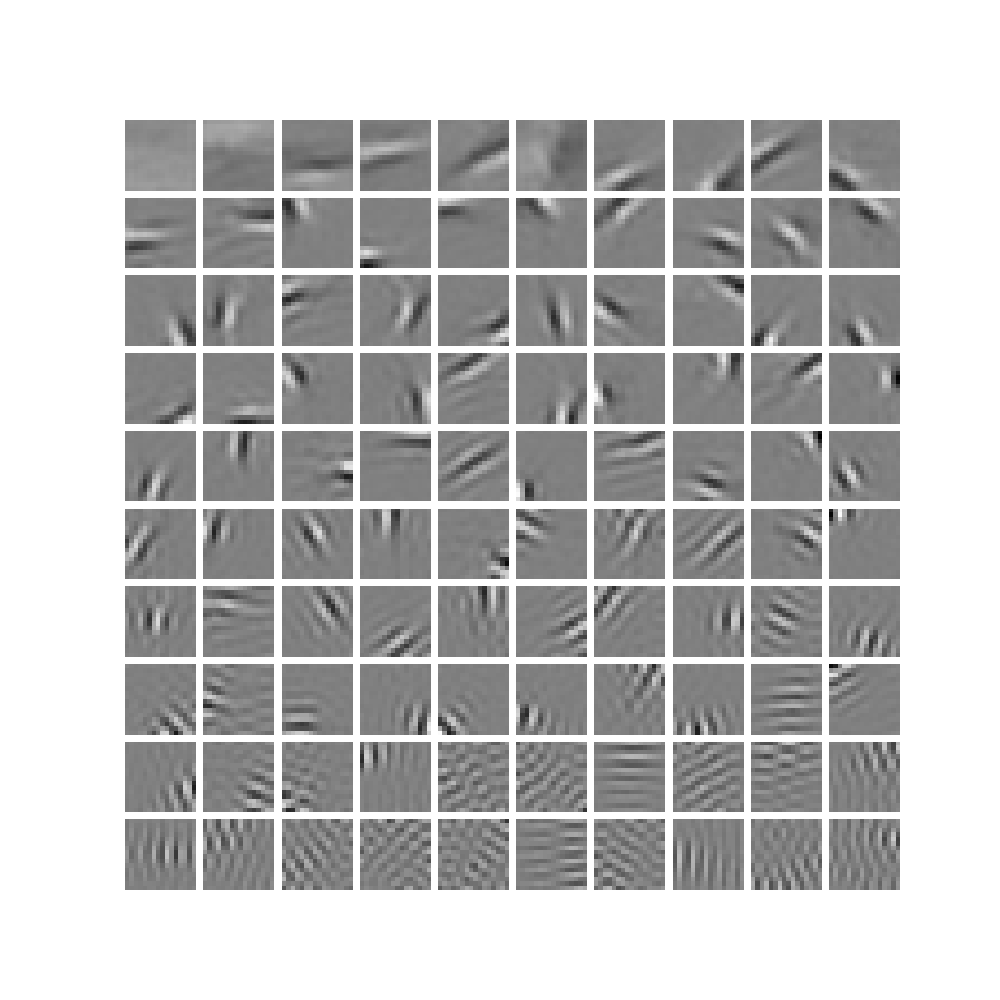}
	\subcaption{Epoch 9}
  \end{minipage}
  \caption{\textbf{Learned generative fields during optimization with prior entropy annealing}. The annealing stops after epoch 5.}
  \label{fig:prior-annealing-all}
\end{figure}

\begin{figure}[ht!]
  \centering
  \begin{minipage}[t]{0.3\textwidth}
    \centering
    \includegraphics[width=\linewidth, trim={2cm 2cm 2cm 2cm}, clip]{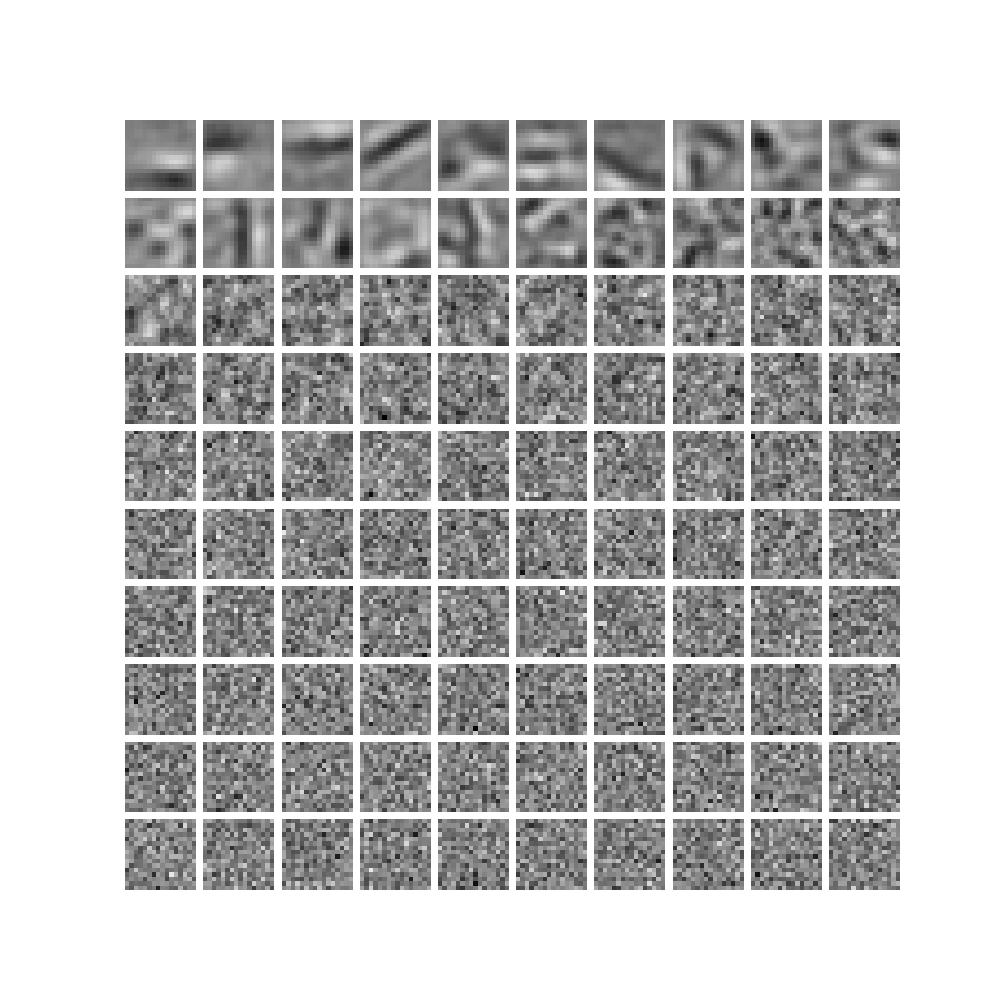}
	\subcaption{Epoch 1}
  \end{minipage}
  \hfill
  \begin{minipage}[t]{0.3\textwidth}
    \centering
    \includegraphics[width=\linewidth, trim={2cm 2cm 2cm 2cm}, clip]{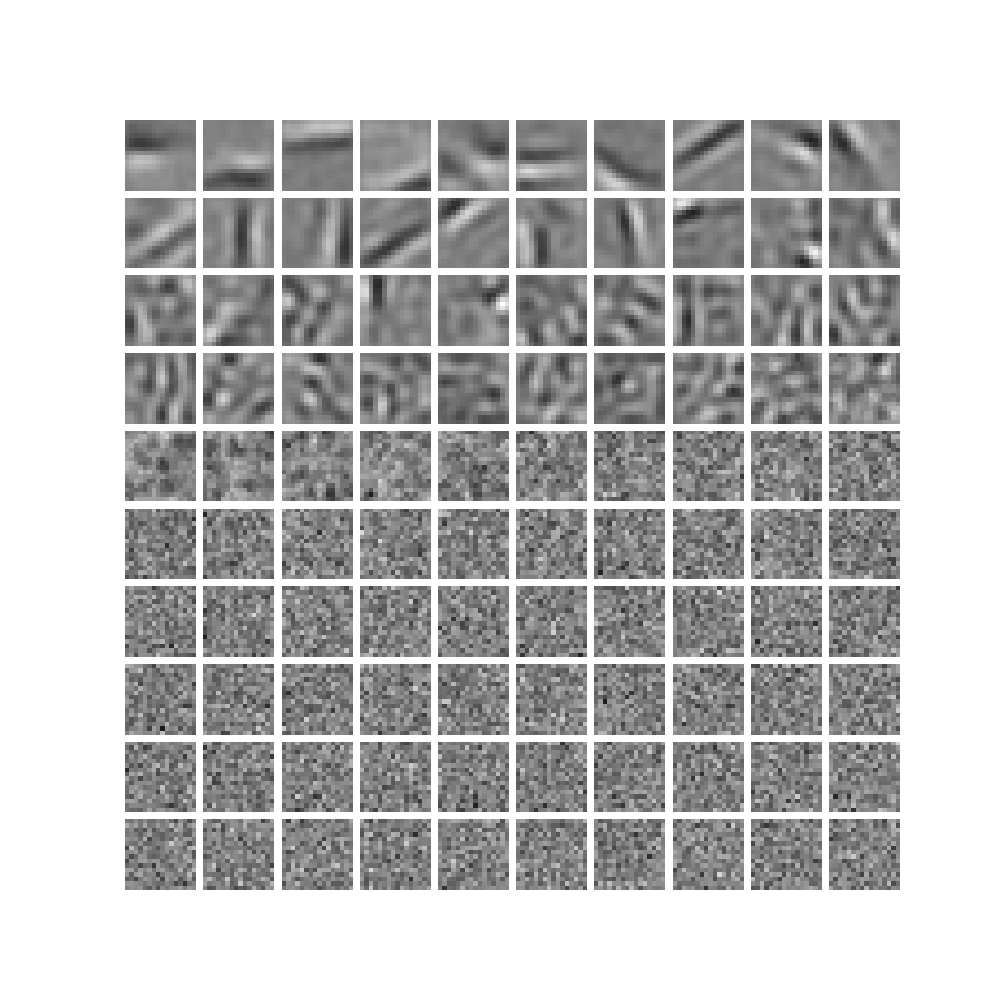}
	\subcaption{Epoch 2}
  \end{minipage}
  \hfill
  \begin{minipage}[t]{0.3\textwidth}
    \centering
    \includegraphics[width=\linewidth, trim={2cm 2cm 2cm 2cm}, clip]{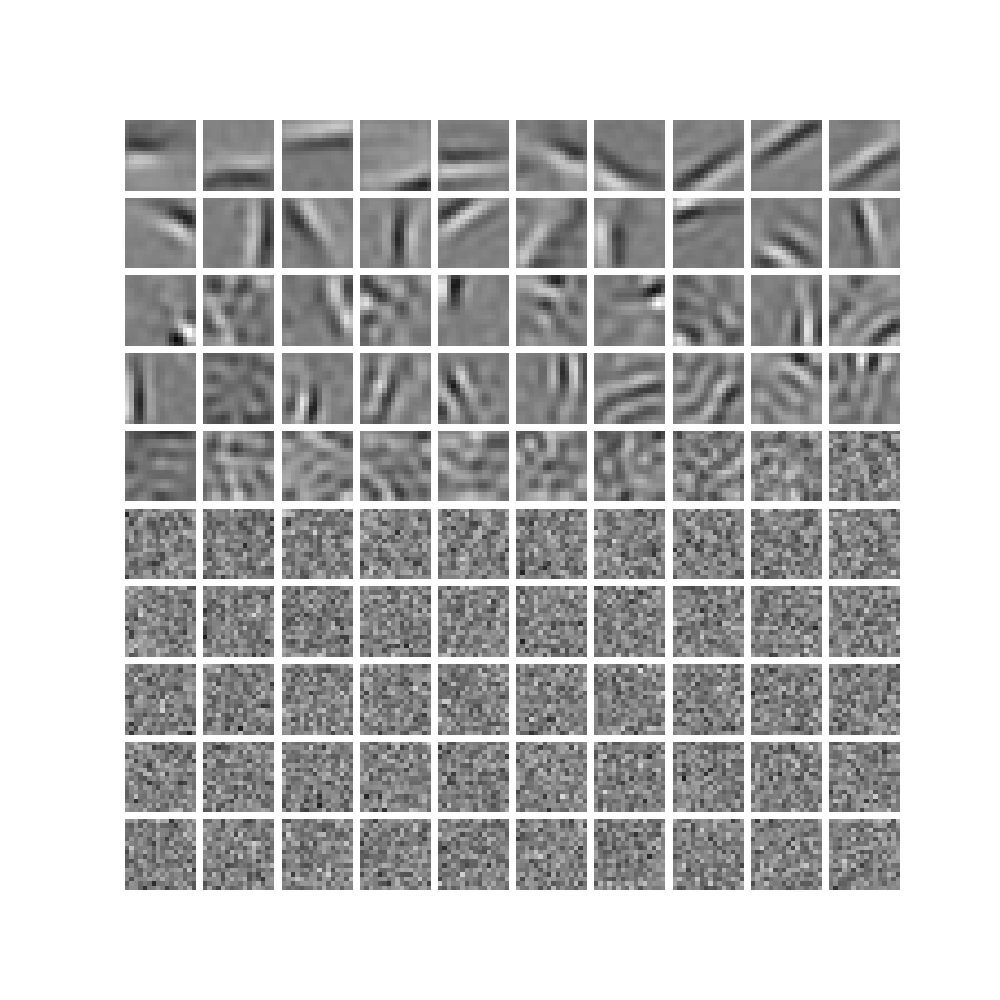}
	\subcaption{Epoch 3}
  \end{minipage} 
  \\
  \vspace{0.5cm}
  \begin{minipage}[t]{0.3\textwidth}
    \centering
    \includegraphics[width=\linewidth, trim={2cm 2cm 2cm 2cm}, clip]{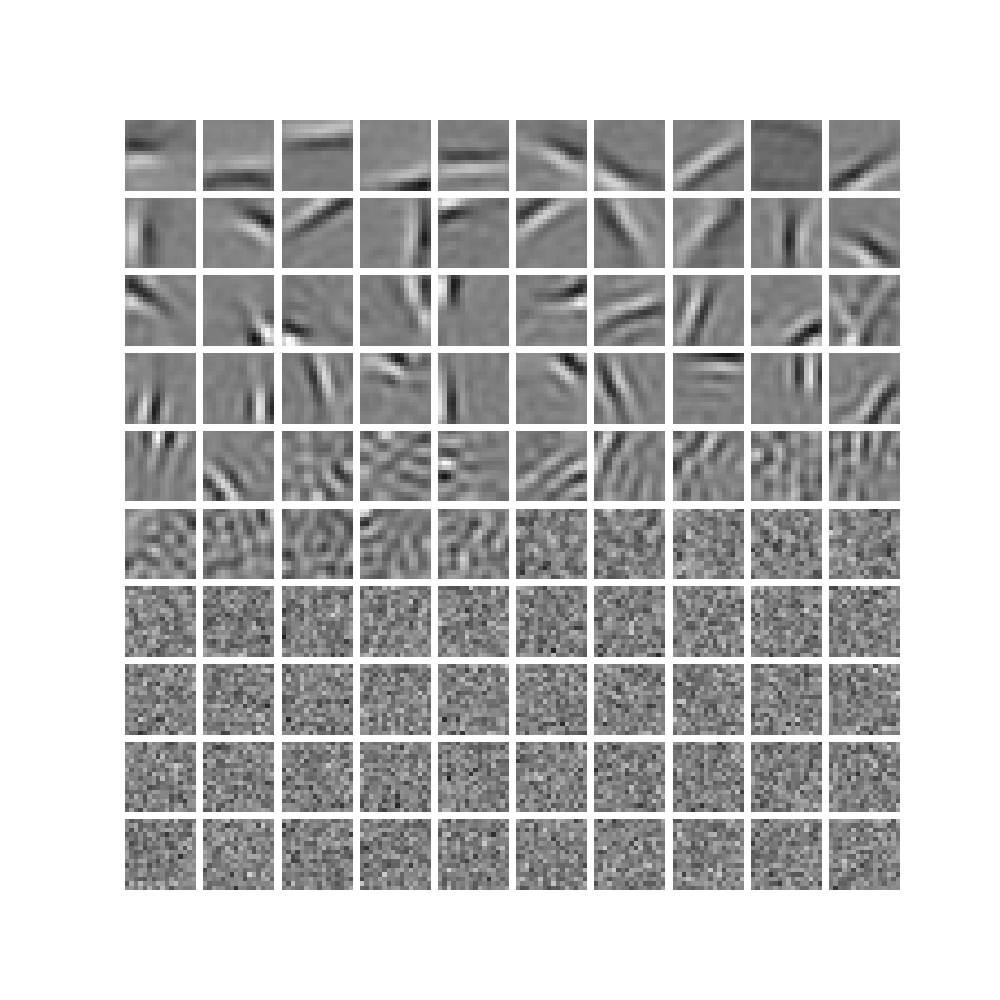}
	\subcaption{Epoch 4}
  \end{minipage}
  \hfill
  \begin{minipage}[t]{0.3\textwidth}
    \centering
    \includegraphics[width=\linewidth, trim={2cm 2cm 2cm 2cm}, clip]{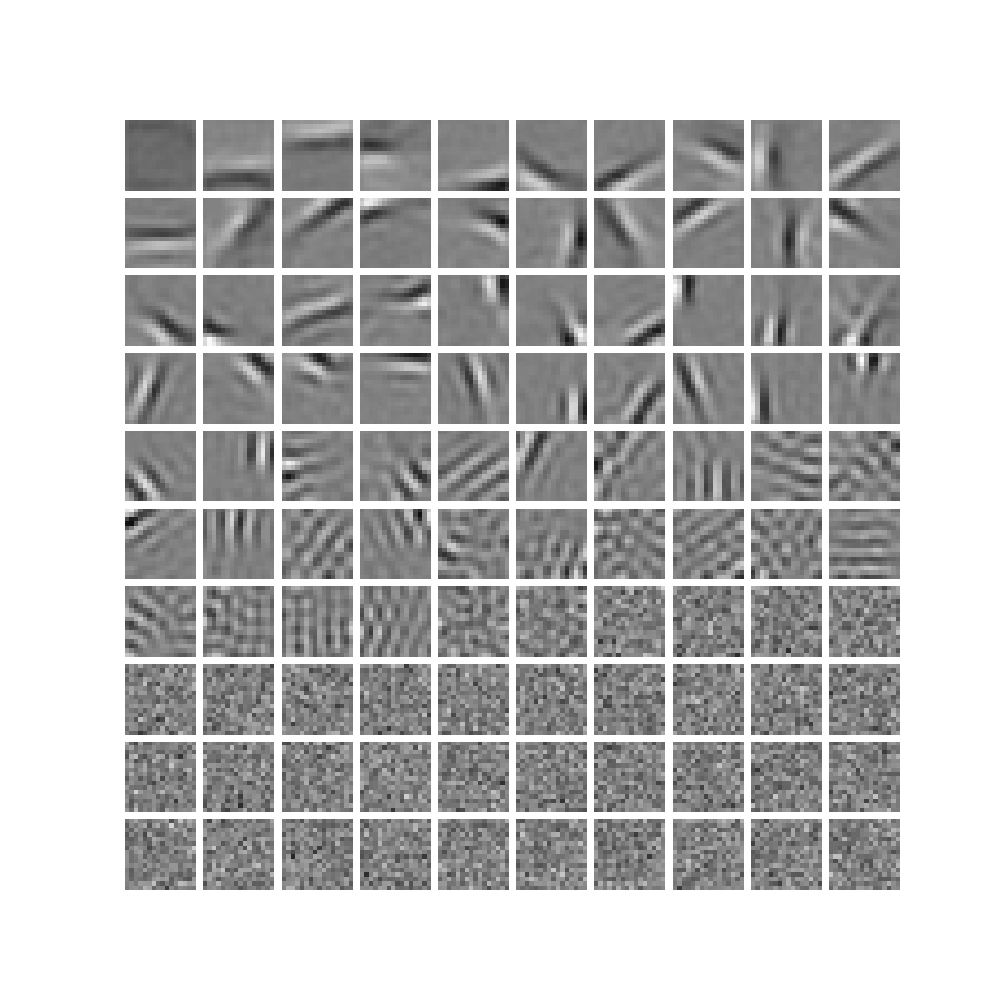}
	\subcaption{Epoch 5}
  \end{minipage}
  \hfill
  \begin{minipage}[t]{0.3\textwidth}
    \centering
    \includegraphics[width=\linewidth, trim={2cm 2cm 2cm 2cm}, clip]{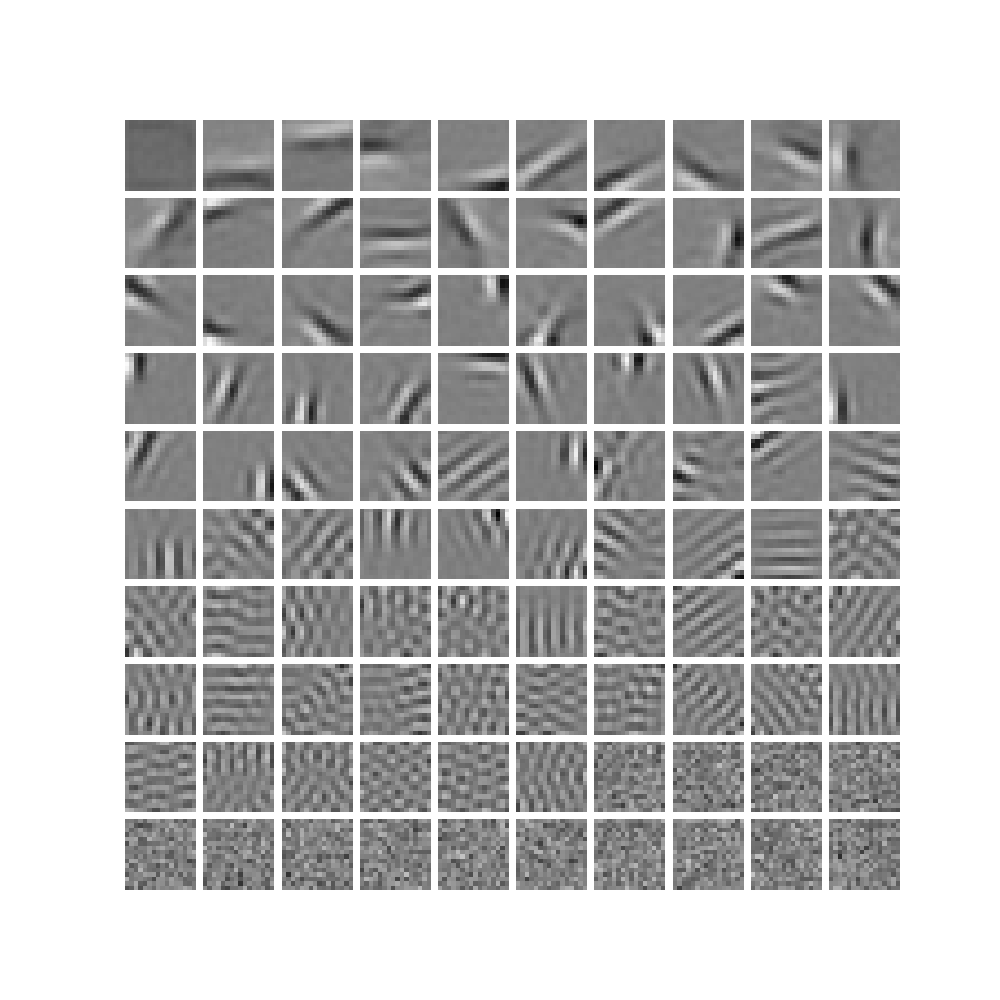}
	\subcaption{Epoch 6}
  \end{minipage}
  \\
  \vspace{0.5cm}
  \begin{minipage}[t]{0.3\textwidth}
    \centering
    \includegraphics[width=\linewidth, trim={2cm 2cm 2cm 2cm}, clip]{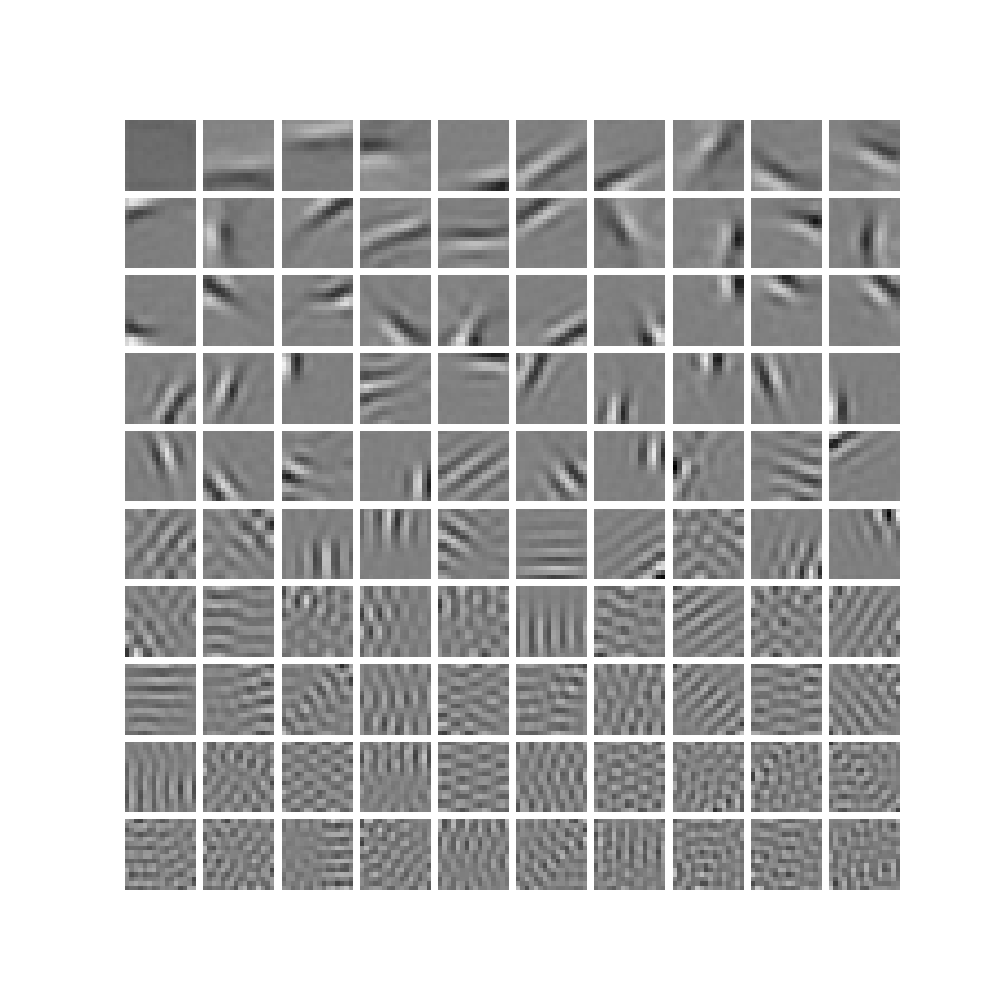}
	\subcaption{Epoch 7}
  \end{minipage}
  \hfill
  \begin{minipage}[t]{0.3\textwidth}
    \centering
    \includegraphics[width=\linewidth, trim={2cm 2cm 2cm 2cm}, clip]{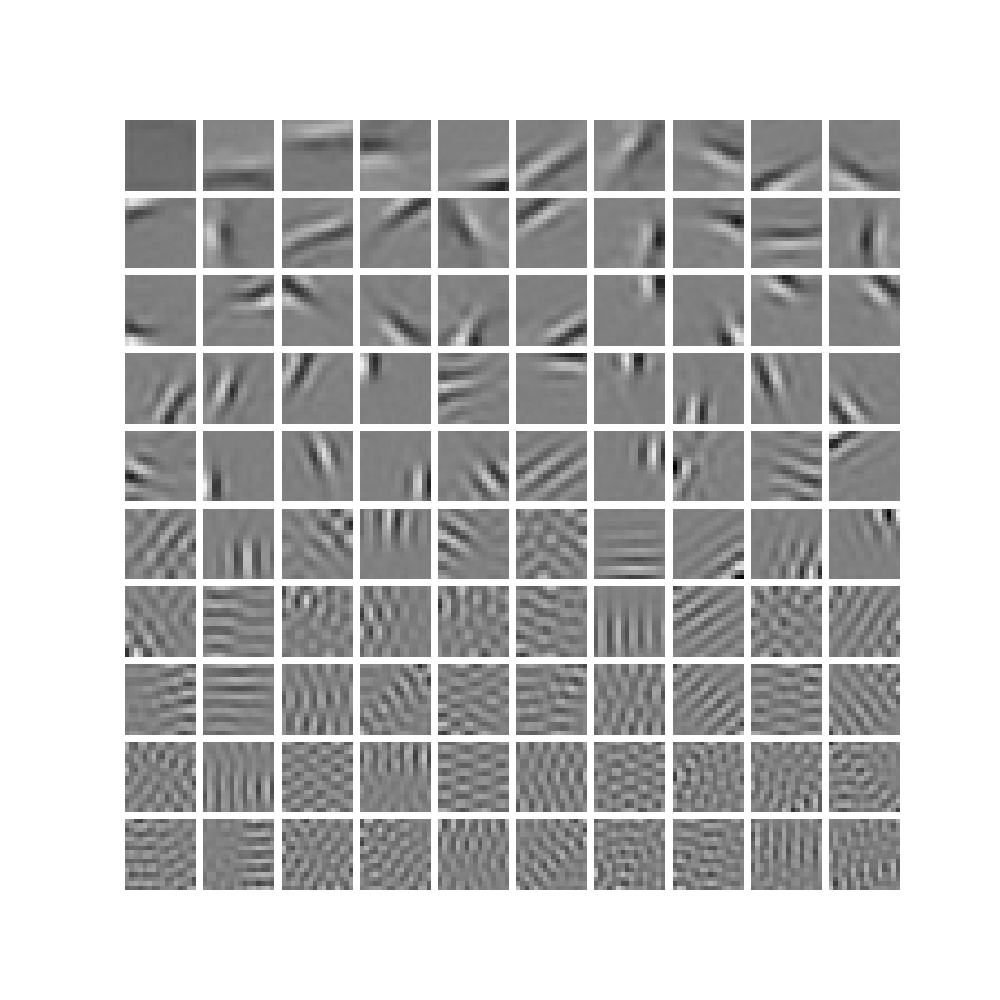}
	\subcaption{Epoch 8}
  \end{minipage}
  \hfill
  \begin{minipage}[t]{0.3\textwidth}
    \centering
    \includegraphics[width=\linewidth, trim={2cm 2cm 2cm 2cm}, clip]{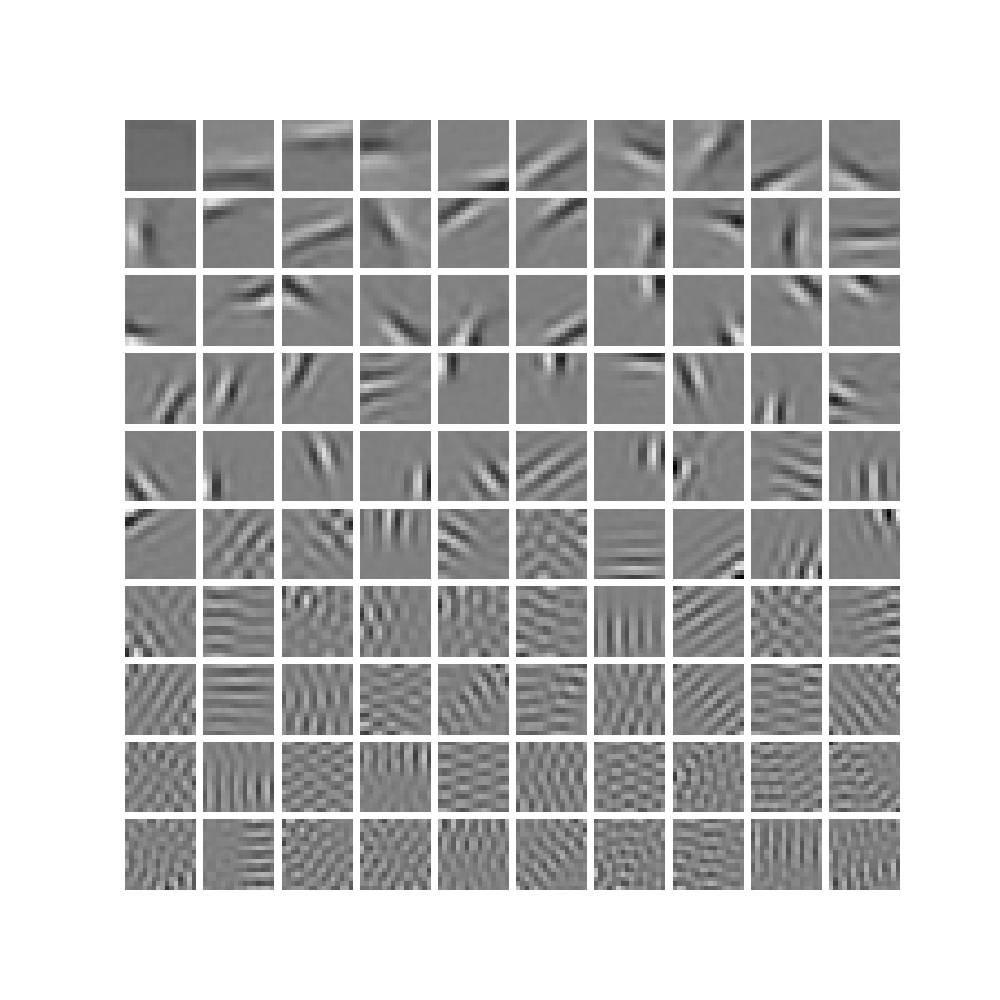}
	\subcaption{Epoch 9}
  \end{minipage}
  \caption{\textbf{Learned generative fields during optimization with likelihood annealing
  of entropy-ELBOs}. The annealing stops after epoch 6. It is equivalent to $\beta$-annealing, which is a popular scheme to tune reconstruction--embedding quality trade-off for VAE training.}
  \label{fig:beta-annealing-all}
\end{figure}

\clearpage
\subsection{Learning overcomplete basis}
\label{app:overcomplete-basis}
Prior entropy annealing significantly improves the quality of the learned dictionary and the sparseness of the latent codes (see \cref{fig:image-patches-dataset-overcomplete}). When after annealing the prior entropy weight is set to 1, approximately half of the generative fields converge to high-frequency textures and contribute only marginally to the reconstruction. Emphasizing the prior allows us to learn a rich dictionary of localized Gabor-like generative fields that span a wide range of frequencies, positions, and orientations, positively contributing to the sparsity of the latent codes.

\begin{figure}[ht!]
  \begin{minipage}[t]{.49\textwidth}
    \centering
    \includegraphics[width=\linewidth, trim={2cm 2cm 2cm 2cm}, clip]{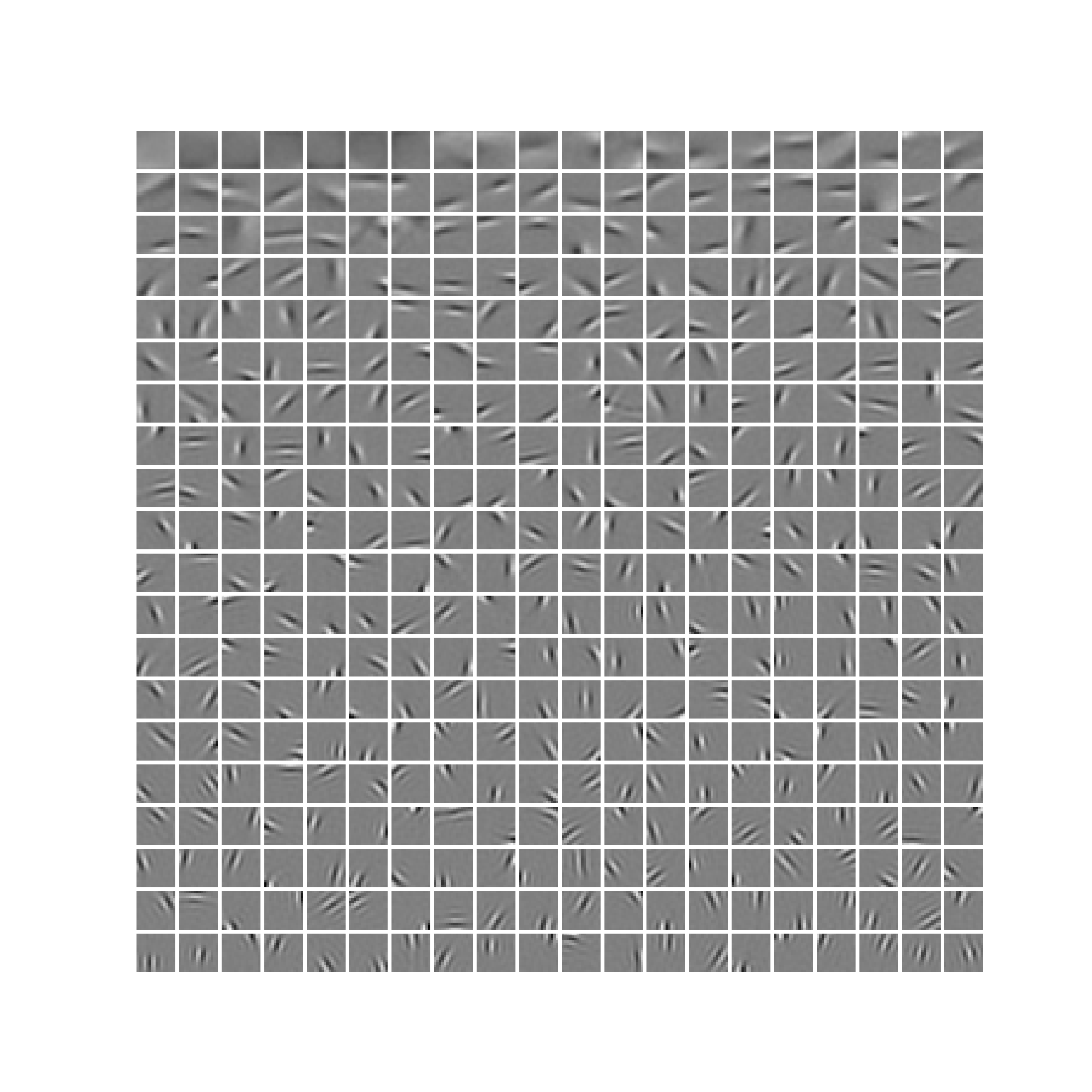}
    \includegraphics[width=\linewidth]{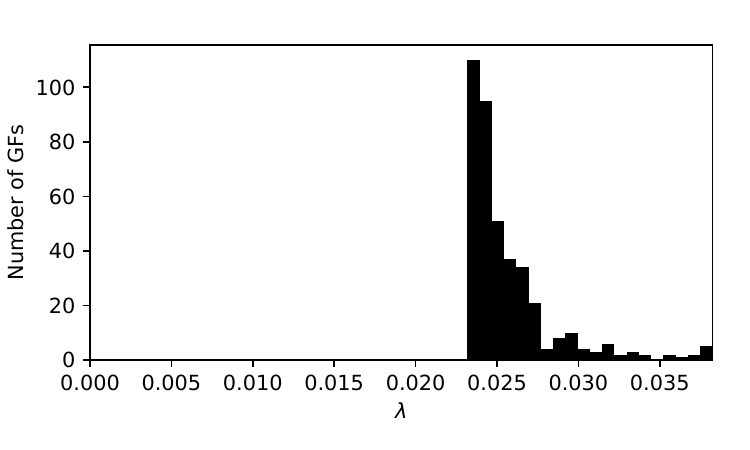}
	\subcaption{Prior entropy weight: 2}%
  \end{minipage}
  \hfill
  \begin{minipage}[t]{.49\textwidth}
    \centering
    \includegraphics[width=\linewidth, trim={2cm 2cm 2cm 2cm}, clip]{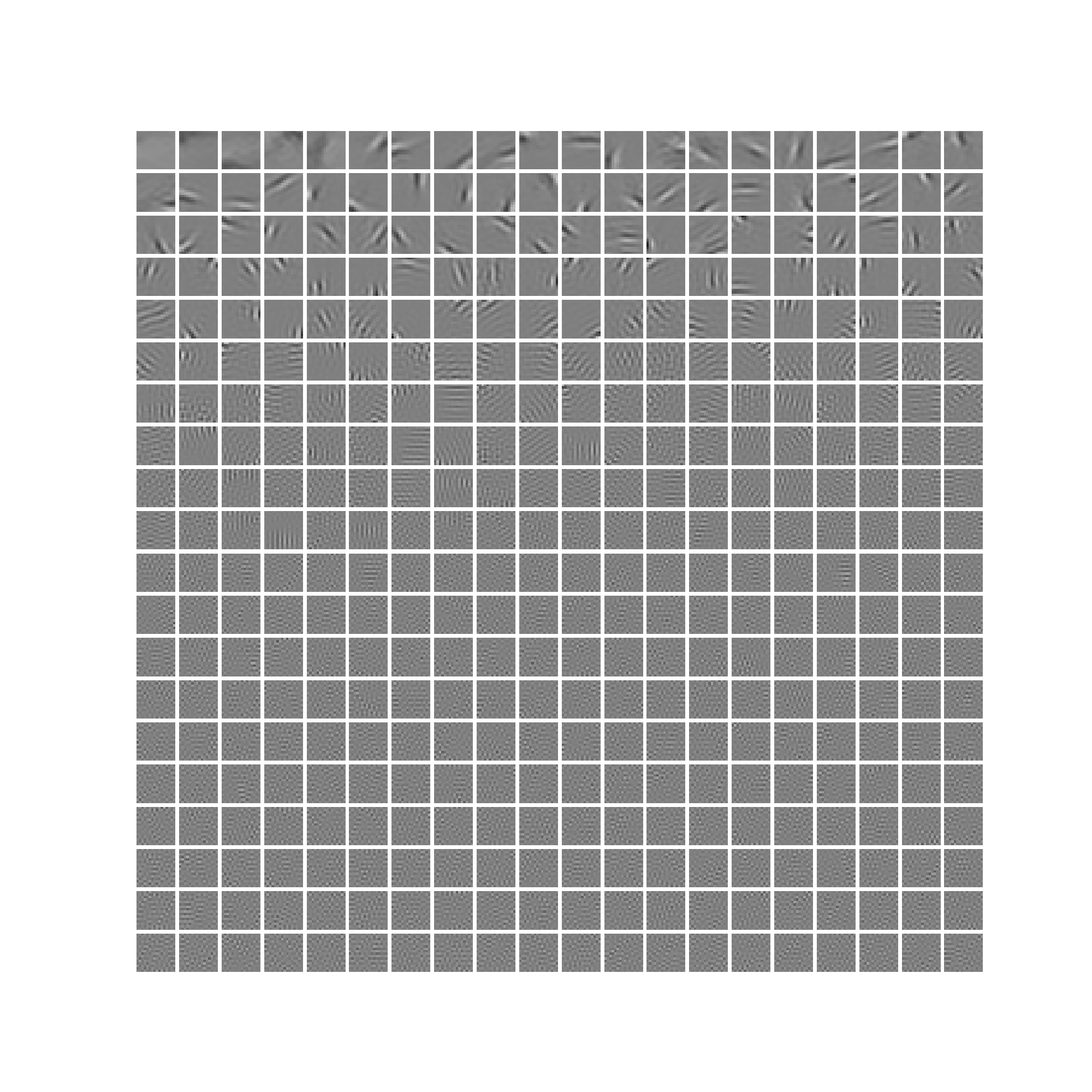}
    \includegraphics[width=\linewidth]{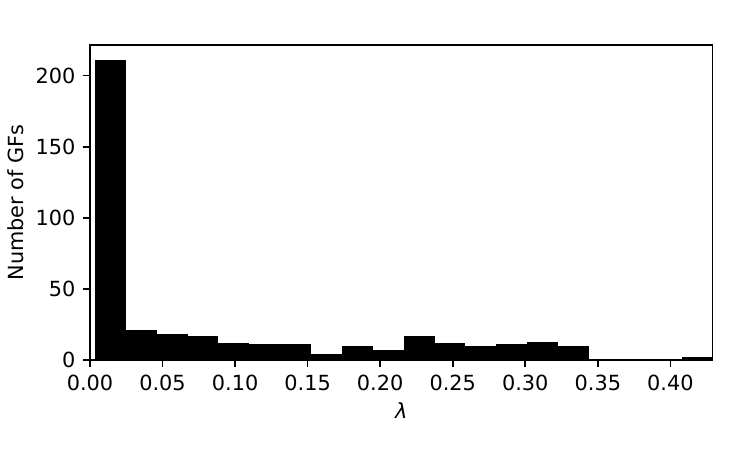}
    \subcaption{Prior entropy weight: 1}%
  \end{minipage}  
  \caption{\textbf{Learned overcomplete bases for the image patches dataset.} 400 generative fields are learned from $16\times16$ image patches. Different generative fields were obtained with prior entropy weight set to 2 (a), and with the original ELBO (b). The bottom histograms illustrate how this reweighting affects the prior scales $\lambda_i$. In (b) more than 200 prior coefficients are close to 0, which hints to the posterior collapse. The generative fields are sorted according to their $\lambda_i$ scale. }
  \label{fig:image-patches-dataset-overcomplete}
\end{figure}

\clearpage
\section{PROPERTIES OF THE FUNCTION \texorpdfstring{$\mathcal{M}$}{M} AND SOFTENED MAGNITUDE} 
\label{app:m-function}
We study the properties of ${\cal M}(a)$ in \cref{eq:DefFuncM}. We find for very small and for very large arguments $a$ of the function that:
\begin{align}
	\lim_{a\rightarrow\infty} {\cal M}(a) = |a|\ \ \mbox{and}\ \ \lim_{a\rightarrow{}-\infty} {\cal M}(a) = |a| \enspace .
\end{align}
So ${\cal M}(a)$ approximates the ``$l_1$'' magnitude function $|a|$ if $a$ has large or small values. %
Furthermore, the function upper-bounds the magnitude function everywhere, and the largest difference of ${\cal M}(a)$ compared
to $|a|$ is at zero:
\begin{align}
	\mbox{for all\ $a\in\RRR$:}\ {\cal M}(a) > |a|\ \ \mbox{and}\ \ {\cal M}(0) = \sqrt{ 2 / \pi } \enspace .
\end{align}
Turning back to the relatively intricate expression for $\lambda^{\mathrm{opt}}_h$ in \cref{theo:optimal-parameters}, we can now define a `softened' magnitude function (cf.\,Sec.\,\ref{subsec:classicalSC}) that formally simplifies the expression significantly:
\begin{align}
	\lambda^{\mathrm{opt}}_h = \frac{1}{N} \sum_{n=1}^N \big|\nu_h^{(n)}\big|^\star\ \ \mbox{where}\ \ \big|\nu_h^{(n)}\big|^\star = \tau_h^{(n)} {\cal M}\Big( \frac{\nu_h^{(n)}}{\tau_h^{(n)}} \Big) \enspace. \label{EqnMStar}
\end{align}
Using the properties of ${\cal M}$, it can directly be observed that $|\nu_h^{(n)}|^\star\approx{}|\nu_h^{(n)}|$ whenever $\nu_h^{(n)} \gg \tau_h^{(n)}$, so for small $\tau_h^{(n)}$ the function $|\nu_h^{(n)}|^\star$ essentially represents the $l_1$ magnitude. We have to keep in mind, however, that $|\nu_h^{(n)}|^\star$ depends on $\tau_h^{(n)}$ (which we have omitted in the notation for convenience). 
As a principled difference between $|\nu_h^{(n)}|^\star$ and $|\nu_h^{(n)}|$ it remains that ultimately the derived function $|\nu^{(n)}_h|^\star$ does (in contrast to $|\nu^{(n)}_h|$) {\em not} vanish for vanishing $\nu_h$. So while many entries $\nu^{(n)}_h$ will be pushed towards zero, the minimum of $|\nu_h^{(n)}|^\star$ will not be at zero. The derived objective is, therefore, genuinely different from $l_1$-sparse coding. It may be used, however, to relate to recent threshold-based variants of the sparse coding objectives \citep[][]{rozell_sparse_2008,fallah_variational_2022}.

The derivative of the function ${\cal M}(a)$ has a particularly simple form, which we already discussed in \cref{theo:analytic-elbo}:
\begin{align}
    \frac{\partial \mathcal{M}(a)}{\partial a} = \erf\left(\frac{a}{\sqrt{2}}\right) \enspace .
\end{align} %

\end{document}